%% file: main.tex
\pdfoutput=1
\documentclass[10pt,twocolumn,letterpaper]{article}

\usepackage{cvpr}
\usepackage{times}
\usepackage{epsfig} 
\usepackage{comment}
\usepackage{xcolor}
\usepackage{graphicx}
\usepackage{amsmath}
\usepackage{amssymb}
\usepackage{amsthm}
\usepackage[ruled,vlined]{algorithm2e}
\include{pythonlisting}
\usepackage{multicol}
\usepackage{optidef}
\newtheorem{thm}{Theorem}
\newtheorem{lemma}{Lemma}
\usepackage{subcaption}
\usepackage{float}
\usepackage[numbers,sort]{natbib}

\usepackage{xpatch}

\xpatchcmd{\section}{%
  \normalfont\scshape\centering}{%
  \normalfont\scshape}{\typeout{Success}}{\typeout{Failure}}%

\usepackage[pagebackref=true,breaklinks=true,colorlinks,bookmarks=false]{hyperref}

\cvprfinalcopy 

\def\cvprPaperID{9517} 

\newcommand{\amit}[1]{{{\color{red} \textbf{(Amit: #1)}}}}

\ifcvprfinal\fi
\begin{document}

\title{Camera On-boarding for Person Re-identification using Hypothesis Transfer Learning}


\author{Sk Miraj Ahmed$^{1,}$\thanks{Equal Contribution}, Aske R Lejb{\o}lle$^{2,\ast,}$\thanks{This work was done while AL was a visiting student at UC Riverside.}, Rameswar Panda$^{3}$, Amit K. Roy-Chowdhury$^{1}$\\
$^{1}$ University of California, Riverside,  $^{2}$ Aalborg University, Denmark, $^{3}$ IBM Research AI, Cambridge\\
{\tt \small \{sahme047@, alejboel@, rpand002@,  amitrc@ece.\}ucr.edu}
}
\maketitle


\begin{abstract}

Most of the existing approaches for person re-identification consider a static setting where the number of cameras in the network is fixed. An interesting direction, which has received little attention, is to explore the dynamic nature of a camera network, where one tries to adapt the existing re-identification models after on-boarding new cameras, with little additional effort. There have been a few recent methods proposed in person re-identification that attempt to address this problem by assuming the labeled data in the existing network is still available while adding new cameras. This is a strong assumption since there may exist some privacy issues for which one may not have access to those data. Rather, based on the fact that it is easy to store the learned re-identifications models, which mitigates any data privacy concern, we develop an efficient model adaptation approach using hypothesis transfer learning that aims to transfer the knowledge using only source models and limited labeled data, but without using any source camera data from the existing network. Our approach minimizes the effect of negative transfer by finding an optimal weighted combination of multiple source models for transferring the knowledge. Extensive experiments on four challenging benchmark datasets with a variable number of cameras well demonstrate the efficacy of our proposed approach over state-of-the-art methods.

\end{abstract}

\section{Introduction}
Person re-identification (re-id), which addresses the problem of matching people across different cameras, has attracted intense attention in recent years~\cite{zheng2016person,gou2018systematic,roy2012camera}. Much progress has been made in developing a variety of methods to learn features \cite{Liao_2015_CVPR,matsukawa2016hierarchical,matsukawa2019hierarchical} or distance metrics by exploiting unlabeled and/or manually labeled data. Recently, deep learning methods have also shown significant performance improvement on person re-id~\cite{ahmed2015improved,li2018harmonious,sun2019perceive,tay2019aanet,yang2019towards,zheng2019joint}.
However, with the notable exception of \cite{panda2017unsupervised,panda2019adaptation}, most of these works have not yet considered the dynamic nature of a camera network, where new cameras can be introduced at any time to cover a certain related area that is not well-covered by the existing network of cameras. To build a more scalable person re-identification system, it is very essential to consider the problem of how to on-board new cameras into an existing network with little additional effort.

\begin{figure}[t]
\hspace{-1em}
\label{cam_insert}
\includegraphics[width=0.5\textwidth, height=2.5in]{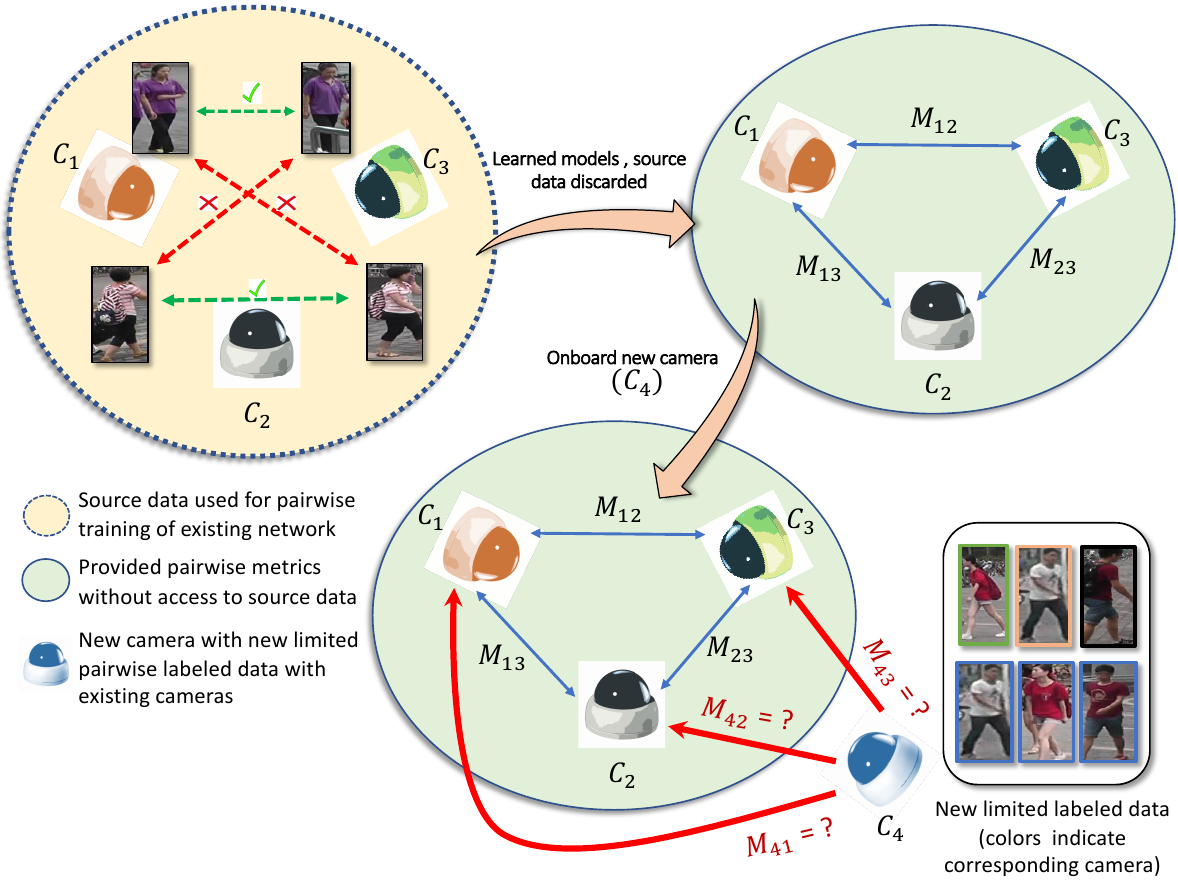}
\caption{Consider a three camera ($C_1$, $C_2$ and $C_3$) network, where we have only three pairwise distance metrics ($M_{12}$, $M_{23}$ and $M_{13}$) available for matching persons, and no access to the labeled data due to privacy concerns. A new camera, $C_4$, needs to be added into the system quickly, thus, allowing us to have only very limited labeled data across the new camera and the existing ones. Our goal in this paper is to learn the pairwise distance metrics ($M_{41}$, $M_{42}$ and $M_{43}$) between the newly inserted camera(s) and the existing cameras, using the learned source metrics from the existing network and a small amount of labeled data available after installing the new camera(s).}

\label{fig:network}
\end{figure}

Let us consider $K$ number of cameras in a network for which we have learned $\binom{K}{2}$ number of optimal pairwise matching metrics, one for each camera pair (see Figure~\ref{fig:network} for an illustrative example). However, during an operational phase of the system, new camera(s) may be temporarily introduced to collect additional information, which ideally should be integrated with minimal effort.
Given newly introduced camera(s), 
the traditional re-id methods aim to re-learn the pairwise matching metrics using a costly training phase. This is impractical in many situations where the newly added camera(s) need to be operational soon after they are added. In this case, we cannot afford to wait a long time to obtain significant amount of labeled data for learning pairwise metrics, thus, we only have limited labeled data of persons that appear in the entire camera network after addition of the new camera(s).

Recently published works~\cite{panda2017unsupervised,panda2019adaptation} attempt to address the problem of on-boarding new cameras to a network by utilizing old data that were collected in the original camera network, combined with newly collected data in the expanded network, and source metrics to learn new pairwise metrics. They also assume the same set of people in all camera views, including the new camera (i.e., before and after on-boarding new cameras) for measuring the view similarity. However, this is unrealistic in many surveillance scenarios as source camera data may have been lost or not accessible due to privacy concerns. Additionally, new people may appear after the target camera is installed who may or may not have appeared in existing cameras. Motivated by this observation, we pose an important question: \textit{How can we swiftly on-board new camera(s) in an existing re-id framework (i) without having access to the source camera data that the original network was trained on, and (ii) relying upon only a small amount of labeled data during the transient phase, i.e., after adding the new camera(s).}

Transfer learning, which focuses on transferring knowledge from a source to a target domain, has recently been very successful in various computer vision problems \cite{lv2018unsupervised,zamir2018taskonomy,sun2019not,yin2019feature,noh2019transfer}. However, knowledge transfer in our system is challenging, because of limited labeled data and absence of source camera data while on-boarding new cameras. To solve these problems, we develop an efficient model adaptation approach using \textit{hypothesis transfer learning} that aims to transfer the knowledge using only source models (i.e., learned metrics) and limited labeled data, but without using any original source camera data. \emph{Only a few labeled identities that are seen by the target camera, and one or more of the source cameras, are needed for effective transfer of source knowledge to the newly introduced target cameras.} Henceforth, we will refer to this as \textit{target data}. Furthermore, unlike~\cite{panda2017unsupervised,panda2019adaptation}, which identify only one best source camera that aligns maximally with the target camera, our approach focuses on identifying an optimal weighted combination of multiple source models for transferring the knowledge.

Our approach works as follows. Given a set of pairwise source metrics and limited labeled target data after adding the new camera(s), we develop an efficient convex optimization formulation based on hypothesis transfer learning~\cite{kuzborskij2013stability,du2017hypothesis} that minimizes the effect of negative transfer from any outlier source metric while transferring knowledge from source to the target cameras. More specifically, we learn the weights of different source metrics and the optimal matching metric jointly by alternating minimization, where the weighted source metric is used as a biased regularizer that aids in learning the optimal target metric only using limited labeled data. The proposed method, essentially, learns which camera pairs in the existing source network best describe the environment that is covered by the new camera and one of the existing cameras.  
Note that our proposed approach can be easily extended to multiple additional cameras being introduced at a time in the network or added sequentially one after another.

\subsection{Contributions}
We address the problem of swiftly on-boarding new camera(s) into an existing person re-identification network without having access to the source camera data, and  relying upon only a small amount of labeled target data in the transient phase, i.e., after adding the new cameras. Towards solving the problem, we make the following contributions.
\begin{itemize}
    \item We propose a robust and efficient multiple metric hypothesis transfer learning algorithm to efficiently adapt a newly introduced camera to an existing person re-id framework without having access to the source data. 
    \item We theoretically analyse the properties of our algorithm and show that it minimizes the risk of negative transfer and performs closely to fully supervised case even when a small amount of labeled data is available. 
    \item We perform rigorous experiments on multiple benchmark datasets to show the effectiveness of our proposed approach over existing alternatives.
\end{itemize}

\section{Related Work}

\noindent\textbf{Person Re-identification.} Most of the methods in person re-id are based on supervised learning. These methods apply extensive training using lots of manually labeled training data, and can be broadly classified in two categories: (i) \textit{Distance metric learning based} \cite{guillaumin2009you,weinberger2009distance,koestinger2012large,Liao_2015_CVPR,yang2017person,yu2017cross} (ii) \textit {Deep learning based} \cite{ahmed2015improved,xiao2016learning,qian2017multi,zhou2017point,zheng2019joint,wang2019spatial,yang2019towards}. \textit{Distance metric learning based} methods tend to learn distance metrics for camera pairs using pairwise labeled data between those cameras, whereas end-to-end \textit {Deep learning based} methods tend to learn robust feature representations of the persons, taking into consideration all the labeled data across all the cameras at once. 
To overcome the problem of manual labeling, several unsupervised \cite{yu2017cross,wang2018transferable,lv2018unsupervised,yu2019unsupervised,yang2019patch,lin2019bottom} and semi-supervised \cite{wu2018exploit,fan2018unsupervised,xin2019semi,wu2019distilled} methods have been developed over the past decade. However, these methods do not consider the case where new cameras are added to an existing network. The most recent approach in this direction \cite{panda2017unsupervised,panda2019adaptation} has considered unsupervised domain adaptation of the target camera by making a strong assumption of accessibility of the source data. None of these methods have considered the fact of not having access to the source data in the dynamic camera network setting. This is relevant, as source camera data might have been deleted after a while due to privacy concerns. \\
\noindent\textbf{Hypothesis Transfer Learning.} Hypothesis transfer learning \cite{yang2007cross,mansour2009domain,orabona2009model,kuzborskij2013stability,du2017hypothesis} is a type of transfer learning that uses only the learned classifiers from a source domain to efficiently learn a classifier in the target domain, which contains only limited labeled data. This approach is practically appealing as it does not assume any relationship between source and target distribution, nor the availability of source data, which may be non accessible \cite{kuzborskij2013stability}. Most of the literature has dealt with simple linear classifiers for  transferring knowledge \cite{kuzborskij2013stability,wang2016learning}. One recent work \cite{perrot2015theoretical} has addressed the problem of transferring the knowledge of a source metric, which is a positive semi-definite matrix, with some provable guarantees. However, it has been analyzed for only a single source metric and the weight of the metric is calculated by minimizing a cost function using sub-gradient descent from the generalization bound separately, which is a highly non-convex non-differential function. In \cite{wang2016learning}, the method has addressed transfer of multiple linear classifiers in an SVM framework, where the corresponding weights are calculated jointly with the target classifiers in a single optimization. Unlike these approaches, our approach addresses the case of transfer from multiple source  metrics by jointly optimizing for target metric, as well as the source weights to reduce the risk of negative transfer.
\section{Methodology}

Let us consider a camera network with $K$ cameras for which we have learned a total $N= {\binom{K}{2}}$ pairwise metrics using extensive labeled data.
We wish to install some new camera(s) in the system that need to be operational soon after they are added, i.e., without collecting and labeling lots of new training data. We do not have access to the old source camera data, rather, we only have the pairwise source distance metrics.
Moreover, we also have access to only a limited amount of labeled data across the target and different source cameras, which is collected after installing the new cameras.
Using the source metrics and the limited pairwise source-target labeled data, we propose to solve a constrained convex optimization problem (Eq.~\ref{opt:main_opt}) that aims to transfer knowledge from the source metrics to the target efficiently while minimizing the risk of negative transfer. \

\vspace{1mm}
\noindent\textbf{Formulation.} Suppose we have access to the optimal distance metric $M_{ab} \in \mathbb{R}^{d \times d}$ for the $a$ and $b$-th camera pair of an existing re-id network, where $d$ is the dimension of the feature representation of the person images and $a,b \in \{1,2 \ldots K\}$. We also have limited pairwise labeled data $\{(x_{ij},y_{ij})\}_{i=1}^{C}$ between the target camera $\tau$ and the source camera $s$, where $x_{ij}=(x_i-x_j)$ is the feature difference between image $i$ in camera $\tau$ and image $j$ in camera $s$, $C= \binom{n_{\tau s}}{2}$, where $n_{\tau s}$ is the total number of ordered pair images across cameras $\tau$ and $s$,
and $y_{ij} \in \{-1,1\}$. $y_{ij}=1$ if the persons $i$ and $j$ are the same person across the cameras, and $-1$ otherwise. Note that our approach does not need the presence of every person seen in the new target camera across all the source cameras; rather, it is enough for some people in the target camera to be seen in at least one of the source cameras, in order to compute the new distance metric across source-target pairs. Let $S$ and $D$ be defined as $S=\{(i,j) \mid y_{ij}=1\}$ and $D=\{(i,j) \mid y_{ij}=-1\}$. Our main goal is to learn the optimal metric between target and each of the source cameras by using the information from all the pairwise source metrics $\{M_{j}\}_{j=1}^N$ and limited labeled data $\{(x_{ij},y_{ij})\}_{i=1}^{C}$. In standard metric learning context, the distance between two feature vectors $x_i \in \mathbb{R}^d$ and $x_j\in \mathbb{R}^d$ with respect to a metric $M\in \mathbb{R}^{d \times d}$ is calculated by $\sqrt{(x_i-x_j)^\top M (x_i-x_j)}$. 

Thus, we formulate the following optimization problem for calculating the optimal metric $M_{\tau s}$ between target camera $\tau$ and the $s$-th source camera, with $n_s$ and $n_d$ number of similar and dissimilar pairs, as follows:
\begin{mini}
{M_{\tau s},\beta}{\frac{1}{n_s} \sum_{(i,j) \in S}{ x_{ij}^\top M_{\tau s} x_{ij}} +  \lambda \|  M_{\tau s}-\sum_{j=1}^{N}\beta_jM_j \| _F^2}
{}{}
\label{opt:main_opt}
\addConstraint {\frac{1}{n_d} \sum_{(i,j) \in D}(x_{ij}^\top M_{\tau s} x_{ij})-b}{\geq 0 ,}{\ M_{\tau s} \succeq 0}
\addConstraint{\beta}{ \geq 0,}{\ \|\beta\|_2}{ \leq 1}
\end{mini}
The above objective consists of two main terms. The first term is the normalized sum of distances of all similar pair of features between camera $\tau$ and $s$ with respect to the Mahalanobis metric $M_{\tau s}$, and the second term represents the Frobenius norm of the difference of $M_{\tau s}$ and weighted combination of source metrics squared. $\lambda$ is a regularization parameter to balance the two terms. Note that the second term in Eq.~\ref{opt:main_opt} is essentially related to hypothesis transfer learning~\cite{kuzborskij2013stability,du2017hypothesis} where the hypotheses are the source metrics. 
The first constraint represents that the normalized sum of distances of all dissimilar pairs of features with respect to $M_{\tau s}$ is greater than a user defined threshold $b$, and the second constraints the distance metrics to always lie in the positive semi-definite cone. While the third constraint keeps all the elements of the source weight vector non-negative, the last constraint ensures that the weights should not deviate much from zero (through upper-bounding the $\ell$-2 norm by $1$).

\vspace{1mm}
\noindent\textbf{Notation.} We use the following notations in the optimization steps.
\begin{itemize}
    \item[(a)] $\mathcal{C}_1=\{M \in \mathbb{R}^{d \times d} \mid \frac{1}{n_d}\sum\limits_{(i,j) \in D}(x_{ij}^\top M x_{ij})-b \geq 0\}$
    \item[(b)] $\mathcal{C}_2=\{M \in \mathbb{R}^{d \times d}\mid M\succeq 0\}$
    \item[(c)] $\mathcal{C}_3=\{ \beta  \in \mathbb{R}^N \mid \beta \geq 0 \cap  \|\beta\|_2 \leq 1\}$
\end{itemize}

\vspace{1mm}
\noindent\textbf{Optimization.}
The proposed optimization problem (\ref{opt:main_opt}) is jointly convex over $M_{\tau s}$ and $\beta$. To solve this optimization over large size matrices, we devise an iterative algorithm to efficiently solve (\ref{opt:main_opt}) by alternatively solving for two sub-problems. For the sake of brevity, we denote $M_{\tau s}$ as $M$ in the subsequent steps.
Specifically, in the first step, we fix the weight $\beta$ and  take a gradient step with respect to $M$ in the descent direction with step size $\alpha$ (Eq.~\ref{opt:gradM}). Then, we project the updated $M$ onto $\mathcal{C}_1$ and $\mathcal{C}_2$ in an alternating fashion until convergence (Eq.~\ref{opt:pic1} and Eq.~\ref{opt:pic2}). In the next step, we fix the the updated $M$ and take a step with size $\gamma$ towards the direction of negative gradient with respect to $\beta$ (Eq.~\ref{opt:gradbetawhole}). In the last step, we simply project $\beta$ onto the set $\mathcal{C}_3$ (Eq.~\ref{opt:pic3}).
Algorithm~\ref{algo1} summarizes the alternating minimization procedure to optimize (\ref{opt:main_opt}). We briefly describe these steps below and refer the reader to the supplementary material for more mathematical details.

\begin{algorithm}[]
\SetAlgoLined
\textbf{Input:} Source metric $\{M_{j}\}_{j=1}^N$, $\{(x_{ij},y_{ij})\}_{i=1}^{C}$  \\
\textbf{Output:} Optimal metric $M^\star$ \\
\textbf{Initialization}: $M^k,\beta^k,k=0$\;
 \While{convergence}{
 $M^{k+1}=M^k-\alpha \nabla_M f(M,\beta^k)|_{M=M^k}$ (Eq.~\ref{opt:gradM})\;
  \While{convergence}{
  $M^{k+1}=\Pi_{\mathcal{C}_1}(M^{k+1})$ (Eq.~\ref{opt:pic1})\;
  $M^{k+1}=\Pi_{\mathcal{C}_2}(M^{k+1})$ (Eq.~\ref{opt:pic2})\;
 }
  $\beta^{k+1}= \beta^{k}-\gamma \nabla_\beta(f(M^{k+1},\beta)|_{\beta=\beta^k}$ (Eq.~\ref{opt:gradbetawhole})\;
  $\beta^{k+1}=\Pi_{\mathcal{C}_3}(\beta^{k+1})$ (Eq.~\ref{opt:pic3})\;
  $k=k+1$ \;
 
 }
 \caption{Algorithm to Solve Eq.~\ref{opt:main_opt}} 
 \label{algo1}
\end{algorithm}

 

\textbf{Step 1: Gradient w.r.t $M$ with fixed $\beta$.}

With $k$ being the iteration number and $M^k$, $\beta^k$ being $M$ and $\beta$ in the $k$-th iteration, we compute the  gradient of the objective function (\ref{opt:main_opt}) with respect to $M$ by fixing $\beta=\beta^k$ at the $k$-th iteration as follows:
\begin{equation}
\nabla_M f(M,\beta^k)|_{M=M^k}=\Sigma_S+2 \lambda (M^k-\sum_{j=1}^{N}\beta_j^k M_j),\\
\label{opt:gradM}
\end{equation}
where $\scriptsize {\Sigma_S=\frac{1}{n_s}\sum\limits_{(i,j) \in S}{ x_{ij}x_{ij}^\top}}$  and $\scriptsize {f(M,\beta^k)=\frac{1}{n_s}\sum\limits_{(i,j) \in S}{ x_{ij}^\top M x_{ij}} + \lambda \|M-\sum\limits_{j=1}^{N}\beta_j^k M_j\|_F^2}$.

\textbf{Step 2: Projection of $M$ onto $\mathcal{C}_1$ and $\mathcal{C}_2$.}
The projection of $M$ onto $\mathcal{C}_1$ (denoted as $\Pi_{\mathcal{C}_1}(M)$) can be computed by solving a constrained optimization as follows: 

\begin{alignat*}{6}
\Pi_{\mathcal{C}_1}(M) =&\text{arg}\ \underset{\hat{M}}{\text{min}}~~&& \frac{1}{2} \|\hat{M}-M\|_F^2 \\
& \text{Subject to} && \frac{1}{n_d}\sum_{(i,j) \in D}(x_{ij}^\top \hat{M} x_{ij})-b \geq 0 \\
\end{alignat*} 

By writing the Lagrange for the above constrained optimization and using KKT conditions with strong duality, the projection of $M$ onto $\mathcal{C}_1$ can be written as 
\begin{equation}
    \Pi_{\mathcal{C}_1}(M)= M + \max\left \{0,\frac{\left(b-\frac{1}{n_d}\sum\limits_{(i,j) \in D}x_{ij}^\top M x_{ij}\right)}{\|\Sigma_D\|_F^2}\ \right \} \Sigma_D,
    \label{opt:pic1}
\end{equation}
where $\scriptsize{\Sigma_D = \frac{1}{n_d}\sum\limits_{(i,j) \in D}{ x_{ij}x_{ij}^\top}}$.
Similarly, using spectral value decomposition, the projection of $M$ onto $\mathcal{C}_2$ can be written as  
\begin{equation}
   \Pi_{\mathcal{C}_2}( M)=V\textit{diag}(\begin{bmatrix}\hat{\lambda_1} & \hat{\lambda_2} \ldots \hat{\lambda_n} \end{bmatrix}) V^\top,
   \label{opt:pic2}
\end{equation}
where $V$ is the eigenvector matrix of $M$, $\lambda_i$ is the $i$-th eigenvalue of $M$ and $\hat{\lambda_j}=\max \{\lambda_j,0\} \quad \forall \quad j \in \begin{bmatrix} 1 \ldots d\end{bmatrix}$.

\textbf{Step 3: Gradient w.r.t $\beta$ with fixed $M$.} By fixing $M=M^{k+1}$ in the objective function, differentiating it w.r.t $\beta_i$, the $i$-th element of $\beta$ at the point $\beta=\beta^k$, we get 
\begin{equation}
\begin{split}
    \nabla_{\beta_i}(f(M^{k+1},\beta))|_{\beta_i=\beta_i^k} = 2 \lambda \beta_i^k  \mathrm{trace}(M_i^\top M_i) - \\ 2 \lambda \mathrm{trace}( M_i^\top (M^{k+1}-\sum_{j=1,j \neq i}^{N}\beta_j^k M_j) ) 
\end{split}
\label{opt:gradbeta}
\end{equation}
By denoting $\nabla_{\beta_i}(f(M^{k+1},\beta))|_{\beta_i=\beta_i^k}$ as $a_i^k$, we get
\begin{equation}
  \nabla_\beta(f(M^{k+1},\beta))|_{\beta=\beta^k}  = \begin{bmatrix} 
a_1^k   & a_2^k   & \ldots & a_N^k \\
\end{bmatrix}^\top
\label{opt:gradbetawhole}
\end{equation}
\textbf{Step 4: Projection of $\beta$ onto $\mathcal{C}_3$.}
This step essentially projects a vector to the first quadrant of an $N$-dimensional unit norm hyper-sphere. 
The closed form expression of the projection onto $\mathcal{C}_3$ is as follows:
\begin{equation}
    \Pi_{\mathcal{C}_3} (\beta^{k+1}) = \max \Big\{0,\frac{\beta^{k+1}}{\max \{1,\|\beta^{k+1}\|_2 \}}\Big\}
    \label{opt:pic3}
\end{equation}

\section{Discussion and Analysis}
\label{sec:insight}
One of the key differences between our approach and existing methods is that the nature of our problem deals with the multiple metric setting within the hypothesis transfer learning framework. In this section, following~\cite{perrot2015theoretical}, we theoretically analyze the properties of our Algorithm~\ref{algo1} for transferring knowledge from multiple metrics.

Let $\mathcal{T}$ be a domain defined over the set ($\mathcal{X} \times \mathcal{Y}$) where $\mathcal{X} \subseteq \mathbb{R}^d$ and $\mathcal{Y} \in \{-1,1\}$ denote the feature and label set, respectively, and has a probability distribution denoted by $\mathcal{D}_\mathcal{T}$. Let $T$ be the target domain defined by $\{(x_i,y_i)\}_{i=1}^{n}$ consisting of $n$ i.i.d samples, each drawn from the distribution $\mathcal{D}_\mathcal{T}$. The optimization proposed in Eq.1 of \cite{perrot2015theoretical} (page. 2) is defined as: 
\begin{equation}
     \underset{M \succeq 0}{\text{minimize}}~~L_T(M)+\lambda \|M-M_S\|_F^2
     \label{perrotopt}
\end{equation}
Fixing the value of $\beta$ in our proposed optimization (\ref{opt:main_opt}), we have an optimization problem equivalent to \eqref{perrotopt}, where $M_S= \sum_{j=1}^{N} \beta_j M_j$ and
\begin{equation}
   L_T(M)=\frac{1}{n_s}\sum_{(i,j) \in S}{ x_{ij}^\top M x_{ij}}+\mu^\star\big(b-\frac{1}{n_d}\sum_{(i,j) \in D}x_{ij}^\top M x_{ij}\big)
   \label{LTM}
\end{equation}
 Note that $\mu^\star$ in Eq.~\ref{LTM} is the optimal dual variable for the inequality constraint optimization \eqref{opt:main_opt} with the weight vector fixed. Clearly, the expression is linear, hence convex in $M$, and has a finite Lipschitz constant $k$.
\begin{thm}
For the convex and $k$-Lipschitz loss (shown in supp) defined  in~\eqref{LTM} the average bound can be expressed as
\begin{equation}
    \mathbb{E}_{T \sim \mathcal{D}_{\mathcal{T}^n}}[L_{\mathcal{D}_\mathcal{T}}(M^\star)] \leq L_{\mathcal{D}_\mathcal{T}}(\widehat{M_S}) + \frac{8k^2}{\lambda n},
    \label{ineq}
\end{equation}
where $n$ is the number of target labeled examples, $M^\star$ is the optimal metric computed from Algorithm~\ref{algo1}, $\widehat{M_S}$ is the average of all source metrics defined as $\frac{\sum_{j=1}^{N} M_j}{N}$, $\mathbb{E}_{T \sim \mathcal{D}_{\mathcal{T}^n}}[L_{\mathcal{D}_\mathcal{T}}(M^\star)]$ is the expected loss by $M^\star$ computed over distribution $\mathcal{D}_\mathcal{T}$ and $L_{\mathcal{D}_\mathcal{T}}(\widehat{M_S})$ is the loss of average of source metrics computed over  $\mathcal{D}_\mathcal{T}$. 
\label{thm1}
\end{thm}
\begin{proof}

The proof is given in supplementary material.
\end{proof}

\noindent \textbf{Implication of Theorem}~\ref{thm1}: 
Since we transfer knowledge from multiple source metrics, and do not know which is the most generalizable over the target distribution (i.e., the best source metric), the most sensible thing is to check for the average performance of using each of the source metrics directly over the target test data. It is equivalently giving all the source metrics equal weights and not using any of the target data for training purpose. The bound in Theorem~\eqref{LTM} shows that, on average, the  metric learned form Algorithm~\ref{algo1} tends to do better than, or in worst case, at least equivalent to the average of source metrics with a fast convergence rate of $\mathcal{O}(\frac{1}{n})$ with limited number of target samples~\cite{perrot2015theoretical}. 
\\

\begin{thm}
With probability $(1-\delta)$, for any metric $M$ learned from Algorithm~\ref{algo1} we have,
\begin{equation}
\begin{aligned}
    L_{\mathcal{D}_\mathcal{T}}(M) \leq &L_T(M) + \mathcal{O}\big(\frac{1}{n}\big) + \\ \Bigg(\sqrt{\frac{L_T(\sum_{j=1}^{N} \beta_{j} M_j)}{\lambda}}+ 
    &\|  \sum_{j=1}^{N} \beta_j M_j\|_F \Bigg) \sqrt{\frac{\ln(\frac{2}{\delta})}{2n}},
\end{aligned}
\label{gbound}
\end{equation}
\label{thm2}
where $L_{\mathcal{D}_\mathcal{T}}(M)$ is the loss over the original target distribution (true risk), $L_T(M)$ is the loss over the existing target data (empirical risk), and $n$ is the number of target samples. 
\end{thm}
\begin{proof}
See the supplementary material for the proof.
\end{proof}

\noindent \textbf{Implication of Theorem}~\ref{thm2}: This bound shows that given only a small amount of labeled target data, our method performs closely to the fully supervised case.
The right hand side of the inequality~\eqref{gbound} consists of the term $\mathcal{O}\big(\frac{1}{n}\big)+\Phi(\beta) \mathcal{O}\big(\frac{1}{\sqrt{n}}\big)$. Since the optimal weight $\beta^\star$ from optimization~\eqref{opt:main_opt} will be sparse due to the way $\beta$ is constrained, zero weights will automatically be assigned to the outlier metrics, i.e., outlier $M_j$s, resulting in zero values for the terms $\beta_k^\star L_T(M_j)$ corresponding to those indices $j$ and hence smaller value of $\Phi(\beta)$. As a result, the $\mathcal{O}\big(\frac{1}{\sqrt{n}}\big)$ term will be less dominant in~\eqref{gbound} than $\mathcal{O}\big(\frac{1}{n}\big)$, due to smaller associated coeffiecient $\Phi(\beta^\star)$ and, hence, can be ignored. Thus, due to the faster decay rate of $\mathcal{O}\big(\frac{1}{n}\big)$, this implies that with very limited target data, the empirical risk will converge to the true risk. 
Furthermore, when $n$ is very large (the fully supervised case), $\mathcal{O}\big(\frac{1}{\sqrt{n}}\big)$ will be close to zero and cannot be altered by multiplication with any coefficient. This implies that the source metrics will not have any effect on learning when there is enough labeled target data available and are only useful in the presence of limited data as in our application domain.\\
\noindent \textbf{Negative Transfer}: In optimization~(\ref{opt:main_opt}), we jointly estimate the optimal metric, as well as the weight vector, which determines which source to transfer from and with how much weight. If a source metric is not a good representative of the target distribution, for an optimal $\lambda$, the weight associated to this metric will automatically be set to zero or close to zero by optimization~(\ref{opt:main_opt}), due to the sparsity constraint of $\beta$.
Hence, our approach minimizes the risk of negative transfer. 

\begin{figure*}[ht]
\centering
\begin{subfigure}{0.24\textwidth}
\includegraphics[width=\textwidth]{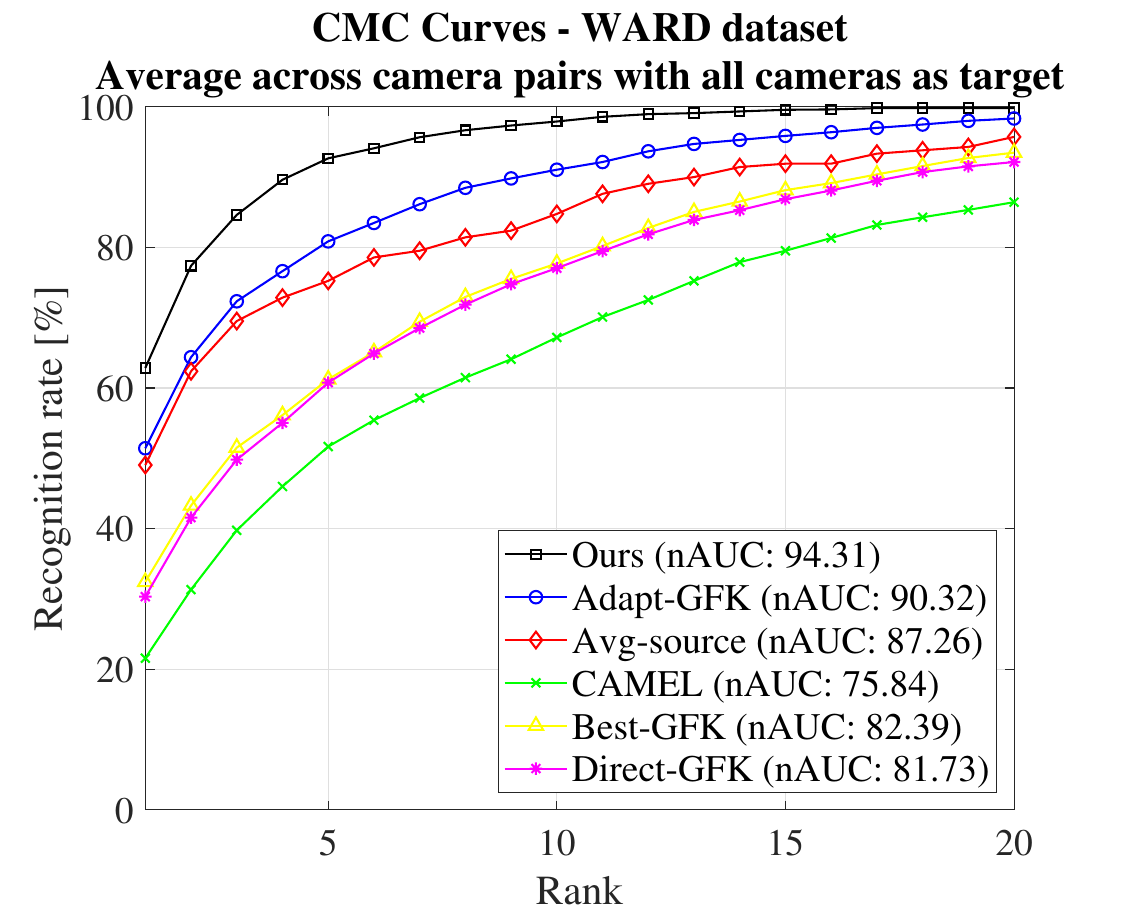}
\caption{}
\end{subfigure}
\begin{subfigure}{0.24\textwidth}
\includegraphics[width=\textwidth]{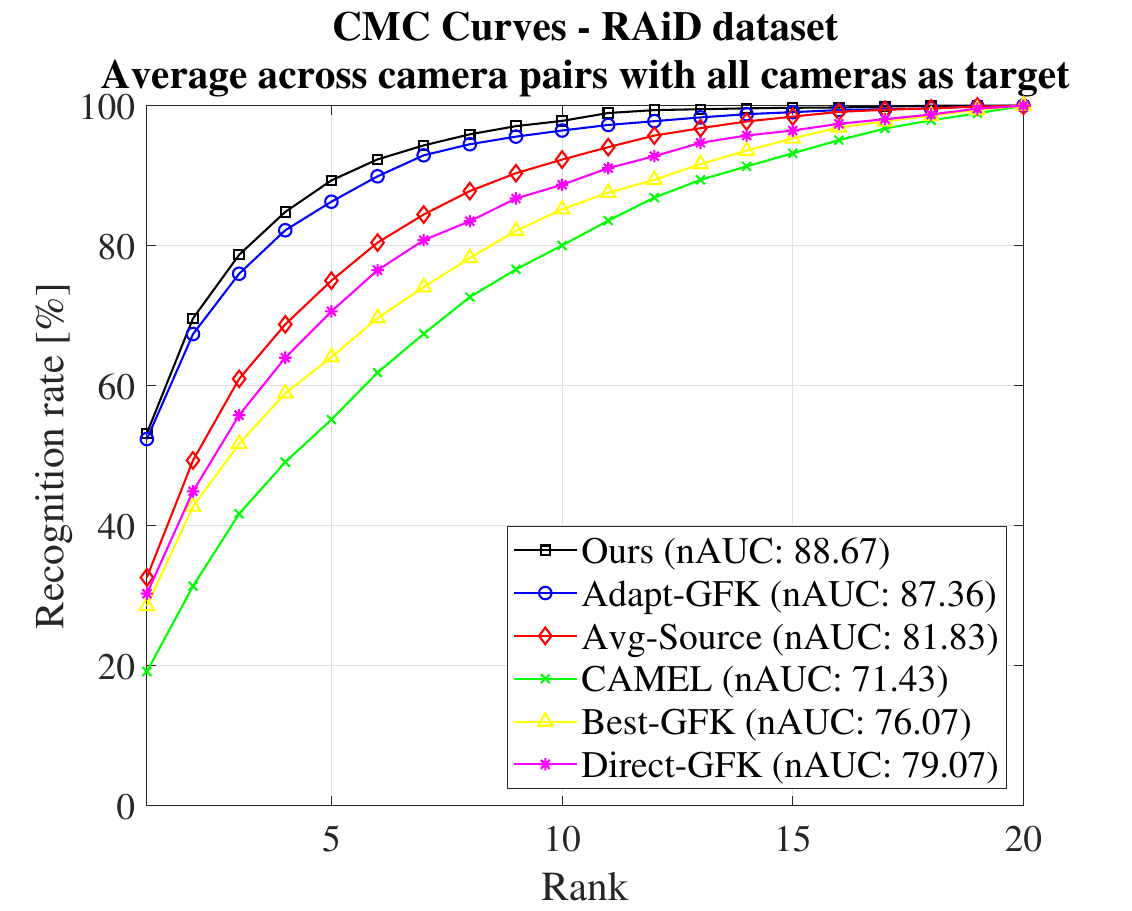}
\caption{}
\end{subfigure}
\begin{subfigure}{0.24\textwidth}
\includegraphics[width=\textwidth]{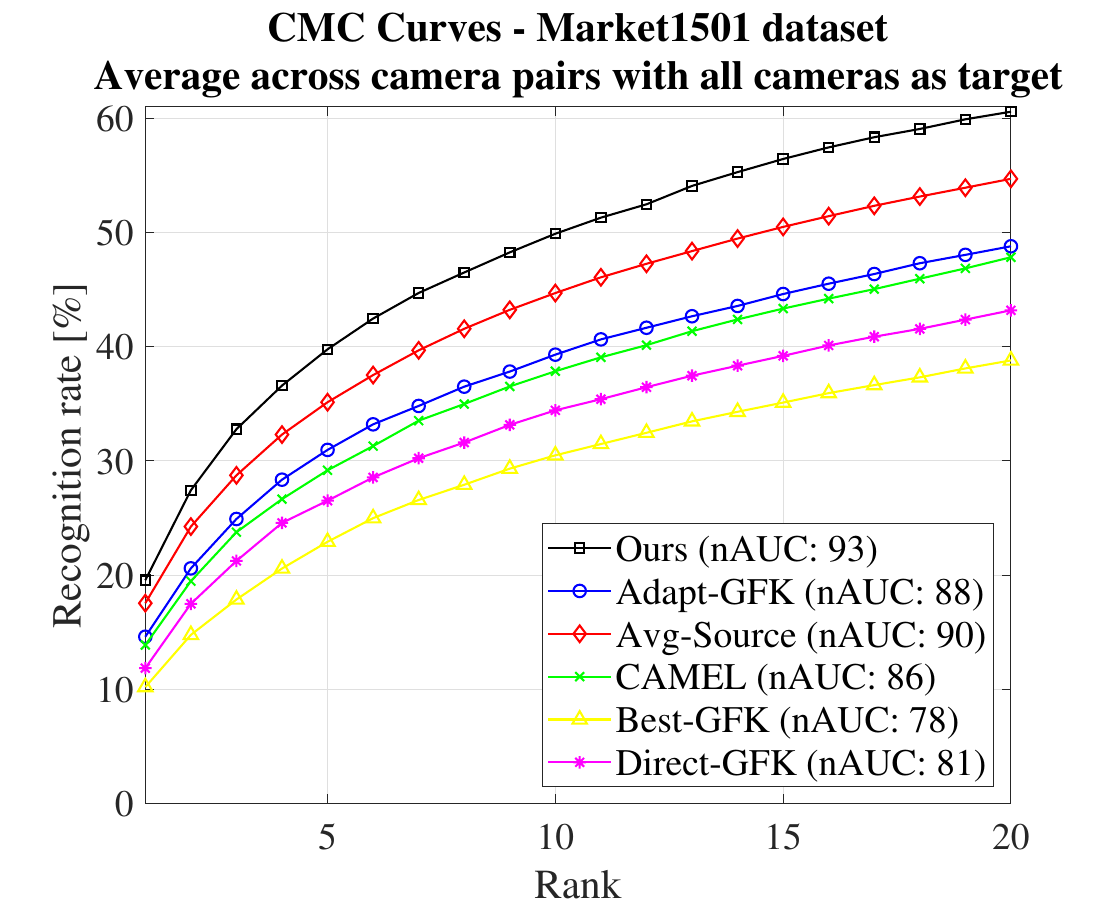}
\caption{}
\end{subfigure}
\begin{subfigure}{0.24\textwidth}
\includegraphics[width=\textwidth]{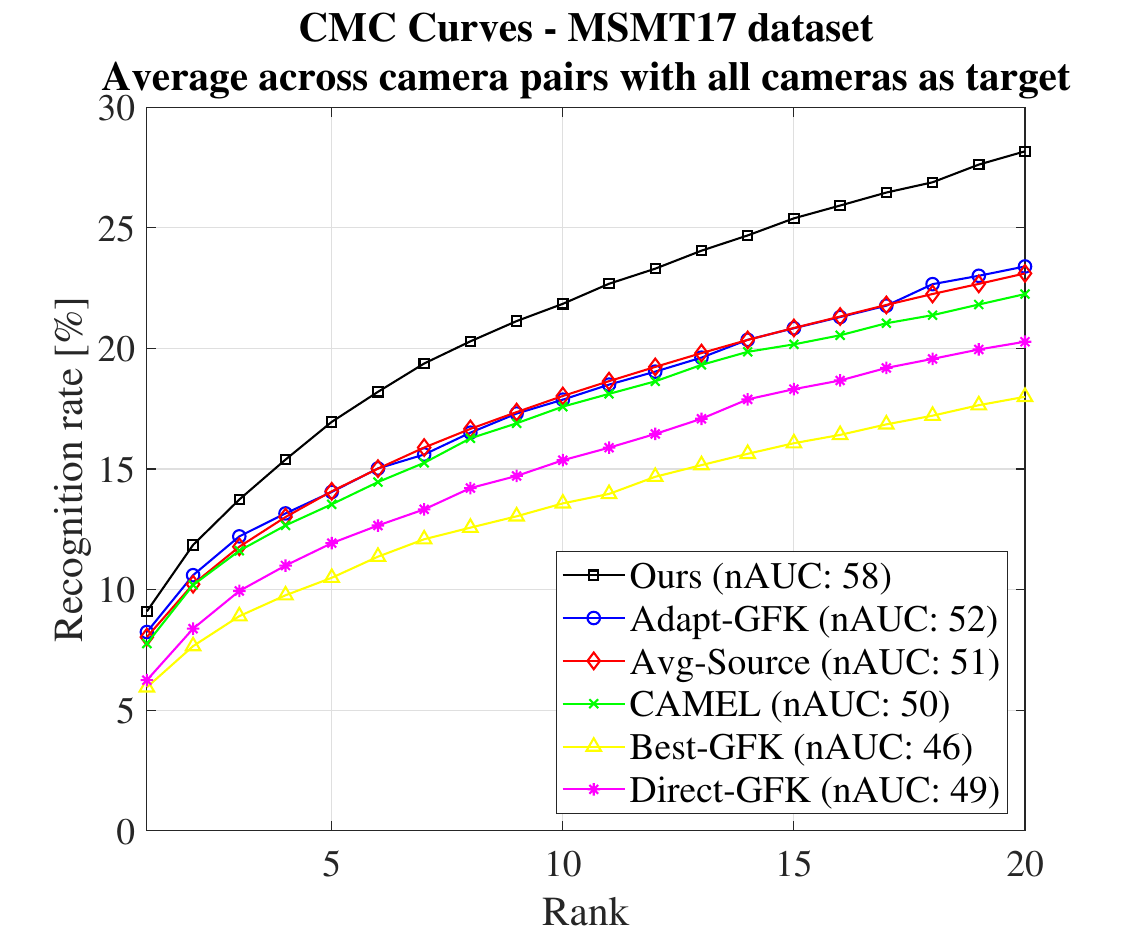}
\caption{}
\end{subfigure}
\caption{CMC curves averaged over all target camera combinations, introduced one at a time. (a) WARD with 3 cameras, (b) RAiD with 4 cameras, (c) Market1501 with 6 cameras and (d) MSMT17 with 15 cameras. Best viewed in color.}
\label{fig:singlecam}
\end{figure*}

\section{Experiments} 
\label{sec:experiments}
\noindent \textbf{Datasets}. We test the effectiveness of our method by experimenting on four publicly available person re-id datasets such as WARD \cite{martinel2012re}, RAiD \cite{das2014consistent}, Market1501 \cite{zheng2015scalable}, and MSMT17 \cite{wei2018person}. There are several other re-id datasets like ViPeR \cite{gray2008viewpoint}, PRID2011 \cite{hirzer2011person} and CUHK01 \cite{li2012human}; however, those do not apply in our case due to availability of only two cameras. RAiD and WARD are smaller datasets with 43 and 70 persons captured in $4$ and $3$ cameras, respectively, whereas Market1501 and MSMT17 are more recent and large datasets with 1,501 and 4,101 persons captured across $6$ and $15$ cameras, respectively. 

\vspace{1mm}
\noindent \textbf{Feature Extraction and Matching.}
We use Local Maximal Occurrence (LOMO) feature \cite{Liao_2015_CVPR} of length $29,960$ in RAiD and WARD datasets. However, since LOMO usually performs poorly on large datasets \cite{gou2018systematic}, for Market1501 and MSMT17 we extract features from the last layer of an Imagenet \cite{deng2009imagenet} pre-trained ResNet50 network \cite{he2016deep} (denoted as IDE features in our work). 
We follow standard PCA technique to reduce the feature dimension to $100$, as in \cite{koestinger2012large,panda2017unsupervised}.  

\vspace{1mm}
\noindent \textbf{Performance Measures.} We provide standard Cumulative Matching Curves (CMC) and normalized Area Under Curve (nAUC), as is common in person re-id \cite{Liao_2015_CVPR,koestinger2012large,das2014consistent,panda2019adaptation}. While the former shows accumulated accuracy by considering the $k$-most similar matches within a ranked list, the latter is a measure of re-id accuracy, independent on the number of test samples.
Due to the space constraint, we only report average
CMC curves for most experiments and leave the full CMC
curves in the supplementary material. 

\vspace{1mm}
\noindent \textbf{Experimental Settings.} 
For RAiD we follow the protocol in \cite{Liao_2015_CVPR} and randomly split the persons into a training set of 21 persons and a test set of 20 persons, whereas for WARD, we randomly split the 70 persons into a set of 35 persons for training and rest 35 persons for testing. For both datasets, we perform 10 train/test splits and average accuracy across all splits. We use the standard training and testing splits for both Market1501 and MSMT17 datasets.
During testing, we follow a multi-query approach by averaging all query features of each id in the target camera and compare with all features in the source camera \cite{zheng2015scalable}. 

\vspace{1mm}
\noindent \textbf{Compared Methods.} We compare our approach with the following methods. (1) Two variants of Geodesic Flow Kernel (GFK) \cite{gong2012geodesic} such as Direct-GFK where the kernel between a source-target camera pair is directly used to evaluate the accuracy and Best-GFK where GFK between the best source camera and the target camera is used to evaluate accuracy between all source-target camera pairs as in \cite{panda2017unsupervised,panda2019adaptation}. Both methods use the supervised dimensionality reduction method, Partial Least Squares (PLS), to project features into a low dimensional subspace \cite{panda2017unsupervised,panda2019adaptation}. (2) State-of-the-art method for on-boarding new cameras \cite{panda2017unsupervised,panda2019adaptation} that uses transitive inference over the learned GFK across the best source and target camera (Adapt-GFK). (3) Clustering-based Asymmetric MEtric Learning (CAMEL) method of \cite{yu2017cross}, which projects features from source and target camera to a shared space using a learned projection matrix. 
For all compared methods, we use their publicly available code and perform evaluation in our setting.

 \subsection{On-boarding a Single New Camera}
 \label{onboardsinglecam}
We consider one camera as newly introduced target camera and all the other as source cameras. We consider all the possible combinations for conducting experiments. 
In addition to the baselines described above, we compare against the accuracy of average of the source metrics (Avg-Source) by applying it directly over the target test set to prove the validity of Theorem~\ref{thm1}.
We also compute the GFK kernels in two settings; by considering only target data available after introducing the new cameras (Figure~\ref{fig:singlecam}) and by considering the presence of both old source data and the new labeled data after camera installation as in \cite{panda2017unsupervised,panda2019adaptation} (Figure~\ref{fig:gfkbothdata}).

\vspace{1mm}
\noindent \textbf{Implementation details.}
We split training data into disjoint source and target data considering the fact that the persons that appear in the new camera after installation may or may not be seen before in the source cameras. That is, for Market1501 and MSMT17, we split the training data into 90\% of persons that are only seen by the source cameras and 10\% that are seen in both source cameras and the new target camera after the installation. Since there are much fewer persons in RAiD and WARD training set, we split the persons into 80\% source and 20\% target for those two datasets. For each dataset, we evaluate every source-target pair and average accuracy across all pairs. Furthermore, we average accuracy across all cameras as target. Note that the train and test set are kept disjoint in all our experiments.



\vspace{1mm}
\noindent \textbf{Results.} Figure~\ref{fig:singlecam} and \ref{fig:gfkbothdata} show the results. In all cases, our method outperforms all the compared methods. The most competitive methods are those of Adapt-GFK and Avg-Source that also use source metrics. For the remaining methods, we see the limitation of only using limited target data to compute the new metrics. For Market1501, we see that Avg-Source outperforms the Adapt-GFK baseline indicating the advantage of knowledge transfer from multiple source metric compared to one single best source metric as in \cite{panda2017unsupervised,panda2019adaptation}. However, our approach still outperforms the Avg-Source baseline by a margin of 20.60\%, 13.81\%, 2.01\% and 1.07\% in Rank-1 accuracy on RAiD, WARD, Market1501 and MSMT17, respectively, validating our implications of Theorem~\ref{thm1}. Furthermore, we observe that even without accessing the source training data that was used for training the network before adding a new camera, our method outperforms the GFK based methods that use all the source data in their computations (see Figure~\ref{fig:gfkbothdata}). To summarize, the experimental results show that our method performs better on both small and large camera networks with limited supervision, as it is able to adapt multiple source metrics through reducing negative transfer by dynamically weighting the source metrics.

\begin{figure}[t]
    \centering
    \includegraphics[width=0.35\textwidth]{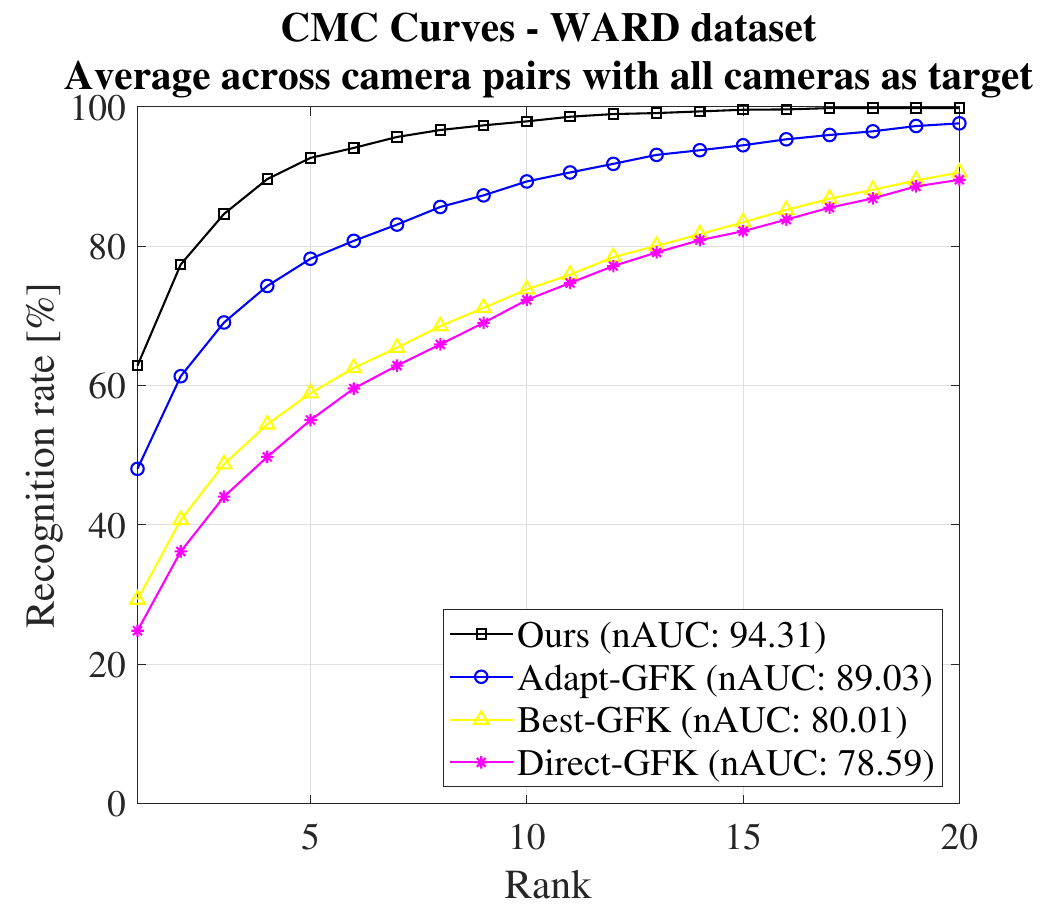}
    \caption{CMC curves averaged over all the target camera combinations, introduced one at a time, on the WARD dataset. Note that both old and new source data are used for calculation of GFK. Best viewed in color.}
    \label{fig:gfkbothdata}
\end{figure}

\begin{figure*}[ht]
\centering
\begin{subfigure}{0.3\textwidth}
\includegraphics[width=\textwidth]{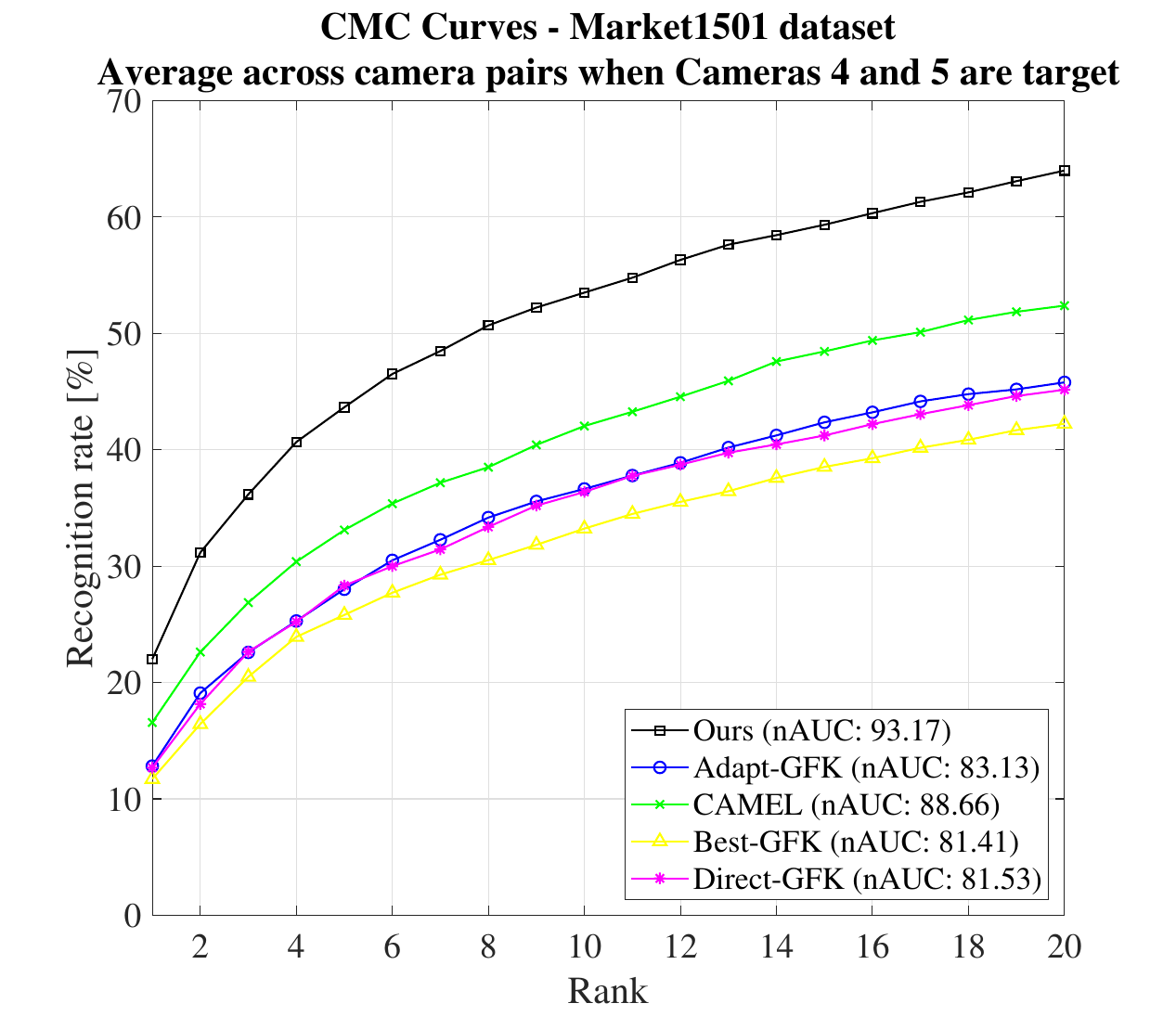}
\caption{}
\end{subfigure}
\begin{subfigure}{0.3\textwidth}
\includegraphics[width=\textwidth]{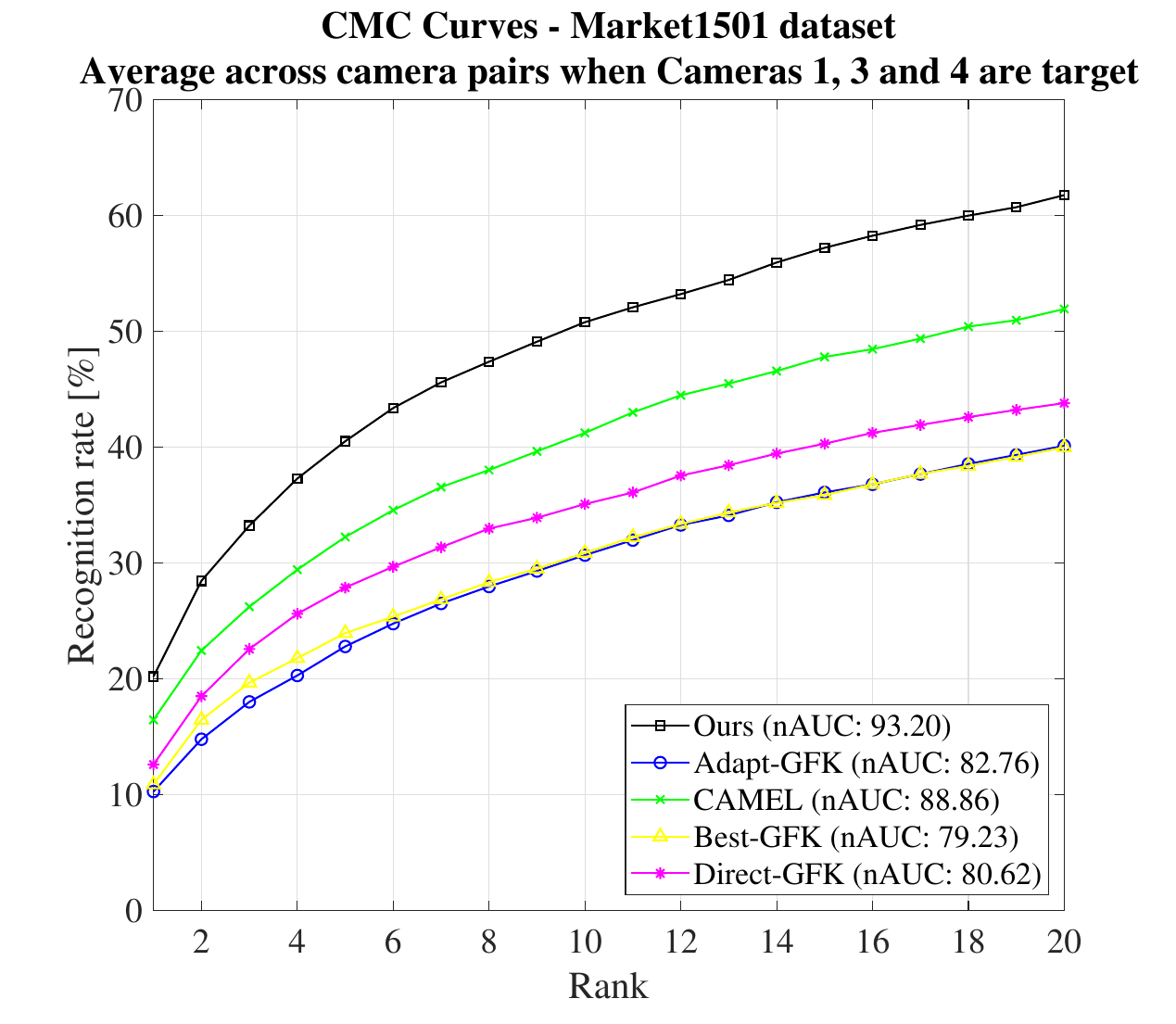}
\caption{}
\end{subfigure}
\begin{subfigure}{0.3\textwidth}
\includegraphics[width=\textwidth]{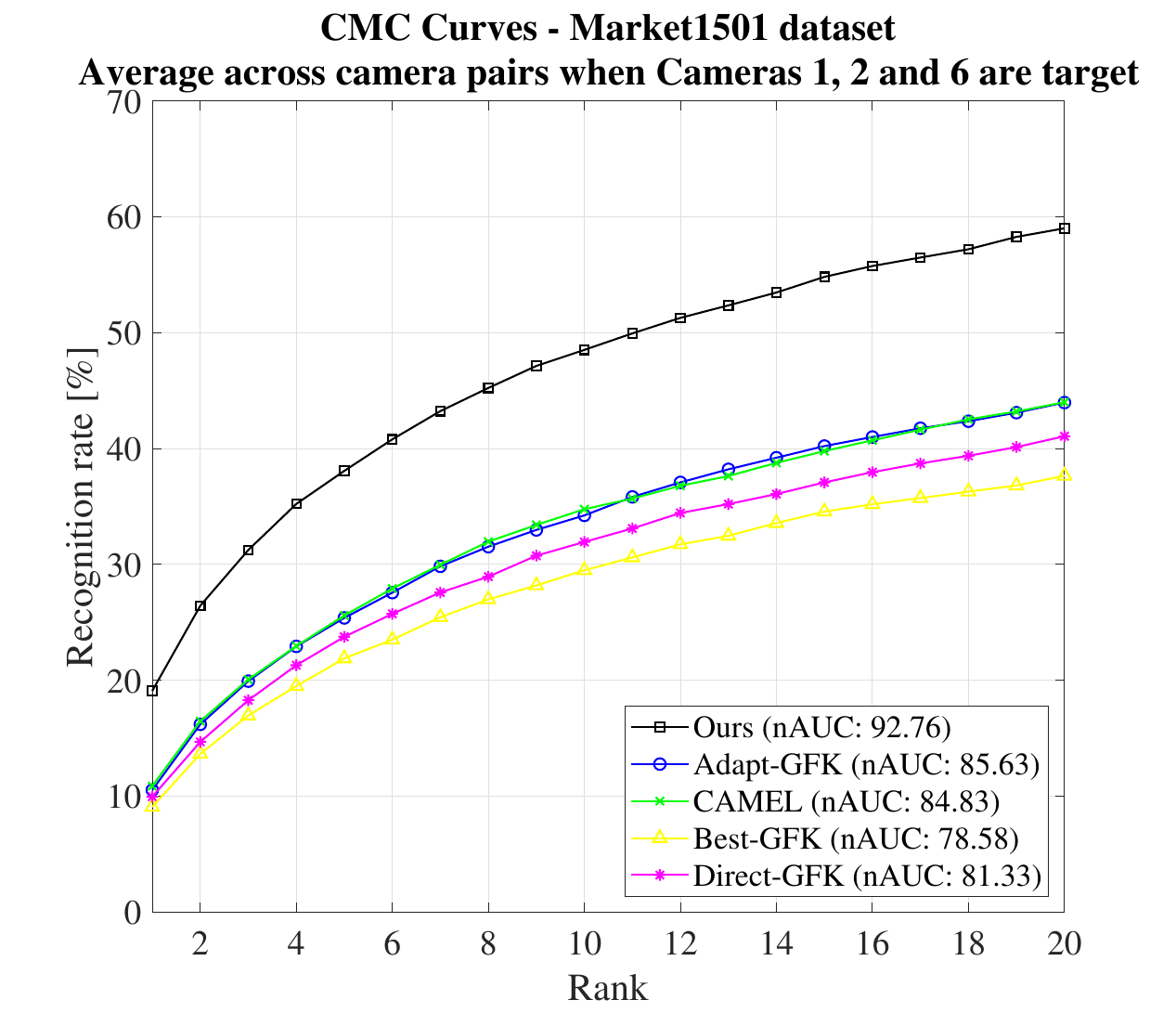}
\caption{}
\end{subfigure}
\caption{CMC curves averaged across target cameras on Market1501 dataset. (a) and (b) show results while adding two and three cameras in parallel, (c) show result while adding three cameras sequentially one after another. Best viewed in color.}
\label{fig:multicam} \vspace{3mm}
\end{figure*}

\subsection{On-boarding Multiple New Cameras} 

We perform this experiment on Market1501 dataset using the same strategy as in Section~\ref{onboardsinglecam} and compare our results with other methods while adding multiple target cameras to the network, either continuously or in parallel.   

\vspace{1mm}
\textbf{Parallel On-boarding of Cameras:} We randomly select two or three cameras as target while keeping the remaining as source. All the new target cameras are tested against both source cameras and other target cameras. The results of adding two and three cameras in parallel (at the same time) are shown in Figure~\ref{fig:multicam} (a) and (b), respectively. In both cases, our method outperforms all the compared methods with an increasing margin as rank increases. 
We outperform the most competitive CAMEL in Rank-1 accuracy by 5.45\% and 3.73\%, while adding two and three cameras respectively. 
Furthermore, our method better adapts source metrics since it has the capability of assigning zero weights to the metrics that do not generalize well over target data. Meanwhile, Adapt-GFK has a high probability of using the outlier source metrics in the presence of fewer available source metrics, which causes negative transfer. This has been shown in Figure~\ref{fig:multicam} where GFK based methods are performing worse than CAMEL, which is computed just with limited supervision without using any source metrics. 

\vspace{1mm}
\textbf{Sequential On-boarding of Cameras:} For this experiment, we randomly select three target cameras that are added sequentially.
A target camera is tested against all source cameras and previously added target cameras.
The results are shown in Figure~\ref{fig:multicam} (c). Similar to parallel on-boarding, our methods outperforms compared methods by a large margin. In this setting, we outperform CAMEL by 8.22\% in Rank-1 accuracy. Additionally, compared to all GFK-based methods, the Rank-1 margin is kept constant at 10\% for both parallel and sequential on-boarding. These results show the scalability of our proposed method while adding multiple cameras to a network, irrespective of whether they are added in parallel or sequentially. 

\subsection{Different Labeled Data in New Cameras}
We perform this experiment to show the implications of Theorem~\ref{thm2} by using different percentages of labeled target data (10\%, 20\%, 30\%, 50\%, 75\% and 100\%) in our method. 
We compare with a widely used KISS metric learning (KISSME) \cite{koestinger2012large} algorithm and show the difference in Rank-1 accuracy as a function of labeled target data. 
Figure~\ref{fig:ablation} (a) shows the results. At only 10\% labeled data, the difference between our method and KISSME \cite{koestinger2012large} is almost 30\%; however, as we add more labeled data, the Rank-1 accuracy becomes equivalent for the two methods at 100\% labeled data. This confirms the implications of Theorem~\ref{thm2}, where we showed that with increasing labeled target data, the effect of source metrics in learning becomes negligible.

\begin{figure*}[ht]
\centering
\begin{subfigure}{0.3\textwidth}
\includegraphics[width=\textwidth]{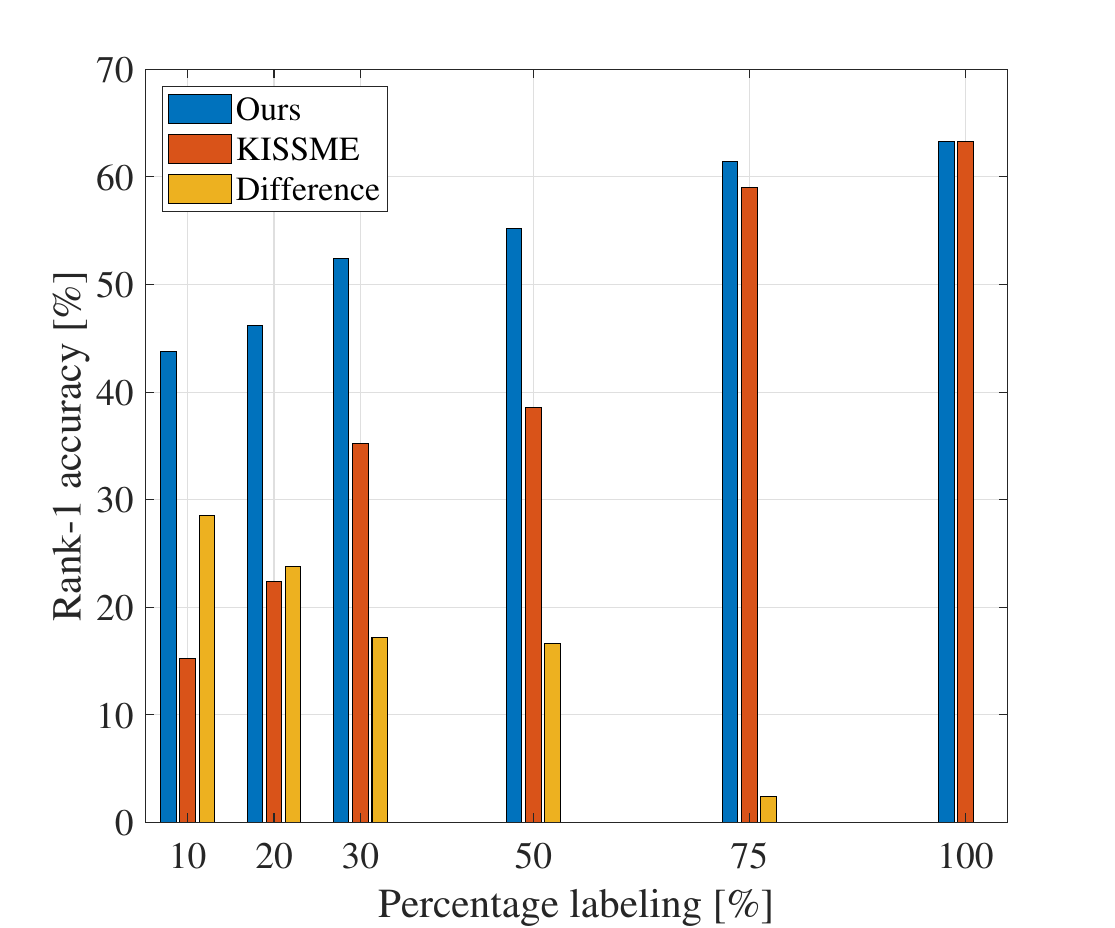}
\caption{}
\end{subfigure}
\begin{subfigure}{0.3\textwidth}
\includegraphics[width=\textwidth]{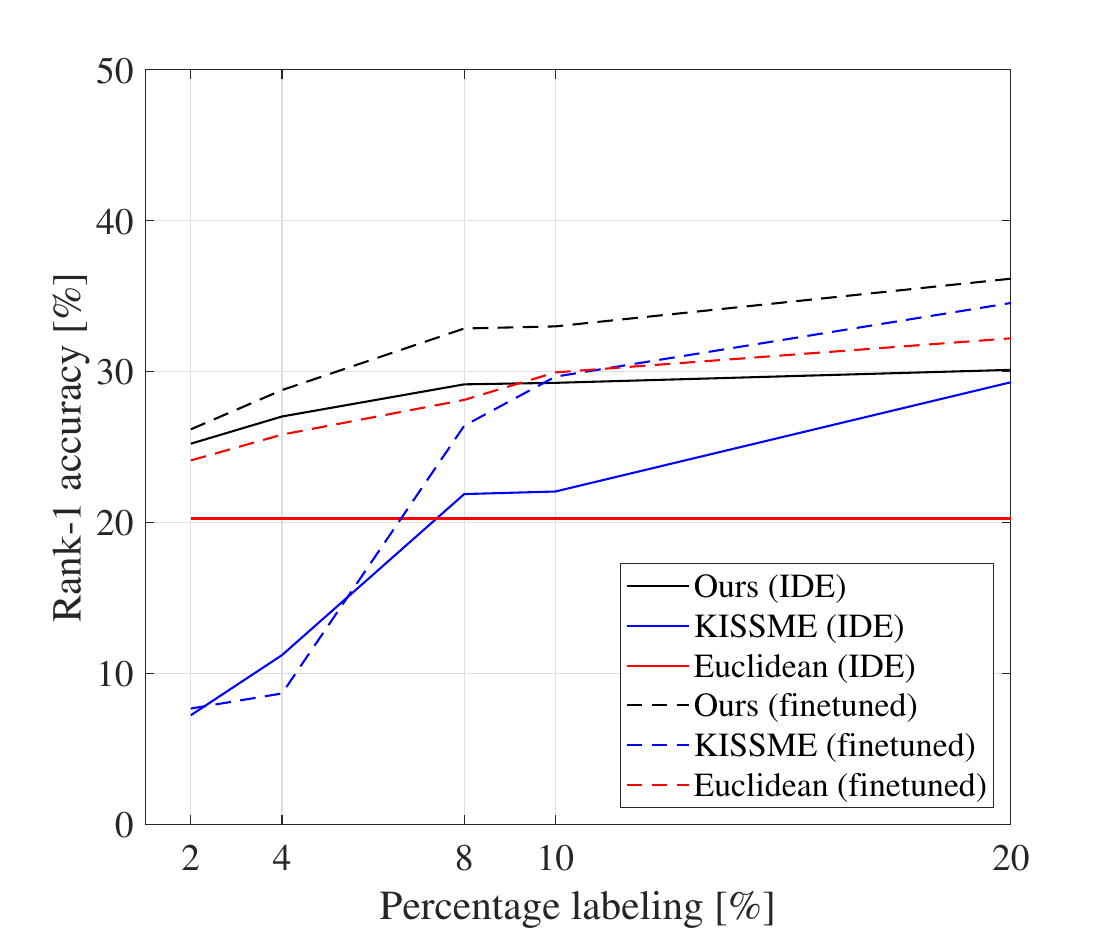}
\caption{}
\end{subfigure}
\begin{subfigure}{0.3\textwidth}
\includegraphics[width=\textwidth]{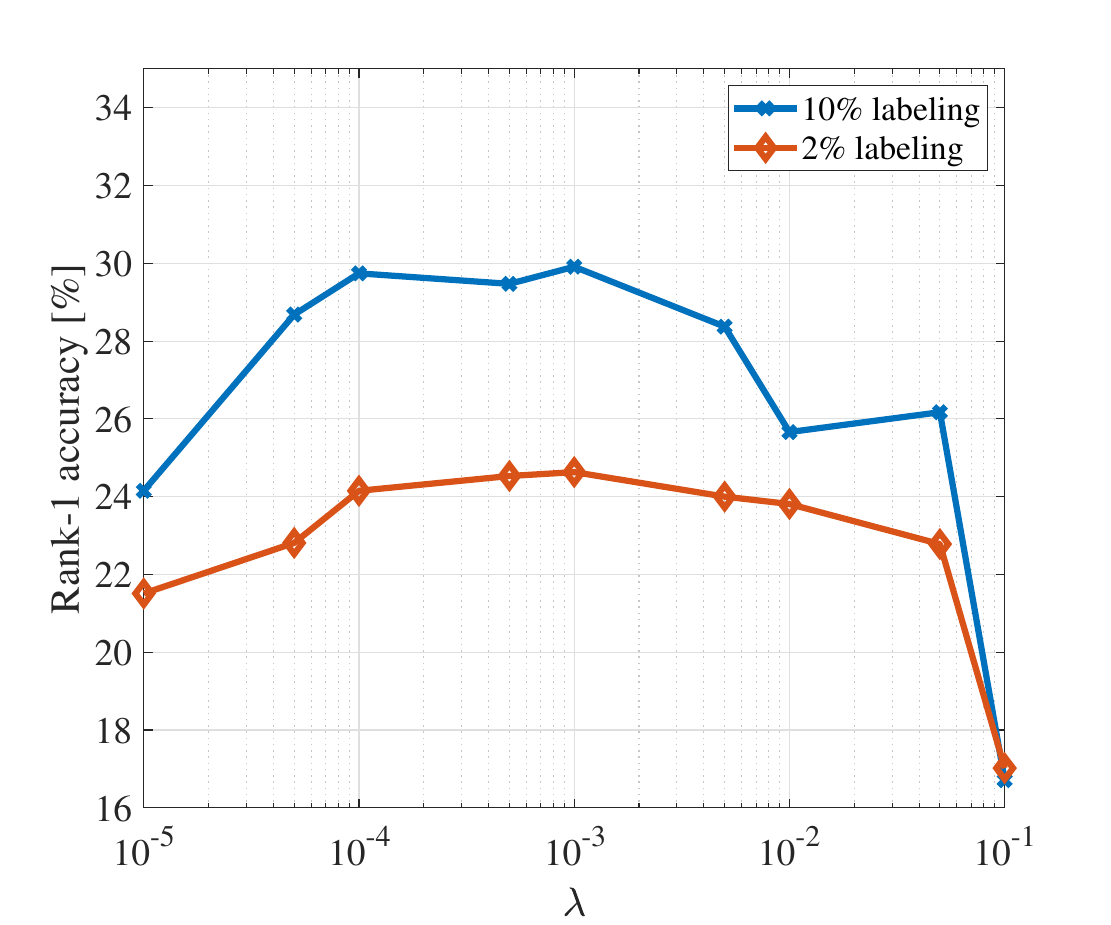}
\caption{}
\end{subfigure} \vspace{2mm}
\caption{(a) Effect of different percentage of target labelling on WARD dataset for justifying Theorem~\ref{thm2}, (b) Analysis of our method with deep features trained on source camera data in Market1501 dataset with 6th camera as target, (c) Sensitivity of $\lambda$ on the Rank-1 performance tested using deep features in Market1501 with 6th camera as target.  Best viewed in color. }
\label{fig:ablation} \vspace{2mm}
\end{figure*}

\subsection{Finetuning with Deep Features}
\label{subsec:deepfeat}

This section shows the strength of our method while comparing with CNN features extracted from a network trained on the source data (we train a ResNet50 model \cite{he2016deep}, pretrained on the Imagenet dataset). Without transfer learning, we have two options: (a) directly use the source model to extract features in the target and do matching based on Euclidean/KISSME metric (IDE), (b) finetune the source model using limited target data and then extract features to do matching using Euclidean/KISSME (finetuned). 
We compared these baselines with our method with different percentage of labeling on Market1501 dataset, where the pairwise metrics are computed using the source features extracted from the model without any finetuning.
We use those source metrics along with the target features, extracted before (Ours(IDE)) and after finetuning the source model (Ours(finetuned)). 
Please see supplementary material for more details. 
Figure~\ref{fig:ablation} (b) shows the results. 
Ours(IDE) outperforms Euclidean(IDE) by a margin of 10\% on Market with 20\% of labeled target data. The difference between Ours(finetuned) and Euclidean/KISSME (finetuned) is more noticeable with less labeled data and it becomes smaller with increase in labeled target data (Theorem~\ref{thm2}). However Ours(finetuned) consistently outperforms all the other baselines for up to 20\% labeling.


\vspace{-0.5 mm}
\subsection{Parameter Sensitivity}

We perform this experiment to study the effect of $\lambda$ in optimization~(\ref{opt:main_opt}) for a given percentage of labeled target data. Figure~\ref{fig:ablation} (c) shows the Rank-1 accuracy of our proposed method for different values of $\lambda$. 
From optimization~\ref{opt:main_opt}, when $\lambda \rightarrow \infty$ the left term can be neglected resulting in optimal $M$ and $\beta$ to be zero. However, when $\lambda \rightarrow 0$, the regularization term is neglected resulting in no transfer. We can see from Figure~\ref{fig:ablation} (c) that there is an operating zone of $\lambda$ (e.g., in the range of $10^{-4}$ to $10^{-2}$), that is neither too high nor too low for useful transfer from source metrics.



\vspace{-2 mm}
\section{Conclusions}
\vspace{-2 mm}
We addressed a critically important problem in person re-identification which has received little attention thus far - how to quickly on-board new cameras into an existing camera network. We showed this can be addressed effectively using hypothesis transfer learning using only learned source metrics and a limited amount of labeled data collected after installing the new camera(s). We provided theoretical analysis to show that our approach minimizes the effect of negative transfer through finding an optimal weighted combination of multiple source metrics. We showed the effectiveness of our approach on four standard datasets, significantly outperforming several baseline methods.

\vspace{-3mm}
\noindent \paragraph{Acknowledgements.} This work was partially supported by ONR grant N00014-19-1-2264 and NSF grant 1724341.

{\small
\bibliographystyle{ieee_fullname}
\bibliography{main}
}
\onecolumn





\newpage
\input{supplymentary}

\end{document}

%% file: supplymentary.tex





\begin{center}
\textbf{\Large{Camera On-boarding for Person Re-identification using Hypothesis Transfer Learning \\ (Supplementary Material)}}
\end{center}

\begin{table} [h]
\begin{center}
\begin{tabular}{ |p{2cm}||p{12cm}|  }
\hline
 
 Page Number    
 & Content \\
 \hline
 \hline
 {\color{blue}[12]}   & Dataset descriptions \\
 \hline
 {\color{blue}[12]}   & Detailed description of the optimization steps \\
 \hline
 {\color{blue}[15]}  & Proof of theorems from the main paper \\
 \hline 
 {\color{blue}[17]}  & On-boarding a single new camera (camera-wise CMC curves)\\
 \hline
 {\color{blue}[21]}   & On-boarding multiple new cameras (camera-wise CMC curves) \\
 \hline
 {\color{blue}[21]}   & Additional Experiments \\
 \hline
 {\color{blue}[22]}  & Finetuning with deep features \\
 \hline
\end{tabular}
\end{center}
\caption{Supplementary Material Overview.}
\end{table}

\newpage
\setcounter{section}{0}
\section{Dataset Descriptions}
This section contains detailed descriptions of the datasets used in our experiemnts (see Figure~\ref{datasets} for sample images).

\subparagraph{WARD~\cite{martinel2012re}} was collected from three outdoor cameras. The dataset contains 4,786 images of 70 different persons and includes variations in illumination.

\subparagraph{RAiD~\cite{das2014consistent}} was collected from four cameras; two indoor and two outdoor. 6,920 images were captured of 43 different persons. However, two of these persons were only seen by two of the four cameras. As a result of having both indoor and outdoor cameras, the dataset includes large illumination and viewpoint variations.

\subparagraph{Market1501~\cite{zheng2015scalable}} was collected from six cameras and used a Deformable Part Model~\cite{felzenszwalb2009object} to annotate images. This resulted in 32,668 images of 1,501 different persons, but also 2,793 ``distractors'' that are badly drawn bounding boxes. The dataset includes variations in both detection precision, resolution and viewpoint.

\subparagraph{MSMT17~\cite{wei2018person}} is the largest person re-identification dataset to date, and contains images collected by no more than 15 cameras; 3 indoor and 12 outdoor. Data was collected over the course of four different days in a month, and Faster RCNN~\cite{ren2015faster} was using for bounding box detection, resulting in 126,441 images of 4,101 different persons. Due to the diversity in data collection, this dataset contains large variations in illumination and viewpoint.
\begin{figure}[h]
\centering
\begin{subfigure}{0.201\textwidth}
\includegraphics[width=\textwidth]{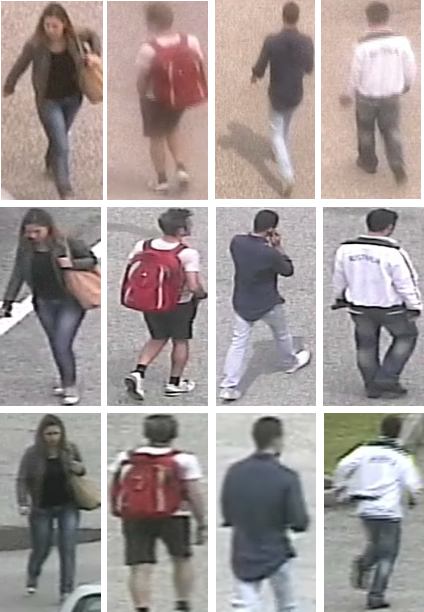}
\caption{WARD}
\end{subfigure}
\hspace{0.33cm}
\begin{subfigure}{0.2\textwidth}
\includegraphics[width=\textwidth]{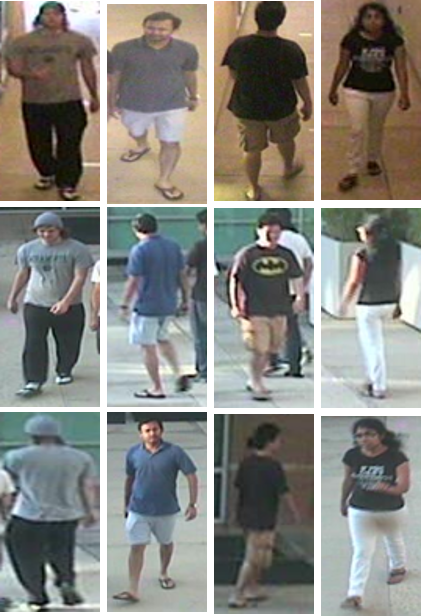}
\caption{RAiD}
\end{subfigure} 
\hspace{0.33cm}
\begin{subfigure}{0.202\textwidth}
\includegraphics[width=\textwidth]{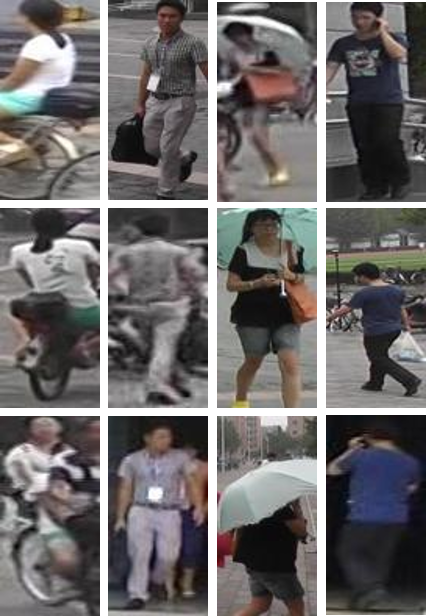}
\caption{Market1501}
\end{subfigure}
\hspace{0.33cm}
\begin{subfigure}{0.205\textwidth}
\includegraphics[width=\textwidth]{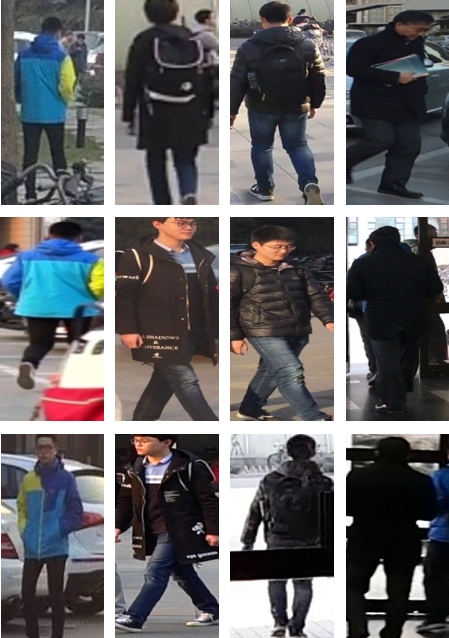}
\caption{MSMT17}
\end{subfigure}
\caption{A total of 48 Sample images from the 4 datasets used in our experimentation. In each row 4 different persons are shown whereas for each column 3 different views of the same person from 3 different cameras are shown. We can see the that across cameras, the viewpoint of the same person is very diverse because of change in illumination condition or occlusion.}
\label{datasets}
\end{figure}


\section{Detailed Description of the Optimization Steps}
 In this section we will rigorously discuss all the necessary derivations of the steps of our proposed algorithm that could not be shown in the main paper due to space constraint. We first present the notations that we will use throughout this section. \\
\noindent\textbf{Notations:}
\begin{itemize}
     \item $\frac{1}{n_s}\sum\limits_{(i,j) \in S}{ x_{ij}x_{ij}^\top}=\Sigma_S$
      \item $\frac{1}{n_d}\sum\limits_{(i,j) \in D}{ x_{ij}x_{ij}^\top}=\Sigma_D$
      \item $\mathcal{C}_1=\{M \mid \frac{1}{n_d}\sum\limits_{(i,j) \in D}(x_{ij}^\top M x_{ij})-b \geq 0\}$
      \item $\mathcal{C}_2=\{M \mid M\succeq 0\}$
      \item $\mathcal{C}_3=\{\beta \mid \|\beta\|_2 \leq 1\}$
      \item $\Pi_\mathcal{C}(X)=\underset{\hat{X} \in \mathcal{C}}{\text{minimize}}~~ \frac{1}{2}\|\hat{X}-X\|_F^2$
      \item $f(M,\beta)=\frac{1}{n_s}\sum\limits_{(i,j) \in S}{ x_{ij}^\top M x_{ij}} + \lambda \|M-\sum_{j=1}^{N}\beta_j M_j\|_F^2$
\end{itemize}

The proposed optimization problem in the main paper is defined below.
\setcounter{equation}{0}
\begin{mini}
{M,\beta}{\frac{1}{n_s} \sum_{(i,j) \in S}{ x_{ij}^\top M x_{ij}} +  \lambda \|M-\sum_{j=1}^{N}\beta_jM_j\|_F^2}
{}{}
\label{opt:main_opt_supp}
\addConstraint {\frac{1}{n_d} \sum_{(i,j) \in D}(x_{ij}^\top M x_{ij})-b}{\geq 0 ,}{\ M \succeq 0}
\addConstraint{\beta}{ \geq 0,}{\ \|\beta\|_2}{ \leq 1}
\end{mini}

\noindent\textbf{Step 1: Gradient w.r.t $M$ with fixed $\beta$.}

\begin{equation}
\begin{split}
    \nabla_M(f(M,\beta)) &= \frac{1}{n_s}\sum_{(i,j) \in S}x_{ij} x_{ij}^\top +2 \lambda (M-\sum_{j=1}^{N}\beta_jM_j)\\
    &= \Sigma_S+2 \lambda (M-\sum_{j=1}^{N}\beta_jM_j)
    \end{split}
\end{equation}

\noindent\textbf{Step 2: Projection of $M$ onto $\mathcal{C}_1$ and $\mathcal{C}_2$.}

This can be done by solving a constrained optimization problem.
\begin{alignat*}{4}
\Pi_{\mathcal{C}_1}(M) &= \text{arg}\ \underset{\hat{M}}{\text{min}}~~&& \frac{1}{2} \|\hat{M}-M\|_F^2 \\
& \text{Subject to} && \frac{1}{n_d}\sum\limits_{(i,j) \in D}(x_{ij}^\top \hat{M} x_{ij})-b \geq 0
\end{alignat*} 
We can write the lagrangian as follows,
\begin{equation}
\mathcal{L}(\hat{M},\psi)= \frac{1}{2} \|\hat{M}-M\|_F^2 + \psi(b-\frac{1}{n_d}\sum\limits_{(i,j) \in D}x_{ij}^\top \hat{M}x_{ij})
\label{lagrange}
\end{equation}
The KKT conditions for this problem are:
\begin{enumerate}
    \item  
    \begin{equation*}
    \nabla_{\hat{M}} \mathcal{L}(\hat{M},\psi)|_{\hat{M}=\hat{M}^\star}=0 \\
    \implies  (\hat{M}^\star-M)-\frac{\psi}{n_d} \sum_{(i,j) \in D} x_{ij}x_{ij}^\top=0 \\
    \implies  (\hat{M}^\star-M)-\psi \Sigma_D=0 \\
    \implies  \hat{M}^\star= (M+\psi \Sigma_D) \\
    \end{equation*}
    \item $\psi^\star(b-\frac{1}{n_d}\sum\limits_{(i,j) \in D}x_{ij}^\top \hat{M}^\star x_{ij}) \geq 0$
    \item $\psi^\star \geq 0$
\end{enumerate}
The optimization problem is convex, so strong duality should hold. So, we put the value of $\hat{M}^\star$ from KKT condition $1$ in the equation~\eqref{lagrange} to get the dual objective function as follows,
\begin{equation}
    \begin{split}
    g(\psi)= \mathcal{L}(\hat{M}^\star,\psi) &=\frac{1}{2} \|M+\psi \Sigma_D-M\|_F^2 + \psi\big(b-\frac{1}{n_d}\sum_{(i,j) \in D}x_{ij}^\top (M+\psi \Sigma_D) x_{ij}\big) \\
    & = \frac{1}{2} \psi^2 \|\Sigma_D\|_F^2 + \psi(b-\frac{1}{n_d}\sum\limits_{(i,j) \in D}x_{ij}^\top M x_{ij}) -\frac{\psi^2}{n_d} \sum\limits_{(i,j) \in D}x_{ij}^\top \Sigma_D x_{ij} \\
    & =  \frac{1}{2} \psi^2 \|\Sigma_D\|_F^2 + \psi(b-\frac{1}{n_d}\sum\limits_{(i,j) \in D}x_{ij}^\top M x_{ij}) -\frac{\psi^2}{n_d} \sum\limits_{(i,j) \in D} \mathrm{trace}(x_{ij}^\top \Sigma_D x_{ij}) \\
    & =  \frac{1}{2} \psi^2 \|\Sigma_D\|_F^2 + \psi(b-\frac{1}{n_d}\sum\limits_{(i,j) \in D}x_{ij}^\top M x_{ij}) -\frac{\psi^2}{n_d} \sum\limits_{(i,j) \in D} \mathrm{trace}( \Sigma_D x_{ij}x_{ij}^\top) \\
    & =  \frac{1}{2} \psi^2 \|\Sigma_D\|_F^2 + \psi(b-\frac{1}{n_d}\sum\limits_{(i,j) \in D}x_{ij}^\top M x_{ij}) -\psi^2  \mathrm{trace}( \Sigma_D \frac{1}{n_d}\sum\limits_{(i,j) \in D}x_{ij}x_{ij}^\top) \\
     & =  \frac{1}{2} \psi^2 \|\Sigma_D\|_F^2 + \psi(b-\frac{1}{n_d}\sum\limits_{(i,j) \in D}x_{ij}^\top M x_{ij}) -\psi^2  \mathrm{trace}( \Sigma_D^\top \Sigma_D) \\
      & =  \frac{1}{2} \psi^2 \|\Sigma_D\|_F^2 + \psi(b-\frac{1}{n_d}\sum\limits_{(i,j) \in D}x_{ij}^\top M x_{ij}) -\psi^2 \|\Sigma_D\|_F^2  \\
      & =  -\frac{1}{2} \psi^2 \|\Sigma_D\|_F^2 + \psi(b-\frac{1}{n_d}\sum\limits_{(i,j) \in D}x_{ij}^\top M x_{ij})  \\
    \end{split}
    \label{dualfunc}
\end{equation}
To get the optimal $\psi^\star$ we have to maximize $g(\psi)$.
\begin{alignat*}{3}
     & g^\prime(\psi^\star) =  0 \\
    \implies &- \psi^\star  \|\Sigma_D\|_F^2 +  (b-\frac{1}{n_d}\sum\limits_{(i,j) \in D}x_{ij}^\top M x_{ij}) = 0 \\
    \implies &\psi^\star = \frac{(b-\frac{1}{n_d}\sum\limits_{(i,j) \in D}x_{ij}^\top M x_{ij})}{\|\Sigma_D\|_F^2}\\
\end{alignat*}
But also from KKT condition $(3)$, we know $\psi \geq 0$. Combining with the last equation we get 
\begin{equation}
    \psi^\star= \max\Bigg \{0,\frac{(b-\frac{1}{n_d}\sum\limits_{(i,j) \in D}x_{ij}^\top M x_{ij})}{\|\Sigma_D\|_F^2}\ \Bigg \}
\end{equation}

So, putting the value of $\psi^\star$, finally we can write the projection from KKT condition 1 as,
\begin{equation}
    \Pi_{C_1}(M)= M + \max\Bigg \{0,\frac{(b-\frac{1}{n_d}\sum\limits_{(i,j) \in D}x_{ij}^\top M x_{ij})}{\|\Sigma_D\|_F^2}\ \Bigg \} \Sigma_D
\end{equation}

projection onto $\mathcal{C}_2$ is standard, so we are not discussing it here.
\newpage
\noindent\textbf{Step 3: Gradient w.r.t $\beta$ with fixed $M$.}

\begin{equation}
\begin{split}
    f(M^{k+1},\beta) &= \frac{1}{n_s}\sum\limits_{(i,j) \in S}{ x_{ij}^\top M^{k+1} x_{ij}} + \lambda \|M^{k+1}-\sum_{j=1}^{N}\beta_j M_j\|_F^2 \\
    &= K+  \lambda \|M^{k+1}-\sum_{j=1}^{N}\beta_j M_j\|_F^2 \\
    &= K+  \lambda \mathrm{trace}\Big((M^{k+1}-\sum_{j=1}^{N}\beta_j M_j)^\top(M^{k+1}-\sum_{j=1}^{N}\beta_j M_j)\Big)  \\
    &= K+ \lambda \beta_i^2 \mathrm{trace}(M_i^\top M_i) -2 \lambda\beta_i  \mathrm{trace}( M_i^\top (M^{k+1}-\sum_{j=1,j \neq i}^{N}\beta_j M_j) )
    \end{split}
    \label{phibeta}
\end{equation}
$K$ is term which is independent of $\beta$. Now differentiating equation~\eqref{phibeta} w.r.t $\beta_i$ we get ,
\begin{equation}
    \nabla_{\beta_i}  f(M^{k+1},\beta) = 2 \lambda\beta_i  \mathrm{trace}(M_i^\top M_i) -2 \lambda \mathrm{trace}( M_i^\top (M^{k+1}-\sum_{j=1,j \neq i}^{N}\beta_j M_j) )=a_i
\end{equation}
So, derivative of $f(M^{k+1},\beta)$ w.r.t $\beta$ is given by,
\begin{equation}
  \nabla_{\beta}  f(M^{k+1},\beta)  = \begin{bmatrix} 
  a_1 & a_2 & \ldots & a_N \\
\end{bmatrix}^\top
\end{equation}

\noindent\textbf{Step 4: Projection of $\beta$ onto $\mathcal{C}_3$.}
 \begin{equation}
    \Pi_{\mathcal{C}_3} (\beta) = \max\Bigg\{0,\frac{\beta}{\max \{1,\|\beta\|_2 \}} \Bigg \}
\end{equation}
The intuition here is that, when the norm of $\beta$ is greater than $1$ then $\max \{1,\|\beta\|_2 \}=\|\beta\|_2$ which implies the normalization of $\beta$. Similarly when the norm of $\beta$ is lesser or equal to $1$ then $\max \{1,\|\beta\|_2 \}=1$, which means keeping the $\beta$ as it is since it already lies in the unit norm ball. The maximum with $0$ essentially denotes the projection of any vector within the unit norm ball to the first quadrant of that ball only.\\

\section{Proof of the Theorems}
As mentioned in the paper the optimization proposed by us can be written in the same format as \cite{perrot2015theoretical}
\begin{equation}
     \underset{M \succeq 0}{\text{minimize}}~~L_T(M)+\lambda \|M-M_S\|_F^2
     \label{perrotopt_supp}
\end{equation}
where $M_S= \sum_{j=1}^{N} \beta_j M_j$ and
\begin{equation}
   L_T(M)=\frac{1}{n_s}\sum_{(i,j) \in S}{ x_{ij}^\top M x_{ij}}+\mu^\star\big(b-\frac{1}{n_d}\sum_{(i,j) \in D}x_{ij}^\top M x_{ij}\big)
   \label{LTM_supp}
\end{equation}
\setcounter{thm}{0}
\begin{thm}
For the convex and $k$-Lipschitz loss  defined  in~\eqref{LTM_supp} the average bound can be expressed as
\begin{equation}
    \mathbb{E}_{T \sim \mathcal{D}_{\mathcal{T}^n}}[L_{\mathcal{D}_\mathcal{T}}(M^\star)] \leq L_{\mathcal{D}_\mathcal{T}}(\widehat{M_S}) + \frac{8k^2}{\lambda n},
    \label{ineq_supp}
\end{equation}
where $n$ is the number of target labeled example, $M^\star$ is the optimal metric computed from Algorithm {\color{red}$1$}, $\widehat{M_S}$ is the average of all source metrics defined as $\frac{\sum_{j=1}^{N} M_j}{N}$, $\mathbb{E}_{T \sim \mathcal{D}_{\mathcal{T}^n}}[L_{\mathcal{D}_\mathcal{T}}(M^\star)]$ is the expected loss by $M^\star$ computed over distribution $\mathcal{D}_\mathcal{T}$ and $L_{\mathcal{D}_\mathcal{T}}(\widehat{M_S})$ is the loss of average of source metrics computed over  $\mathcal{D}_\mathcal{T}$. 
\label{thm1_supp}
\end{thm}
\begin{proof}
If there is a single source metric is available for transfer , the proof has been shown in \cite{perrot2015theoretical}. In case of multiple metric for any fixed $\beta$, we can directly replace $M_S$ by $\sum_{j=1}^{N} \beta_j M_j$ in the \textbf{Theorem 2} in \cite{perrot2015theoretical} to get,
\begin{equation}
    \mathbb{E}_{T \sim \mathcal{D}_{\mathcal{T}^n}}[L_{\mathcal{D}_\mathcal{T}}(M^\star)] \leq L_{\mathcal{D}_\mathcal{T}}\big(\sum_{j=1}^{N} \beta_j M_j\big) + \frac{8k^2}{\lambda n}
    \label{ineqbeta}
\end{equation}

which is true $\forall \beta \in \mathcal{C}_3$. Where,
\begin{equation}
\beta= 
    \begin{bmatrix}
    \beta_1 &\beta_2 &\ldots &\beta_N
    \end{bmatrix}^\top \in \mathbb{R}^N
\end{equation}

Clearly without loss of generality we can write  $\beta=\beta^\prime$ where,
\begin{equation}
\beta^\prime= 
    \begin{bmatrix}
    \frac{1}{N} &\frac{1}{N} &\ldots &\frac{1}{N}
    \end{bmatrix}^\top \in \mathcal{C}_3
\end{equation}
since, $\beta^\prime \geq 0$ and $\|\beta^\prime\|_2 =\frac{1}{\sqrt{N}} \leq 1$. So, plugging $\beta^\prime$ in equation~\eqref{ineqbeta} we get equation~\eqref{ineq}, which completes the proof.
\end{proof}

\begin{thm}
With probability $(1-\delta)$, for any metric $M$ learned from Algorithm {\color{red}$1$} we have,
\begin{equation}
\begin{aligned}
    L_{\mathcal{D}_\mathcal{T}}(M) \leq &L_T(M) + \mathcal{O}\big(\frac{1}{n}\big) +  \Bigg(\sqrt{\frac{L_T(\sum_{j=1}^{N} \beta_{j} M_j)}{\lambda}}+ 
    &\|\sum_{j=1}^{N} \beta_j M_j\|_F \Bigg) \sqrt{\frac{\ln(\frac{2}{\delta})}{2n}},
\end{aligned}
\label{gbound_supp2}
\end{equation}
\label{thm2_supp}
where $L_{\mathcal{D}_\mathcal{T}}(M)$ is the loss over the original target distribution (true risk), $L_T(M)$ is the loss over the existing target data (empirical risk), and $n$ is the number of target samples. 
\end{thm}

\begin{proof}
In \cite{perrot2015theoretical}, $L_T(M)$ is defined as,
\begin{equation}
    L_T(M)=\frac{1}{n^2} \sum\limits_{(z_i,z_j)\in T} l(M,z_i,z_j)
\end{equation}
\end{proof}
The authors in  \cite{perrot2015theoretical} have used a specific loss for analysis,
\begin{equation}
    l(M,z_i,z_j)=[yy^\prime((z_i-z_j)^\top M (z_i-z_j)-\gamma_{yy^\prime})]_+
\end{equation}
For our case,
\begin{equation}
\begin{split}
   L_T(M) &=\frac{1}{n_s}\sum_{(i,j) \in S}{ z_{ij}^\top M z_{ij}}+\mu^\star\big(b-\frac{1}{n_d}\sum_{(i,j) \in D}z_{ij}^\top M z_{ij}\big)\\
   &=\frac{1}{(n_s+n_d)}\frac{(n_s+n_d)}{n_s}\sum_{(i,j) \in S}{ z_{ij}^\top M z_{ij}}+\frac{\mu^\star b (n_s+n_d)}{(n_s+n_d)}-\frac{\mu^\star (n_s+n_d)}{n_d}.\frac{1}{(n_s+n_d)}\sum_{(i,j) \in D}z_{ij}^\top M z_{ij}\\
  &=  \frac{1}{n^2}\sum\limits_{(i,j) \in T}(\zeta_{ij}(z_i-z_j)^\top M (z_i-z_j) +\gamma)
\end{split}
\end{equation}
In our case we took similar and dissimilar pairs in equal number. So, for our case $n_s=n_d=\frac{n^2}{2}$ which implies $(n_s+n_d)=n^2$.
Also, $\zeta_{ij}=(1+\frac{n_d}{n_s})=2$ if $(i,j) \in S$ and 
$\zeta_{ij}=-\mu ^\star(1+\frac{ n_s}{n_d})=-2\mu^\star$ if $(i,j) \in D$ are soft labels. Also $\gamma=  \mu^\star b (n_s+n_d)=\mu^\star b n^2$. 
so for our case,
\begin{equation}
    l(M,z_i,z_j)=(\zeta_{ij}(z_i-z_j)^\top M (z_i-z_j) +\gamma)
\label{smallloss}
\end{equation}
Also unlike \cite{perrot2015theoretical} our source metric is defined as $M_S=\sum_{j=1}^{N} \beta_j M_j$. 
With the loss in equation~\eqref{smallloss} if we follow the exact same steps as in proof of the \textbf{Lemma 2} of \cite{perrot2015theoretical} then we will end up with the fact that our proposed loss is $(\sigma,m)$ admissible with $m=2(1+\mu^\star)\underset{x,x^\prime}{\text{max}}\|x-x^\prime\|_2^2\Bigg(\sqrt{\frac{L_T(\sum_{j=1}^{N} \beta_{j} M_j)}{\lambda}}+ 
    \|\sum_{j=1}^{N} \beta_j M_j\|_F \Bigg)$ and $\sigma=0$.
    
Now putting these values of $\sigma$ and $m$ in the equation of inequality of \textbf{Theorem 4} of   \cite{perrot2015theoretical} which is,\\
\begin{equation}
\begin{aligned}
    L_{\mathcal{D}_\mathcal{T}}(M) \leq &L_T(M) + \mathcal{O}\big(\frac{1}{n}\big) +  (4\sigma+2m+c) \sqrt{\frac{\ln(\frac{2}{\delta})}{2n}},
\end{aligned}
\label{gbound_supp1}
\end{equation}
and ignoring $c$ and the constant factor  which are not functions of source metrics or their weights we conclude our proof.
\subsection{Finding lipschitz constant for our loss}
\noindent\textbf{Goal:} Our goal is to show the $k$ in equation~\eqref{ineq} has a finite value.
According to the definition the loss $l(M,x,x^\prime)$ is $k$-lipschitz with respect to its first argument if for any pair of matrices $M$ and $M^\prime$ and pair of samples $x$ and $x^\prime$ we have the inequality as follows for a finite non-negative $k$ ($0 \leq k < \infty$)
\begin{equation}
  |l(M,x,x^\prime)-l(M^\prime,x,x^\prime)| \leq k\|M-M^\prime\|_F  
  \label{klip}
\end{equation}
\begin{lemma}
The loss defined in equation~\eqref{smallloss} is $k$-lipschitz with $k= 2 \max\left(1,\mu^\star\right)\underset{x,x^\prime}{\max}\|x-x^\prime\|_2^2$
\end{lemma}
\begin{proof}
\begin{equation}
\begin{split}
|l(M,x_i,x_j)-l(M^\prime,x_i,x_j)| &\leq  
|(\zeta_{ij}(x_i-x_j)^\top M (x_i-x_j) +\gamma)-(\zeta_{ij}(x_i-x_j)^\top M^\prime (x_i-x_j) +\gamma)| \\
& \leq |\zeta_{ij} (x_i-x_j)^\top (M-M^\prime) (x_i-x_j)|\\
& \leq \max \left(|\zeta_{ij}|\right)|(x_i-x_j)^\top (M-M^\prime) (x_i-x_j)| \\
& \leq \max \left(2,2\mu^\star \right)|(x_i-x_j)^\top (M-M^\prime) (x_i-x_j)| \\
& \leq 2 \max \left(1,\mu^\star \right)\|x_i-x_j\|_2^2
\|M-M^\prime\|_F \\
& \leq 2 \max\left(1,\mu^\star\right)\underset{x,x^\prime}{\max}\|x-x^\prime\|_2^2 \|M-M^\prime\|_F
\end{split}
\end{equation}
Comparing this inequality with eq.~\eqref{klip} we get $k= 2 \max\left(1,\mu^\star\right)\underset{x,x^\prime}{\max}\|x-x^\prime\|_2^2$, which is clearly non-negative and finite.
\end{proof}

\section{On-boarding a Single New Camera}
This section covers the camera wise experimental results of on-boarding a single new camera (See Figure~(\ref{fig:singlecamwardold},\ref{fig:singlecamraid},\ref{fig:singlecammarket},\ref{fig:singlecammsmt}). We show for each dataset the camera wise CMC curves that are averaged to a single CMC curve in the main paper. We also showed the comparison of GFK based methods in their original setting where source data is used during target adaptation in WARD dataset (See Figure~\ref{fig:singlecamwardoldnew}).

\begin{figure}[H]
\centering
\large \underline{Camera wise CMC curves for WARD dataset}\par\medskip
\begin{subfigure}{0.3\textwidth}
\includegraphics[width=\textwidth]{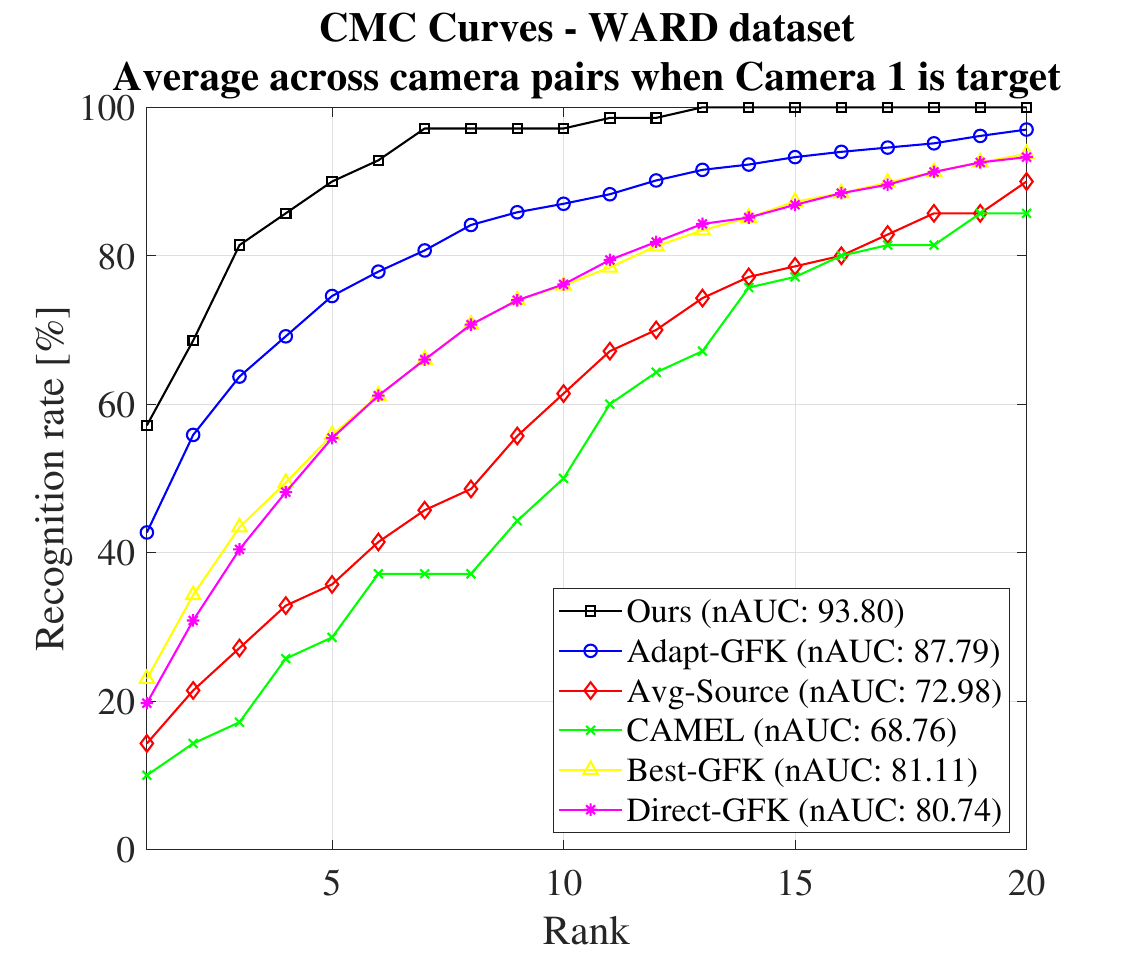}
\caption{}
\end{subfigure}
\begin{subfigure}{0.3\textwidth}
\includegraphics[width=\textwidth]{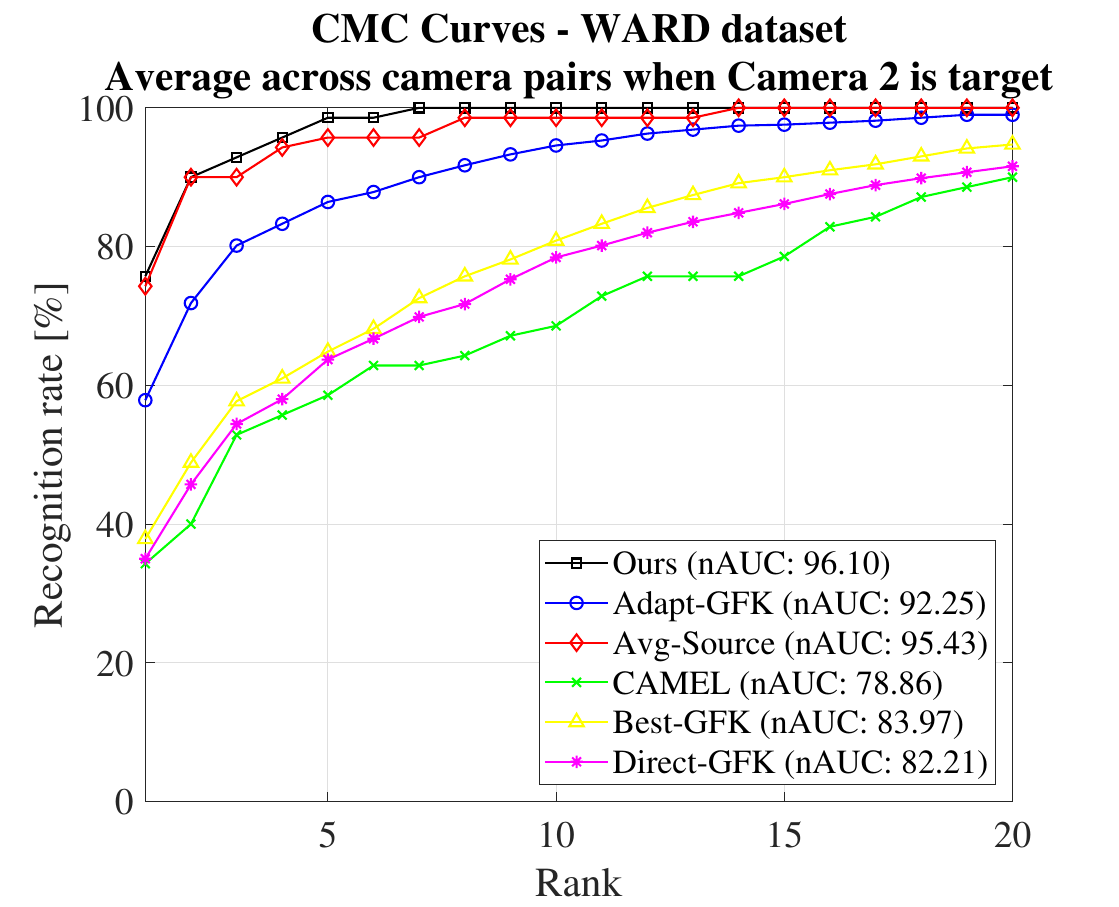}
\caption{}
\end{subfigure}
\begin{subfigure}{0.3\textwidth}
\includegraphics[width=\textwidth]{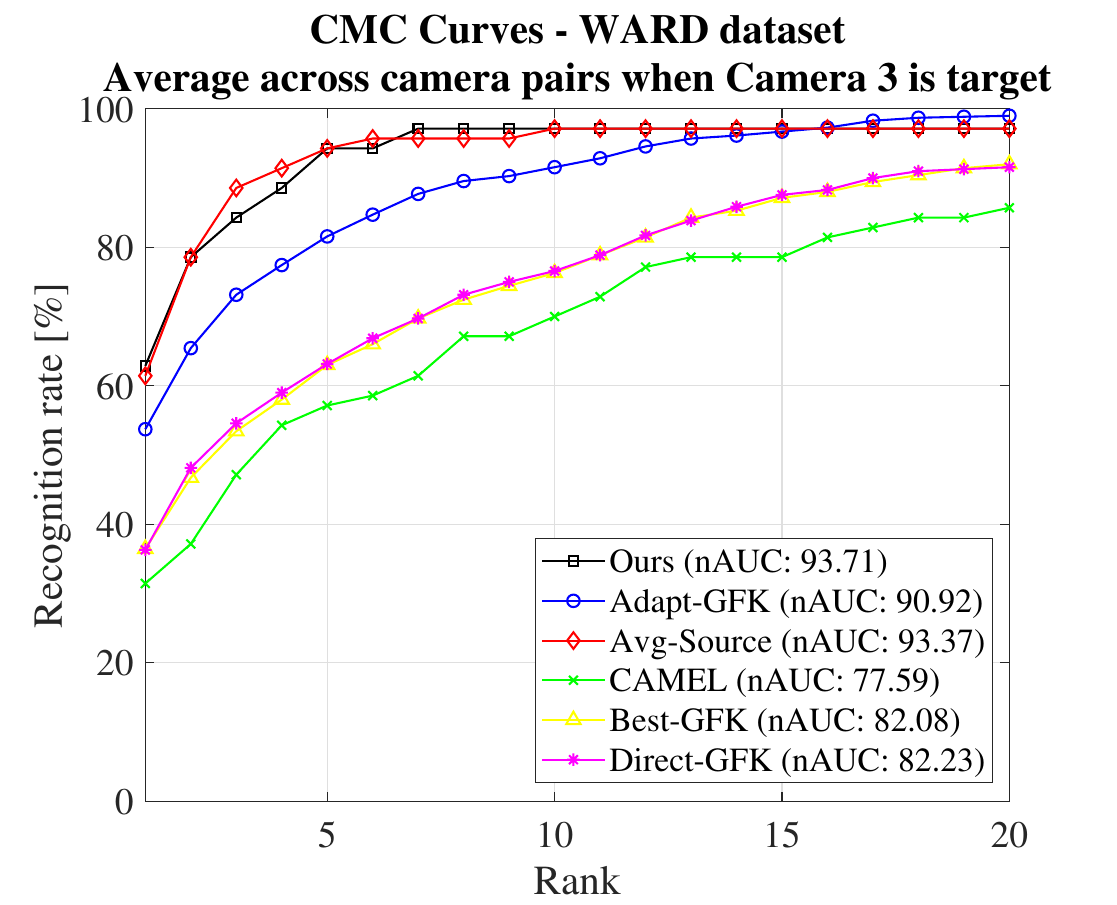}
\caption{}
\end{subfigure}
\caption{CMC curves for WARD\cite{martinel2012re} with 3 cameras. In this experiment each camera is shown as target while other two cameras served as source. The percentage label of new persons between the new target camera and the existing source cameras is taken to be 20\% in this case. The most competitive method here is Adapt-GFK which is outperformed by our method in nAUC with margins 6\%, 3.5\% and 2.79\% for camera 1,2 and 3 as target (plot a, b and c) respectively. In this case Adapt-GFK is calculated using the GFK matrix calculated by only using the limited labelled target data after the installation of new camera. Moreover for camera 1 as target (plot (a)) our method outperforms Adapt-GFK by a large rank-1 margin of almost 16\%. Notable thing in this case is that there is only one source metric available for this dataset which is also handled by our multiple source metric transfer algorithm efficiently. Our method significantly outperform the semisupervised method CAMEL for all the plots which shows the strength of our method when a little target labeled data availabe. Also, our method outperforms Avg-Source for all the plots which is a proof of implication of Theorem~\ref{thm1}. }
\label{fig:singlecamwardold}
\end{figure}

\begin{figure}[ht]
\centering
\large \underline{Camera wise CMC curves for WARD dataset} \\ \large \underline{(GFK computed for other relevant methods using old source data and new target data)}\par\medskip
\begin{subfigure}{0.3\textwidth}
\includegraphics[width=\textwidth]{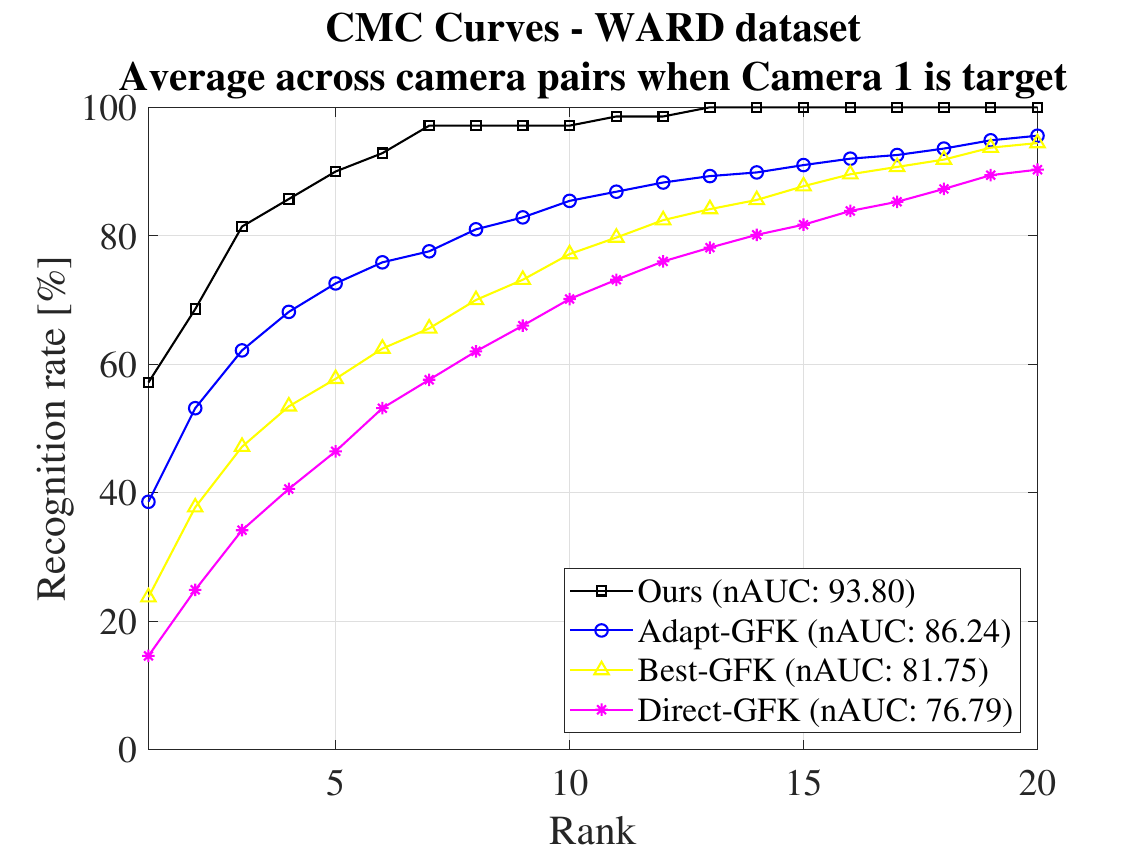}
\caption{}
\end{subfigure}
\begin{subfigure}{0.3\textwidth}
\includegraphics[width=\textwidth]{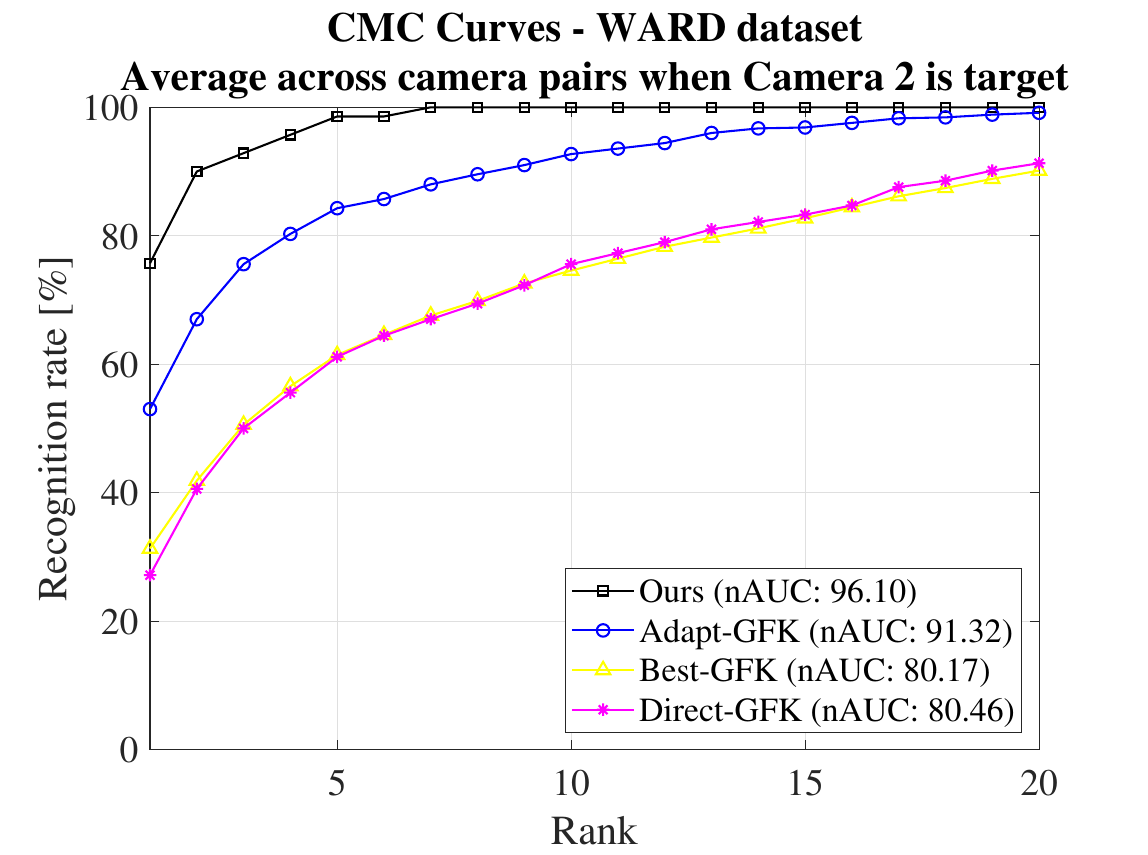}
\caption{}
\end{subfigure} 
\begin{subfigure}{0.3\textwidth}
\includegraphics[width=\textwidth]{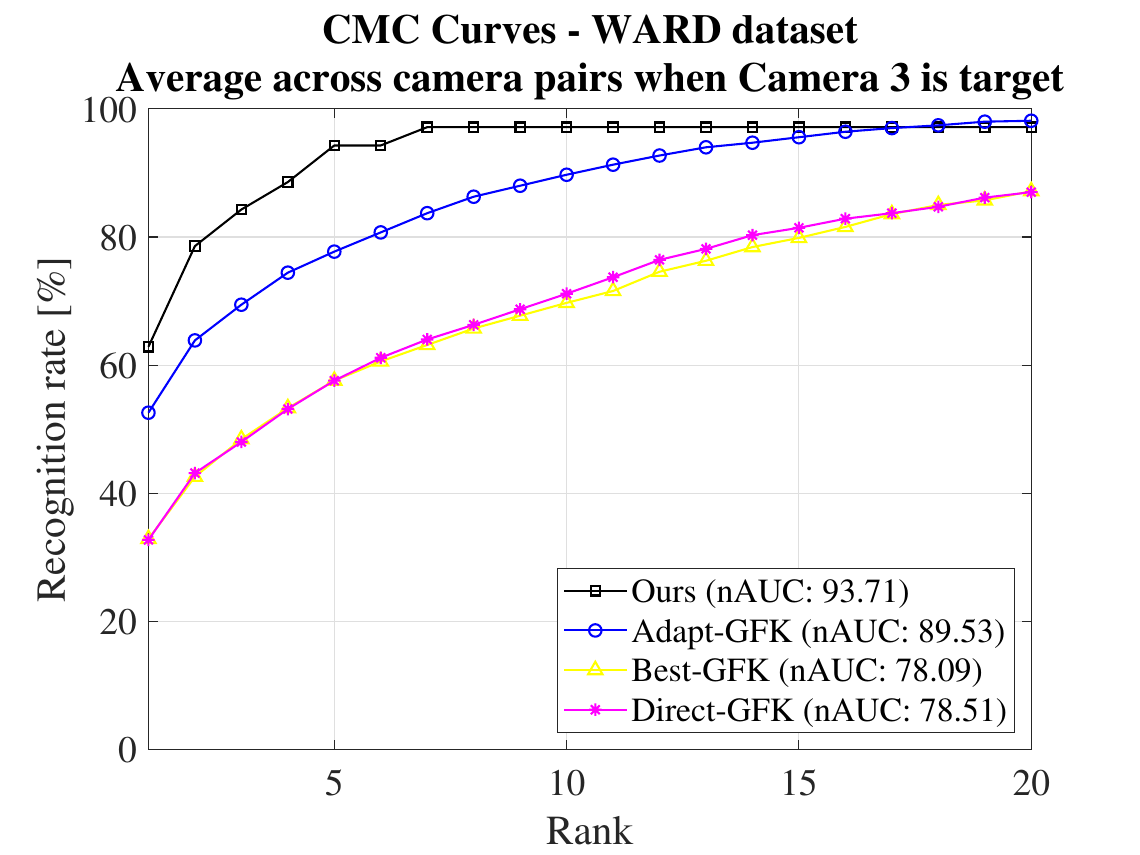}
\caption{}
\end{subfigure}
\caption{The setting in this case is exactly same as the setting of Figure~\ref{fig:singlecamwardold}. However this experiment is done only to compare our method with GFK methods in the original settings \cite{panda2017unsupervised} where the assumption was of the availability of source data. In this case GFK is calculated using the old source data as well as new limited target data. Our method significantly outperforms all the GFK based methods in this case also. It proves that even if our method does not use source data, it still outperforms the doamin adaptation methods which uses source data. }
\label{fig:singlecamwardoldnew}
\end{figure}

\begin{figure}[H]
\centering
\large \underline{Camera wise CMC curves for RAiD dataset}\par\medskip
\begin{subfigure}{0.4\textwidth}
\includegraphics[width=0.9\textwidth]{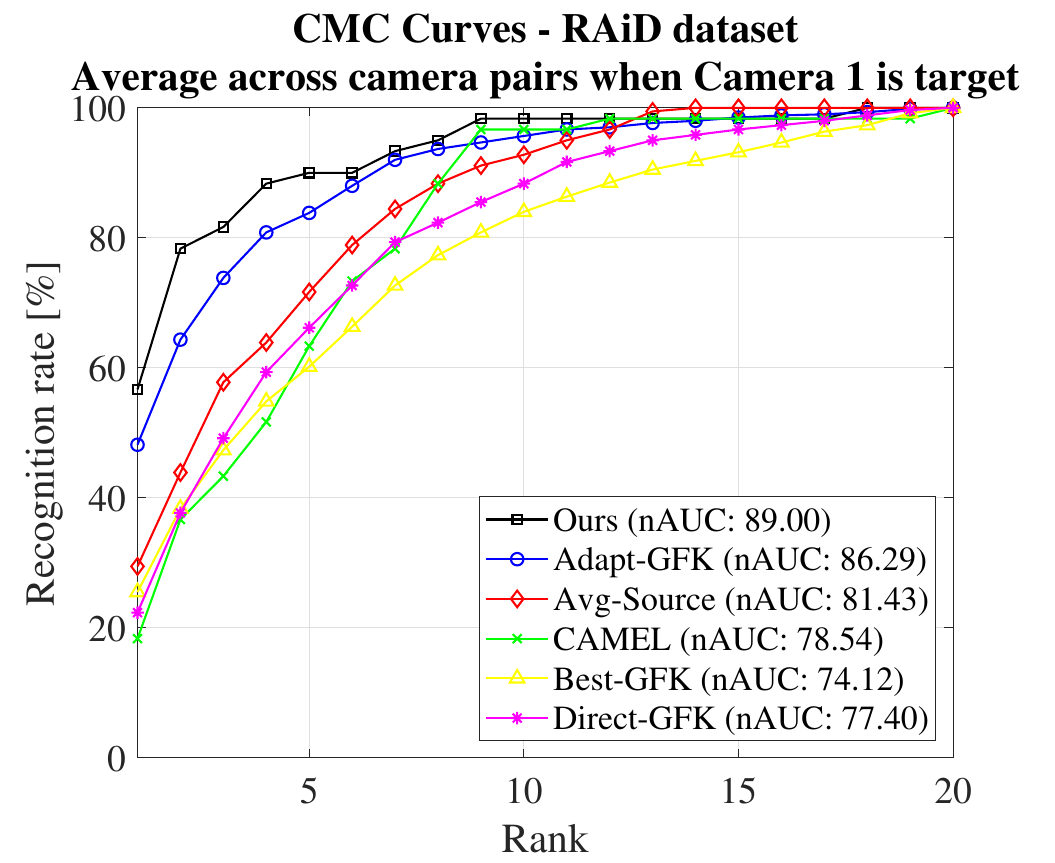}
\caption{}
\end{subfigure}
\begin{subfigure}{0.4\textwidth}
\includegraphics[width=0.9\textwidth]{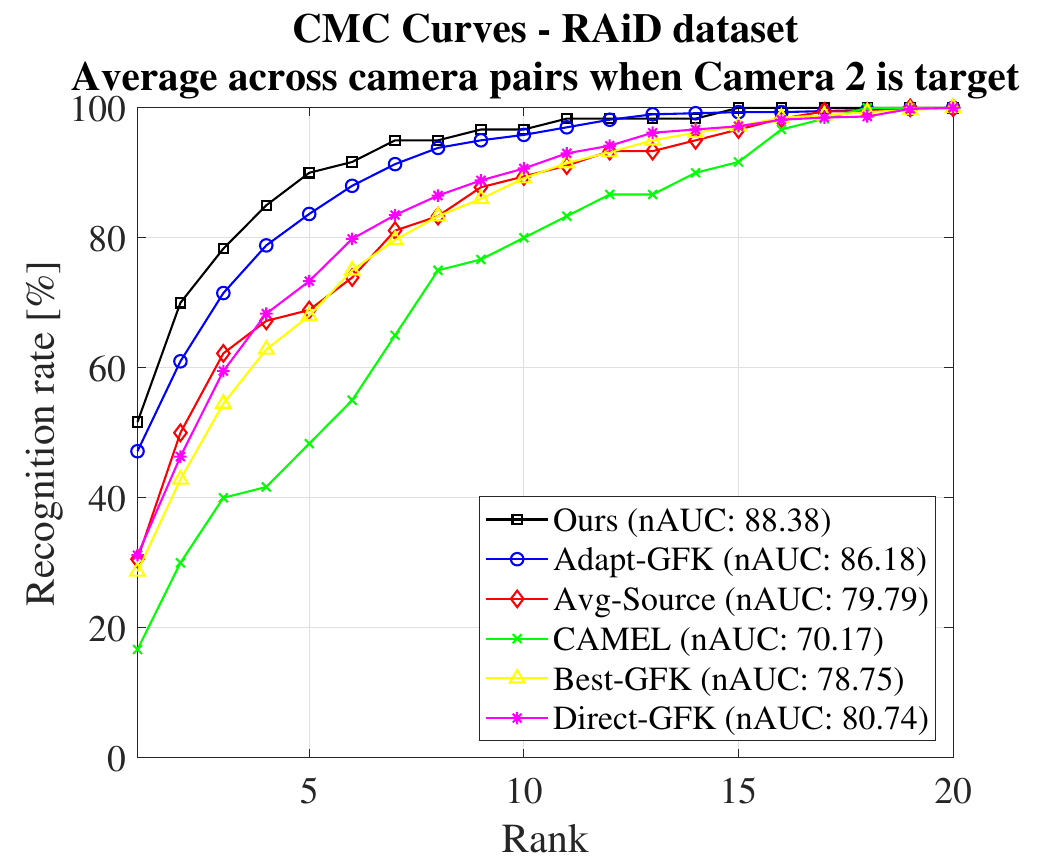}
\caption{}
\end{subfigure} \\
\begin{subfigure}{0.4\textwidth}
\includegraphics[width=0.9\textwidth]{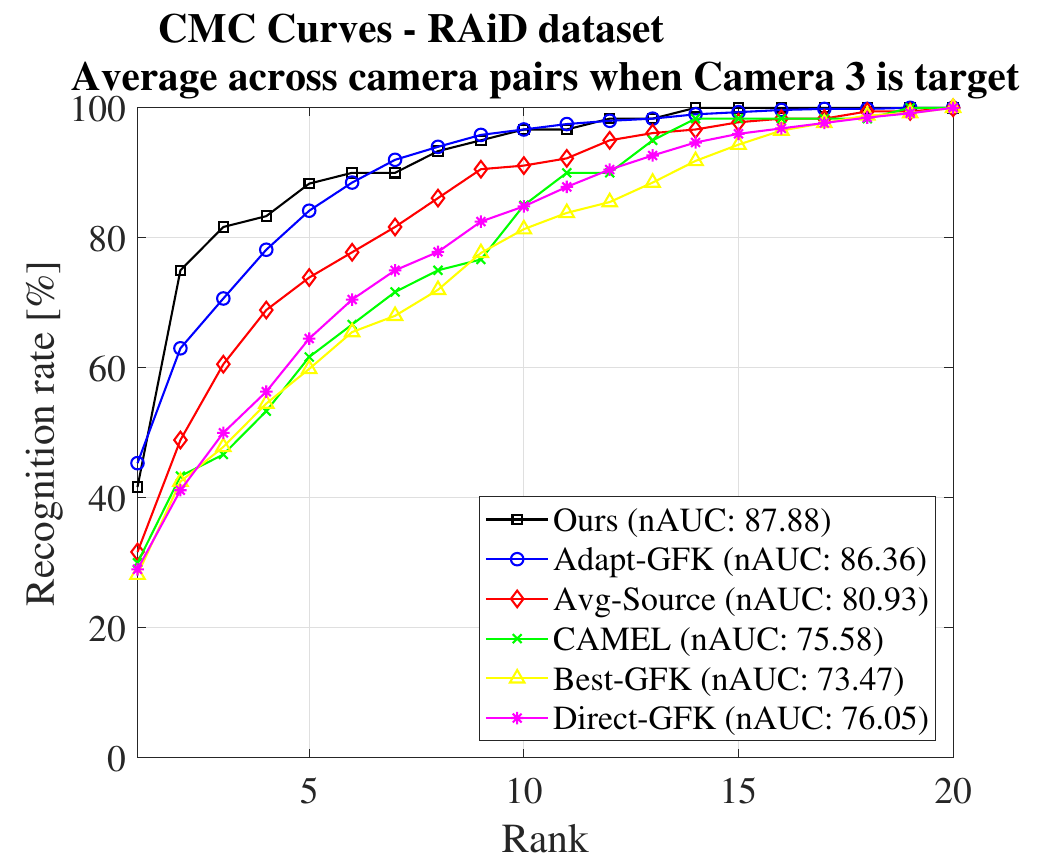}
\caption{}
\end{subfigure}
\begin{subfigure}{0.4\textwidth}
\includegraphics[width=0.9\textwidth]{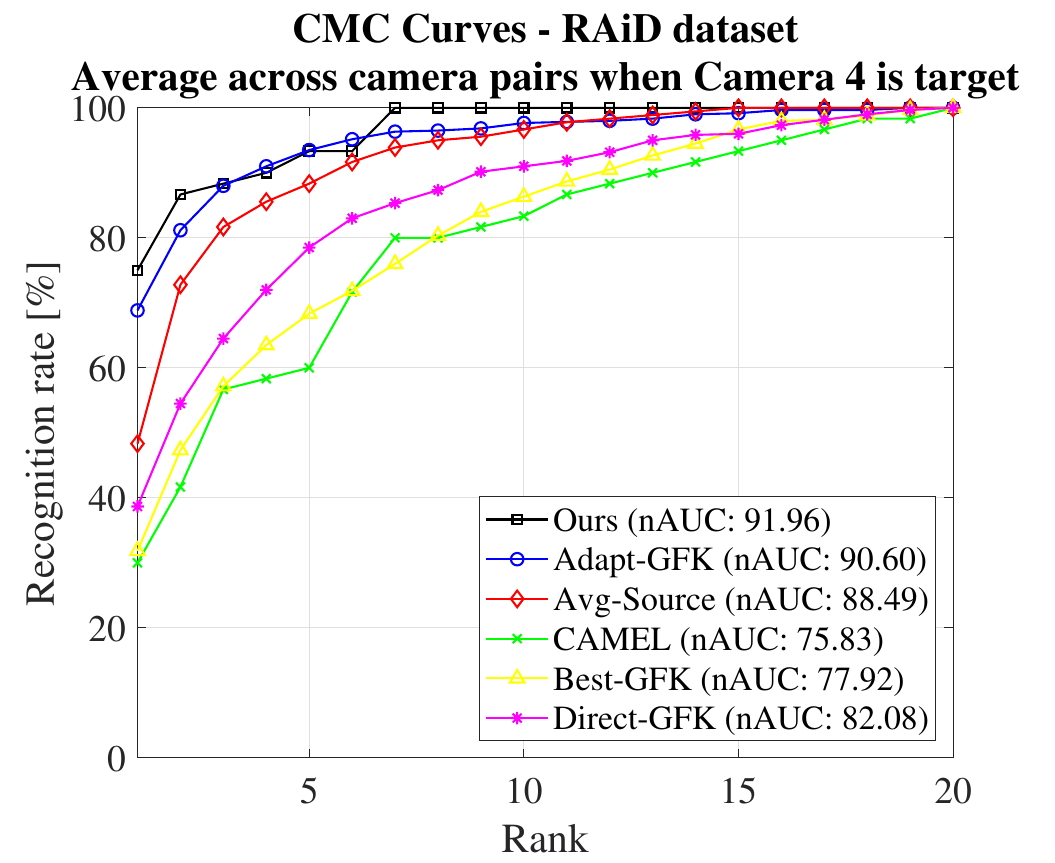}
\caption{}
\end{subfigure}
\caption{In this experiment RAiD dataset with 4 cameras \cite{das2014consistent} is used. Each of the camera has been set as target while rest of the 3 cameras with 3 pairwise metrics served as source metrics. plot (a,b,c,d) are generated from camera 1,2,3 and 4 as target target camera. The most competitive method here is Adapt-GFK which is outperformed by our method in nAUC with margins 2.71\%, 2.2\%, 1.52\% and 1.36\% for camera 1,2,3 and 4 as target respectively. Moreover for camera 1 as target (plot (a)) and  camera 4 as target (plot (d)) our method outperforms Adapt-GFK by a rank-1 margin of almost 7\% and 5\% respectively. Also for each of the cameras our method outperforms Avg-source significantly both in rank-1 and nAUC which proves the  Theorem~\ref{thm1}. Moreover, for all the cases our method outperforms CAMEL significantly (Like in camera 4 rank-1 margin is almost ~36\%) which is equivalent to fully supervised learning with limited labels with no transfer from any sources.}
\label{fig:singlecamraid}
\end{figure}

\begin{figure}[H]
\centering
\large \underline{Camera wise CMC curves for Market1501 dataset}\par\medskip
\begin{subfigure}{0.4\textwidth}
\includegraphics[width=\textwidth]{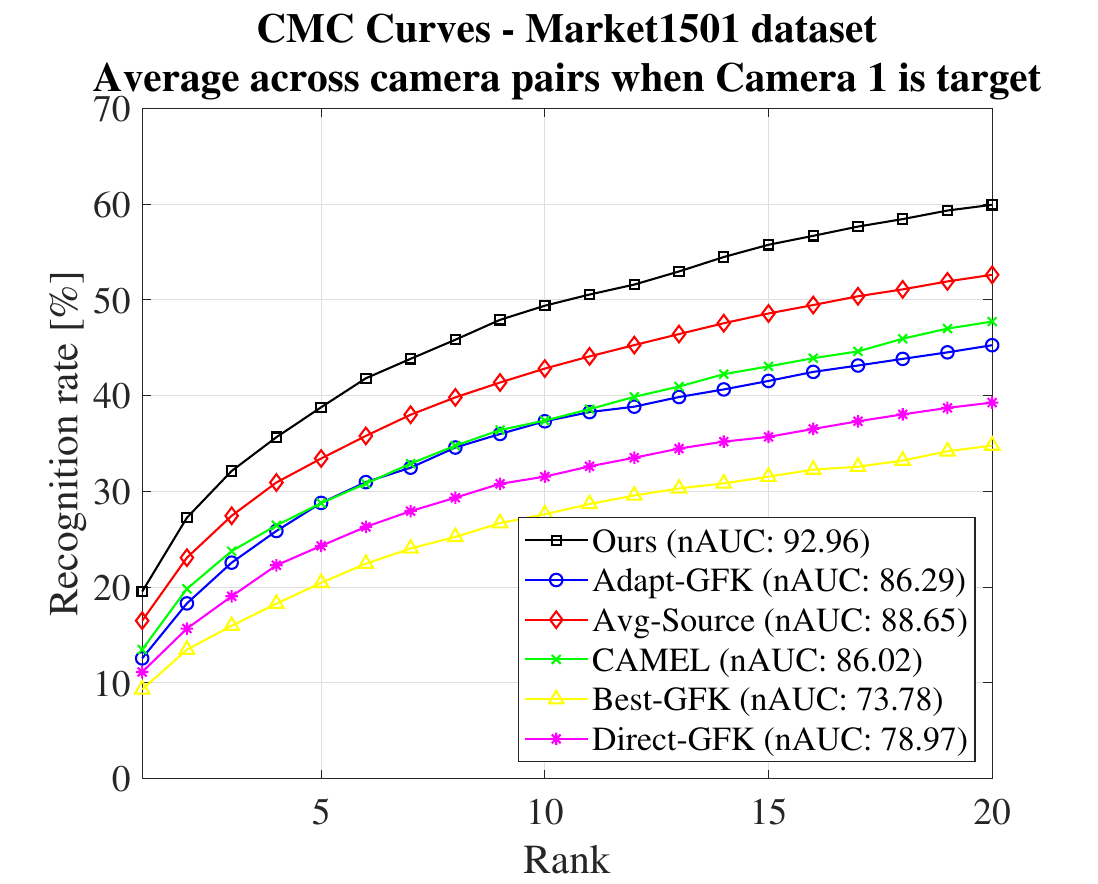}
\caption{}
\end{subfigure}
\begin{subfigure}{0.4\textwidth}
\includegraphics[width=\textwidth]{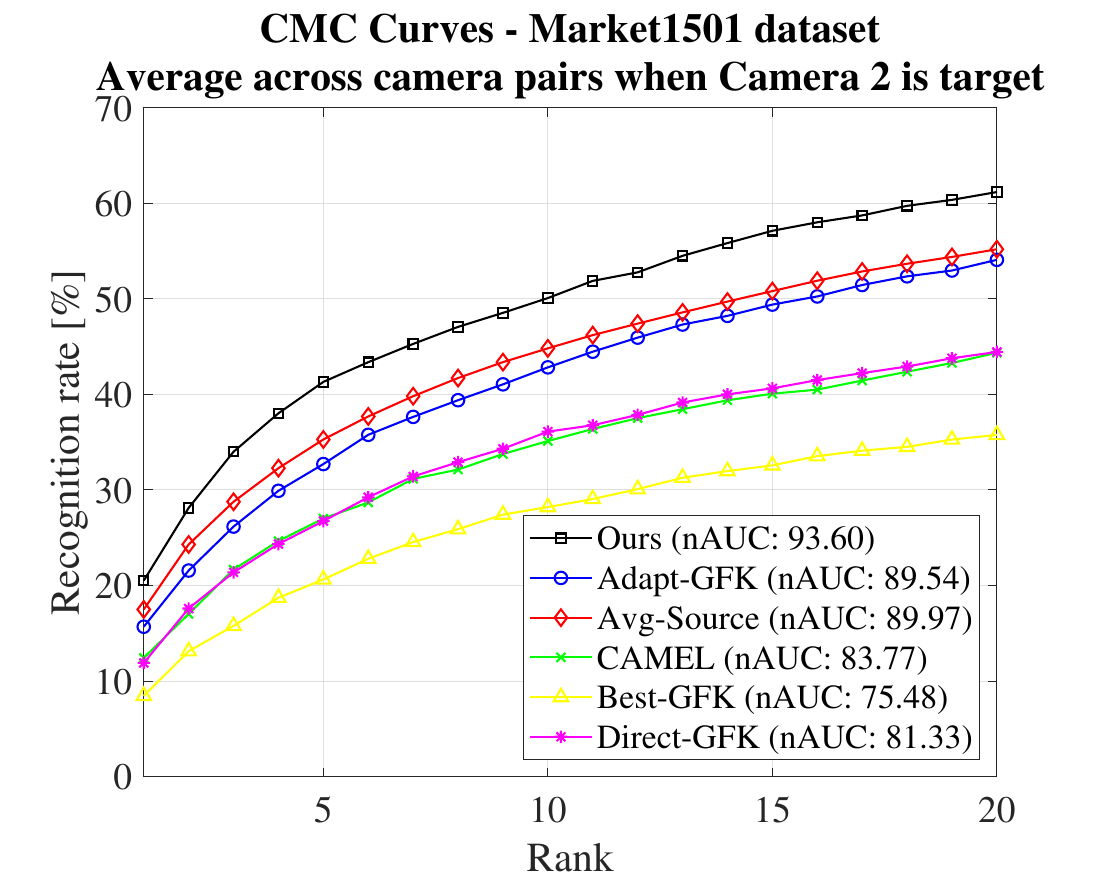}
\caption{}
\end{subfigure} \\
\begin{subfigure}{0.4\textwidth}
\includegraphics[width=\textwidth]{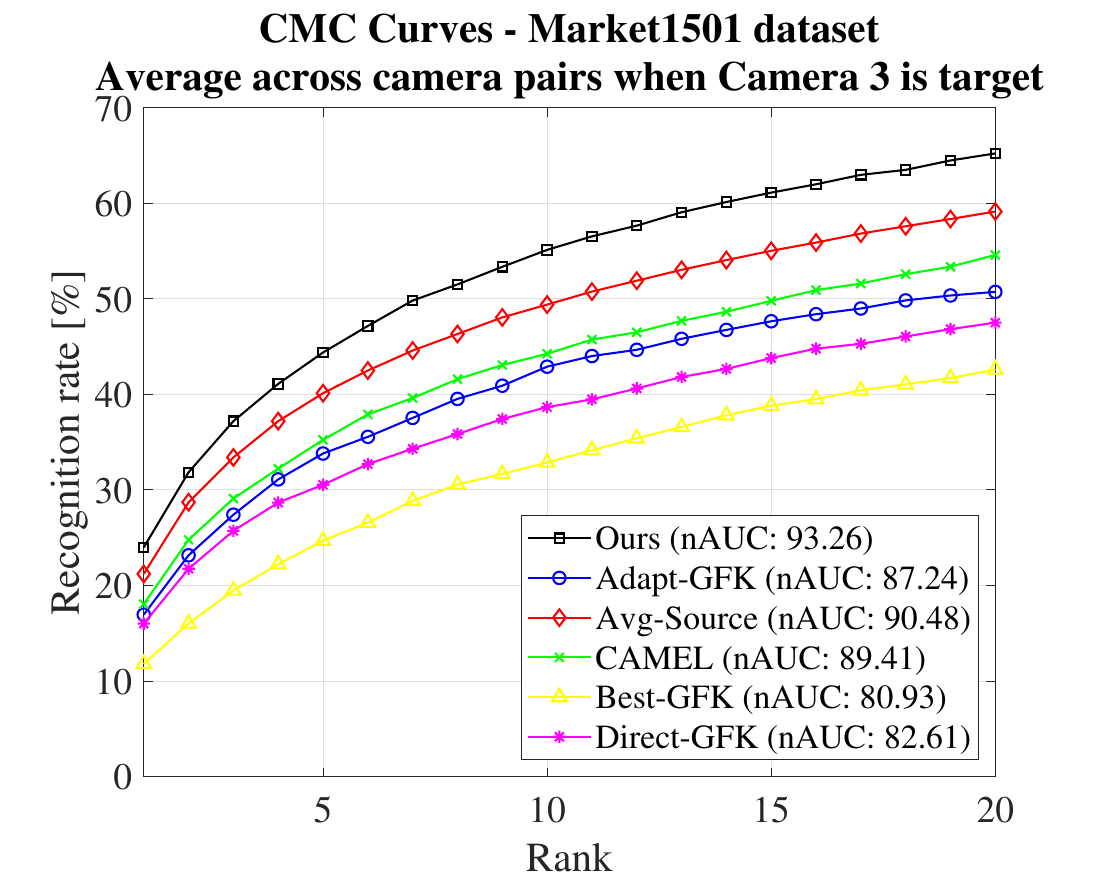}
\caption{}
\end{subfigure} 
\begin{subfigure}{0.4\textwidth}
\includegraphics[width=\textwidth]{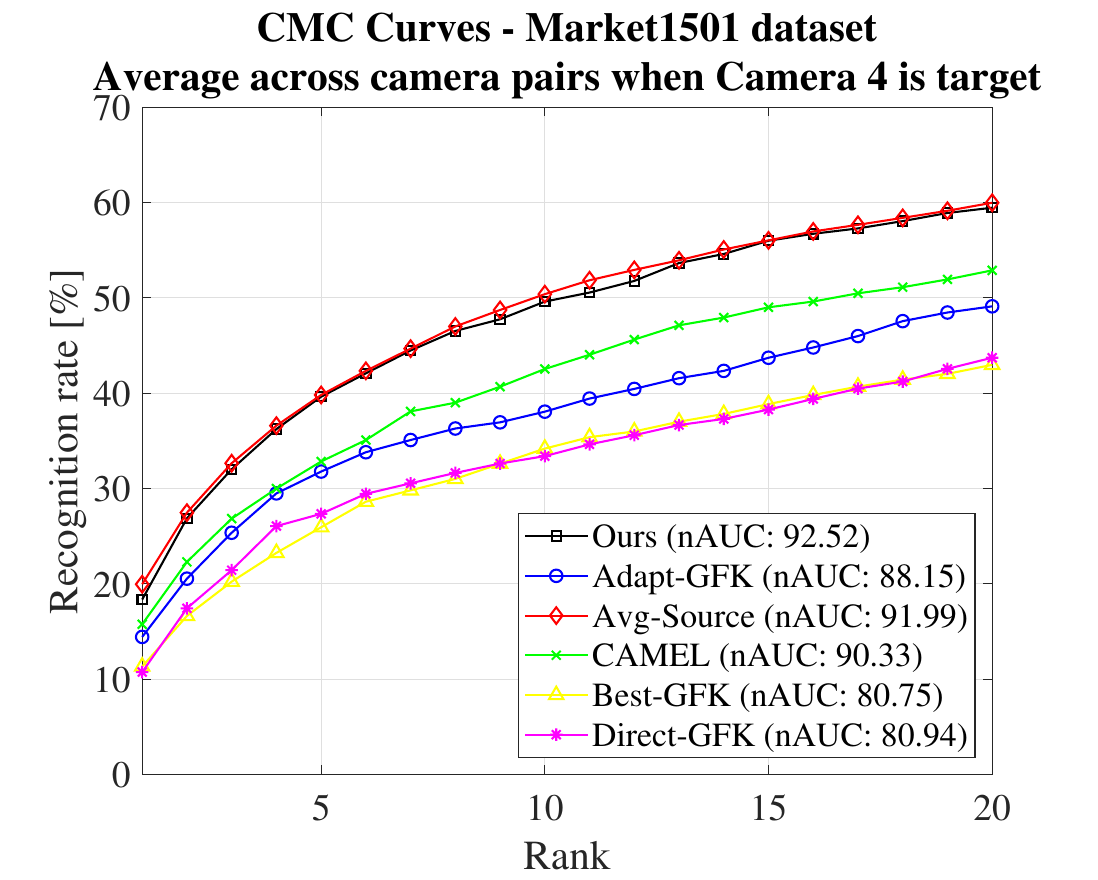}
\caption{}
\end{subfigure} \\
\begin{subfigure}{0.4\textwidth}
\includegraphics[width=\textwidth]{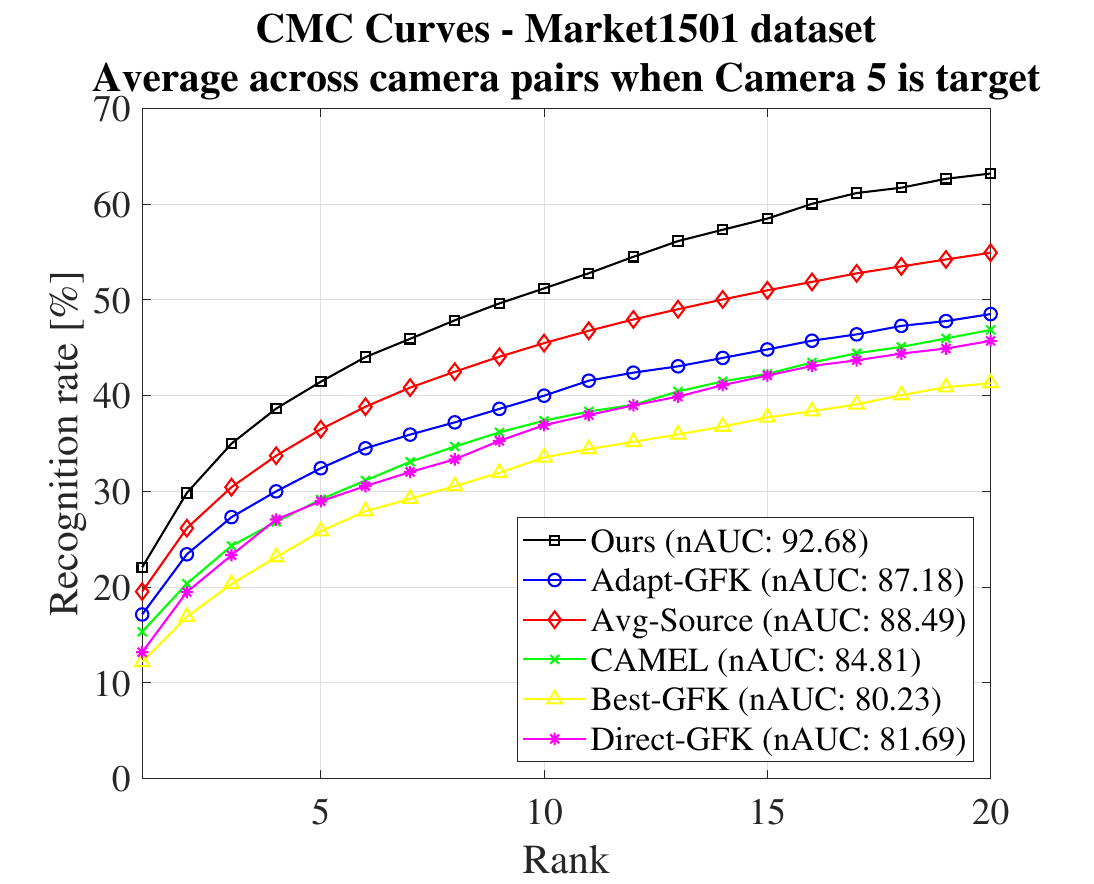}
\caption{}
\end{subfigure}
\begin{subfigure}{0.4\textwidth}
\includegraphics[width=\textwidth]{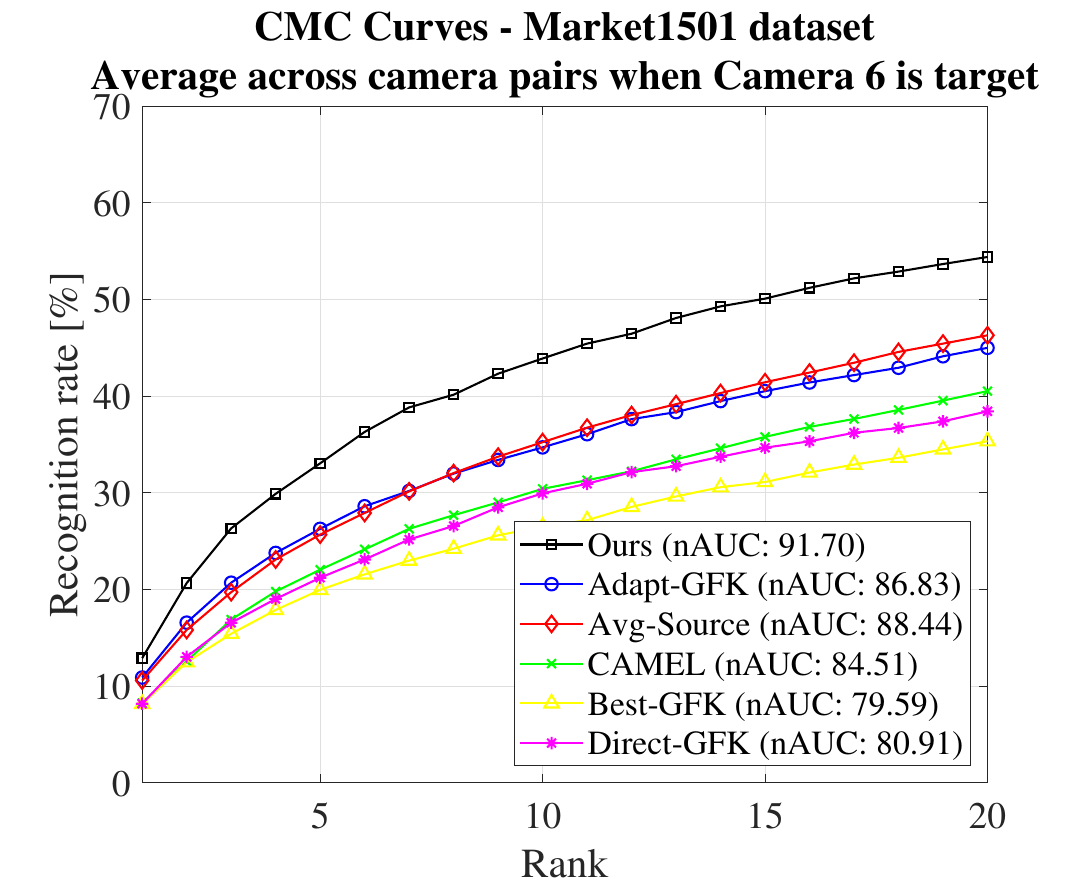}
\caption{}
\end{subfigure}
\caption{In this single camera insertion experiment Market1501 \cite{zheng2015scalable} dataset is used. In plots (a,b,c,d,e and f) cmc curves are shown for camera 1,2,3,4,5 and 6 as target respectively. Only 10\%  of the available data is used between each target-source pairs. Our method outperforms Adapt-GFK which was the most competitive one in case of RAiD and WARD by 6.67\%,4.06\%,6.02\%,4.37\%,5.5\%,4.87\% in nAUC. However, in this case we see that Adapt-GFK has lower accuracy than just the Avg-source, which we outperform in both rank-1 and nAUC for each and every camera as target. Also our method has very high accuracy both in rank-1 and nAUC than CAMEL which is equivalent to no transfer scenario. It is clear that our method gives theoretical guarantee that it would not perform worse than Avg-source case or no transfer case whereas other method has no guarantee which is depicted in this case where Adapt-GFk performed worse than just the Avg-source.}
\label{fig:singlecammarket}
\end{figure}

\begin{figure}[H]
\vspace{-2 cm}
\centering
\large \underline{Camera wise CMC curves for MSMT dataset camera (1-15)}\par\medskip
\begin{subfigure}{0.3\textwidth}
\includegraphics[width=\textwidth]{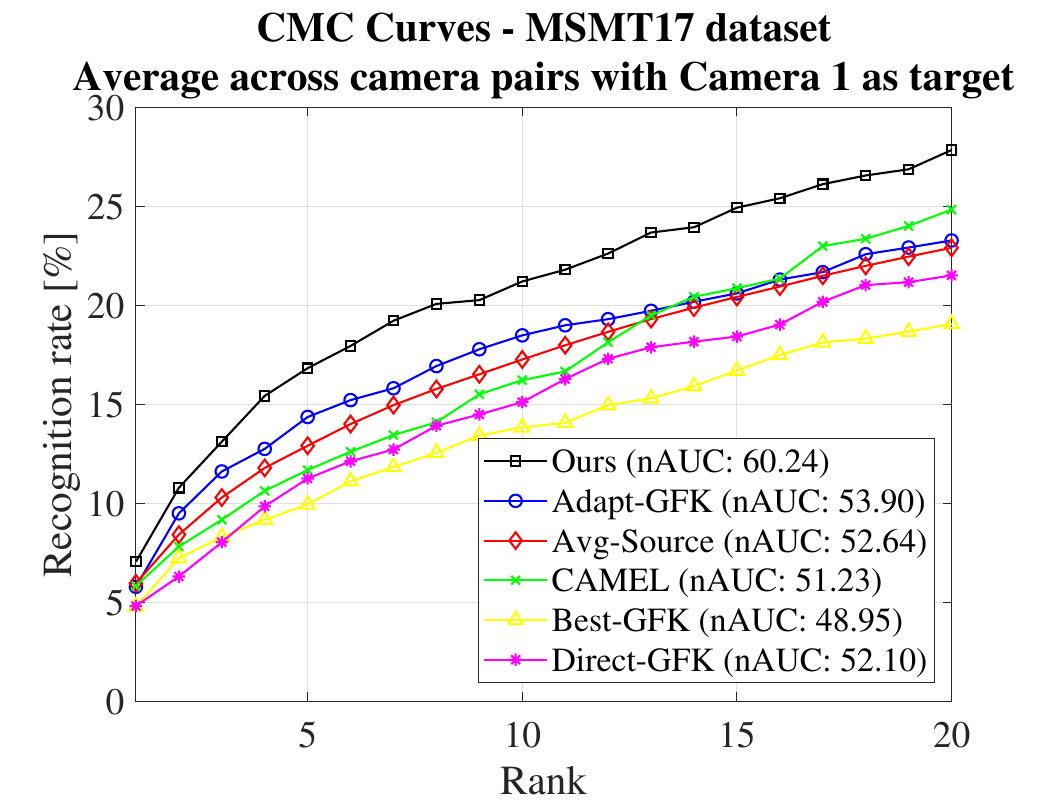}
\caption{}
\end{subfigure}
\begin{subfigure}{0.3\textwidth}
\includegraphics[width=\textwidth]{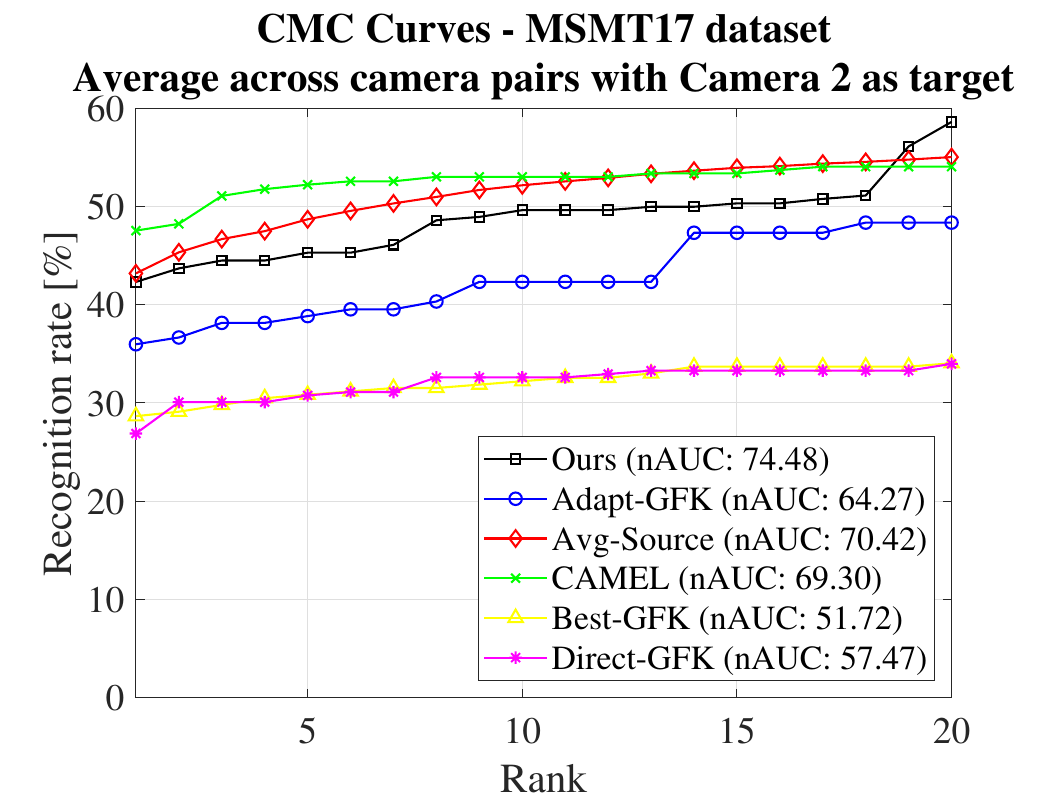}
\caption{}
\end{subfigure} 
\begin{subfigure}{0.3\textwidth}
\includegraphics[width=\textwidth]{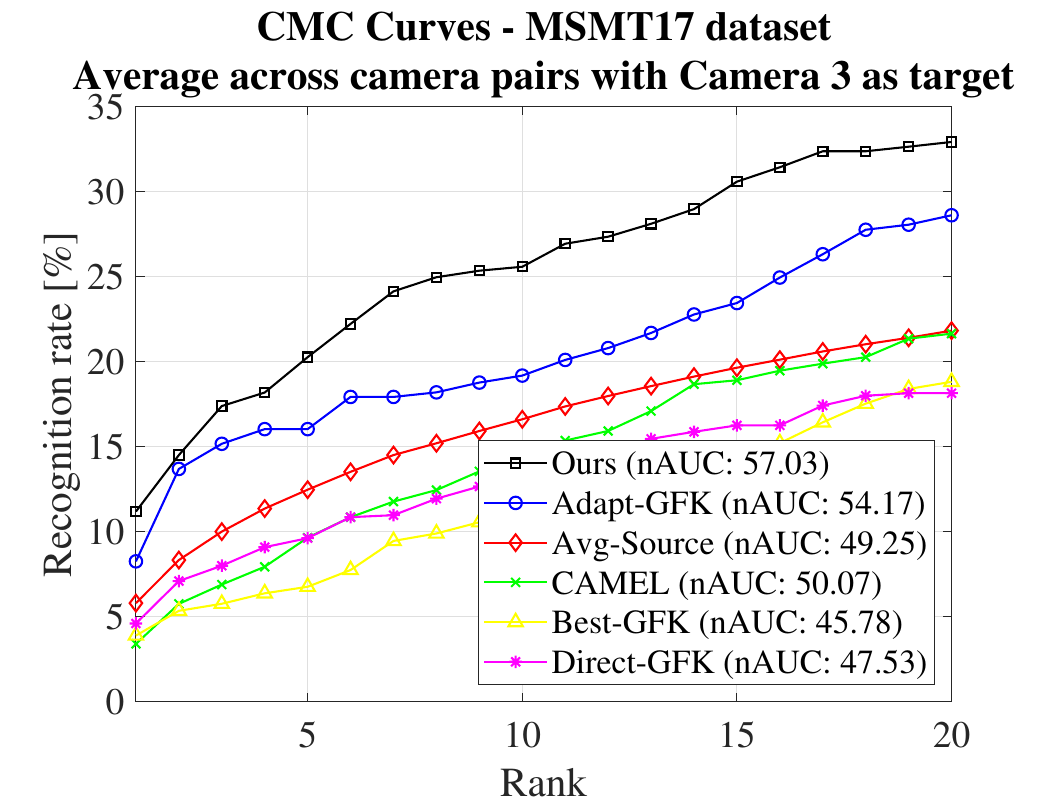}
\caption{}
\end{subfigure}\\
\begin{subfigure}{0.3\textwidth}
\includegraphics[width=\textwidth]{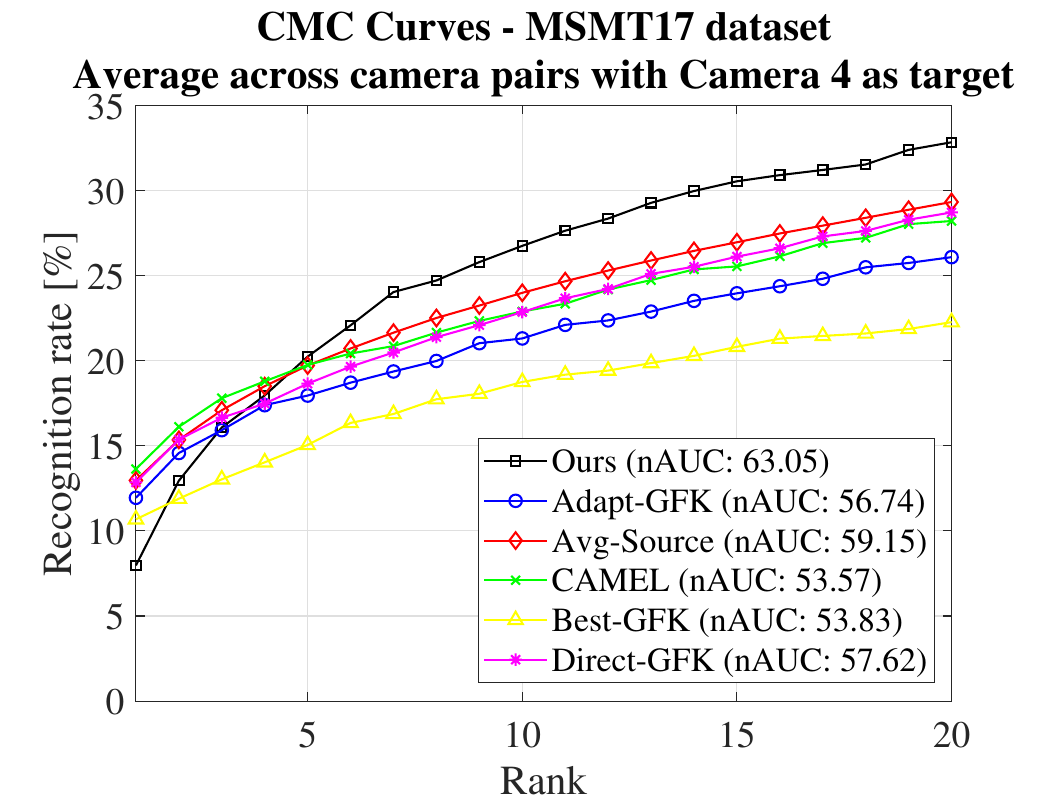}
\caption{}
\end{subfigure} 
\begin{subfigure}{0.3\textwidth}
\includegraphics[width=\textwidth]{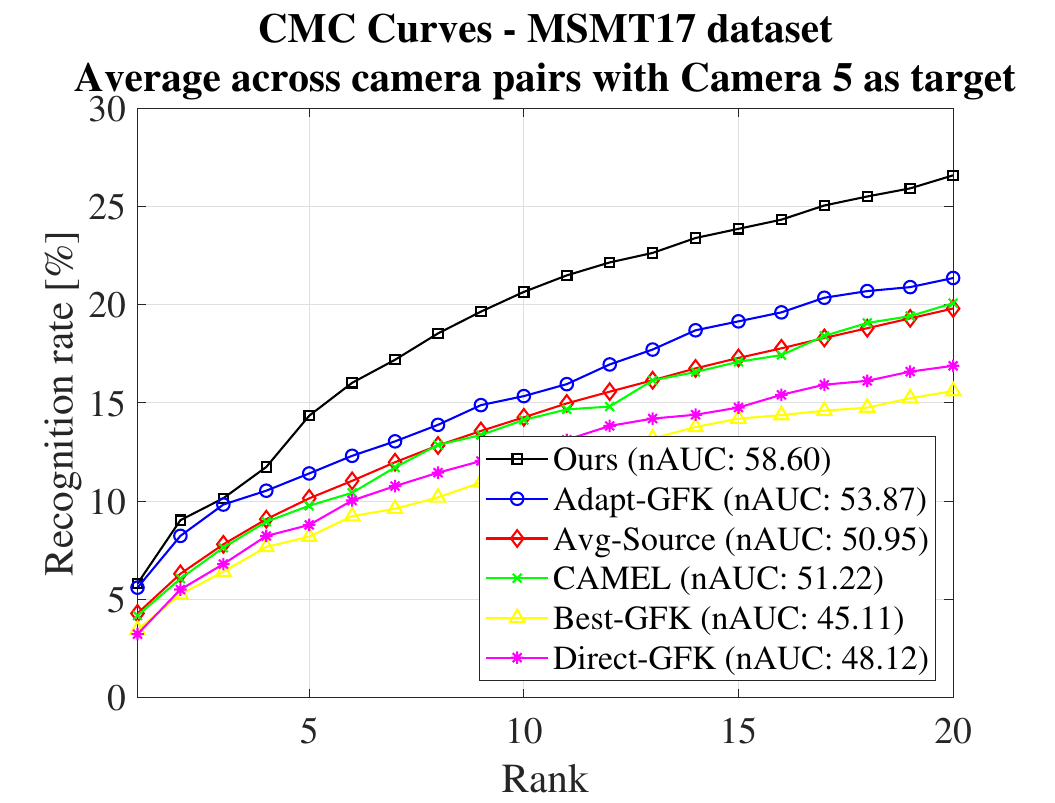}
\caption{}
\end{subfigure}
\begin{subfigure}{0.3\textwidth}
\includegraphics[width=\textwidth]{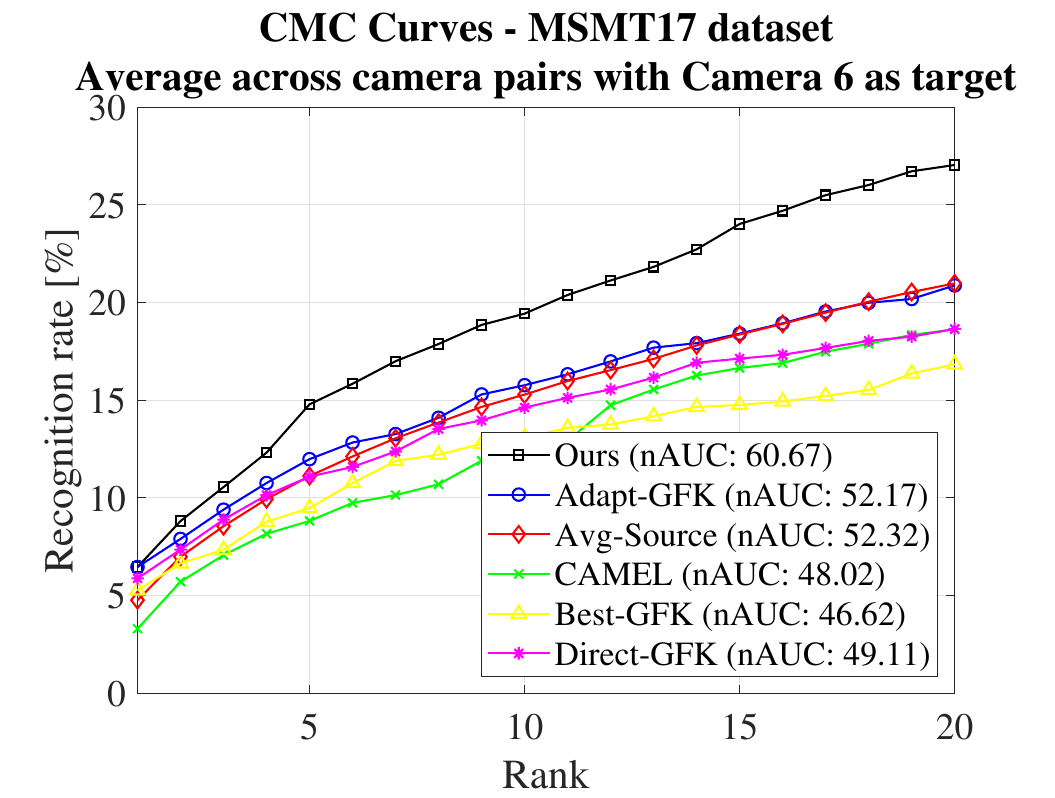}
\caption{}
\end{subfigure} \\
\begin{subfigure}{0.3\textwidth}
\includegraphics[width=\textwidth]{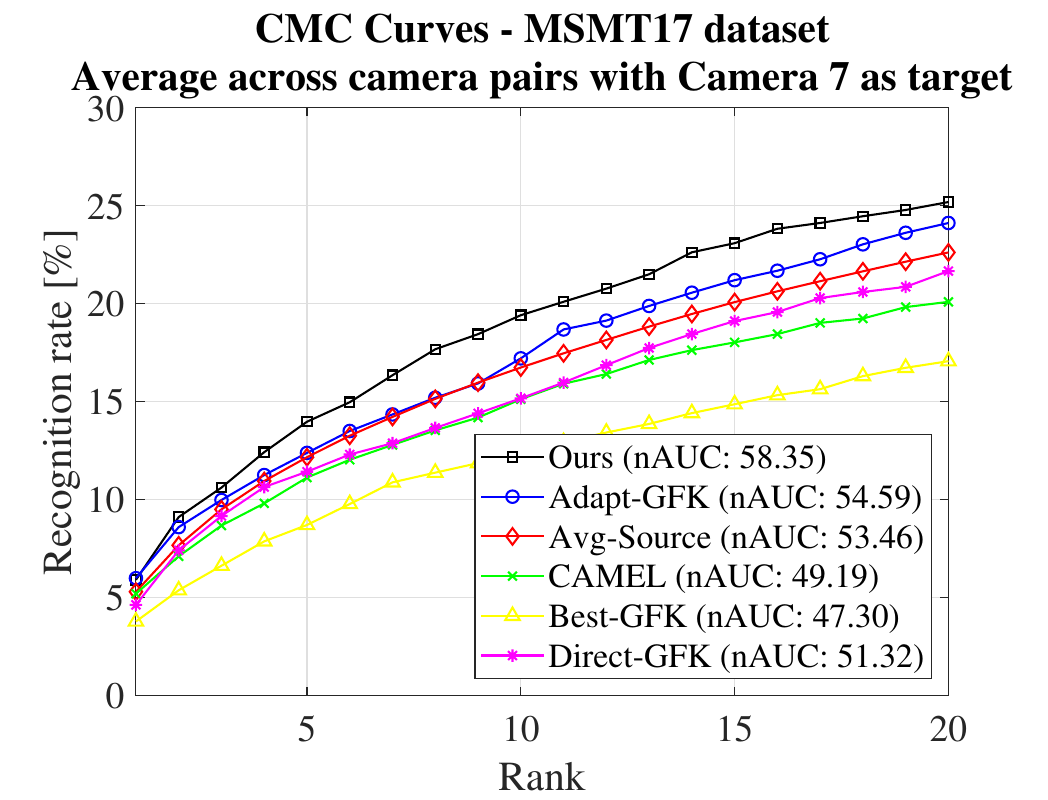}
\caption{}
\end{subfigure}
\begin{subfigure}{0.3\textwidth}
\includegraphics[width=\textwidth]{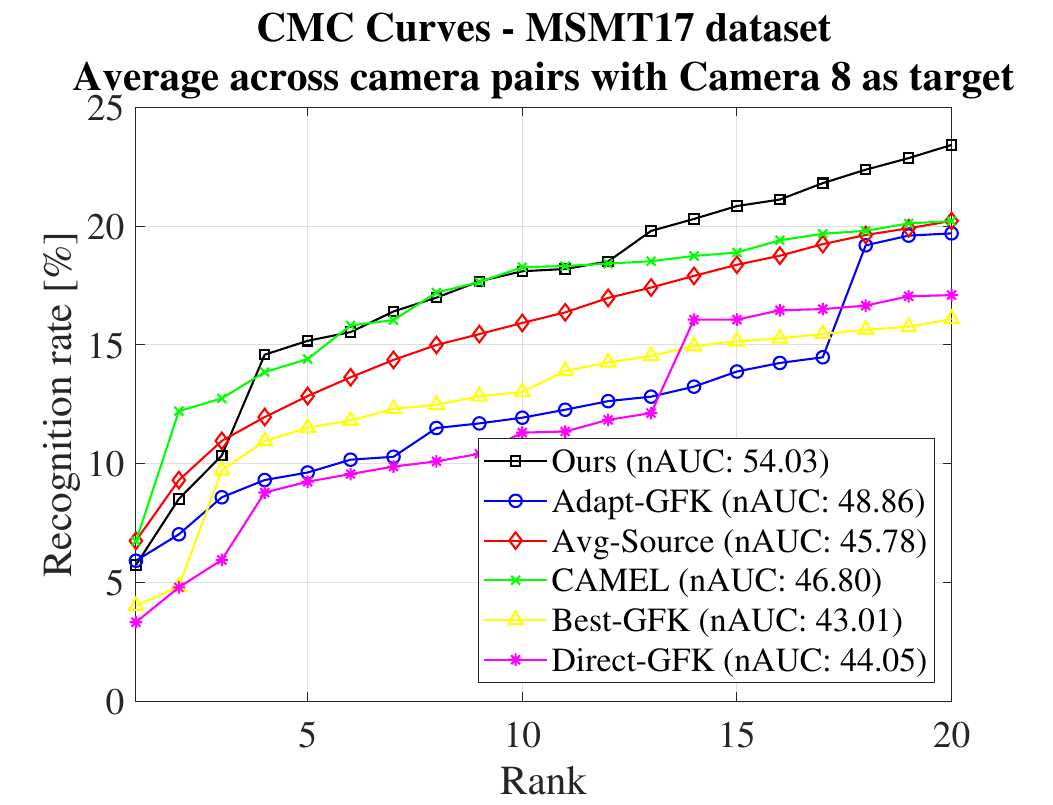}
\caption{}
\end{subfigure} 
\begin{subfigure}{0.3\textwidth}
\includegraphics[width=\textwidth]{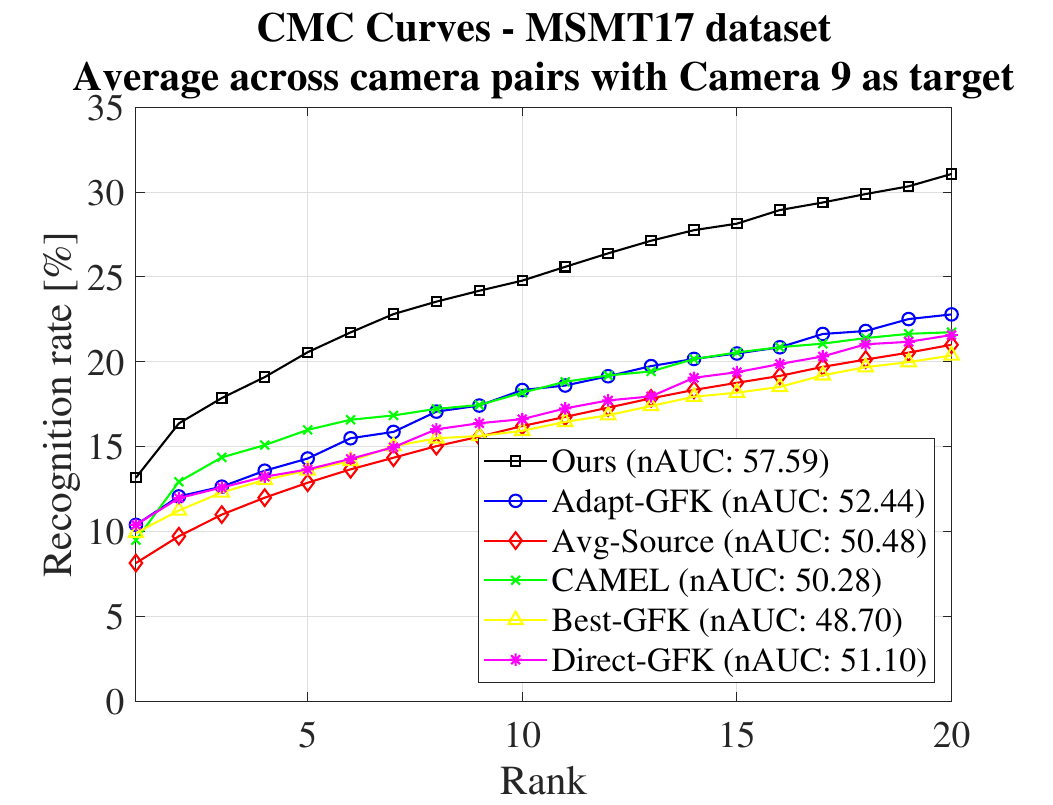}
\caption{}
\end{subfigure}\\
\begin{subfigure}{0.3\textwidth}
\includegraphics[width=\textwidth]{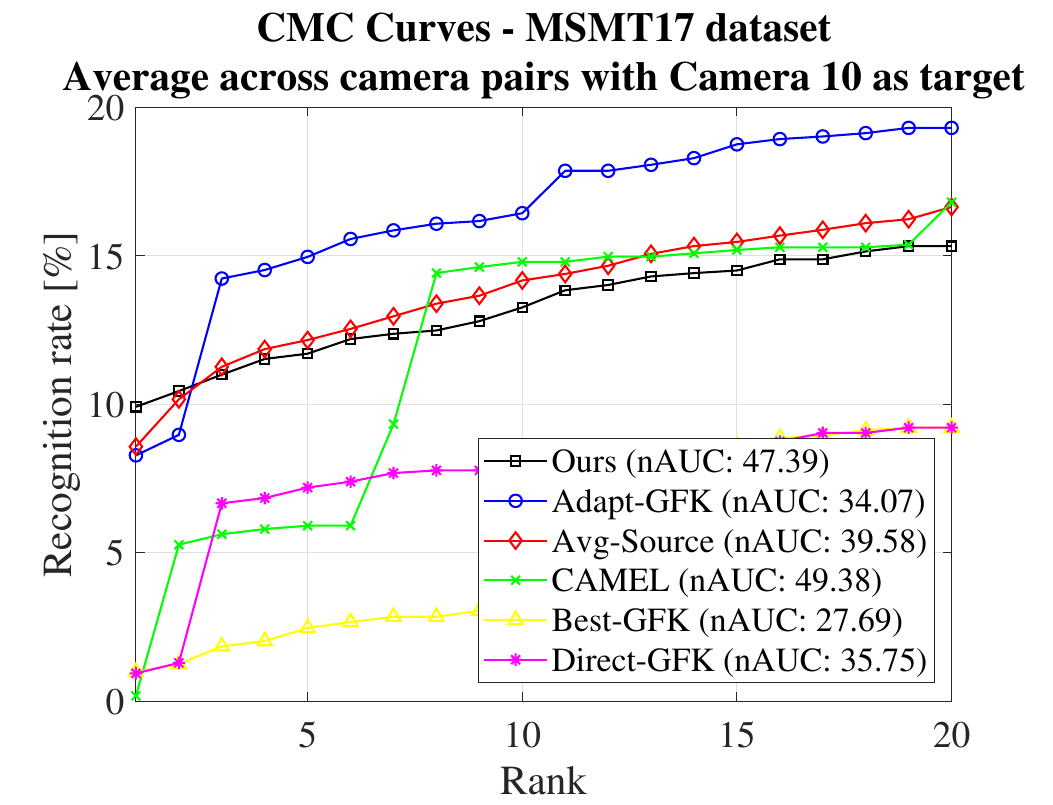}
\caption{}
\end{subfigure} 
\begin{subfigure}{0.3\textwidth}
\includegraphics[width=\textwidth]{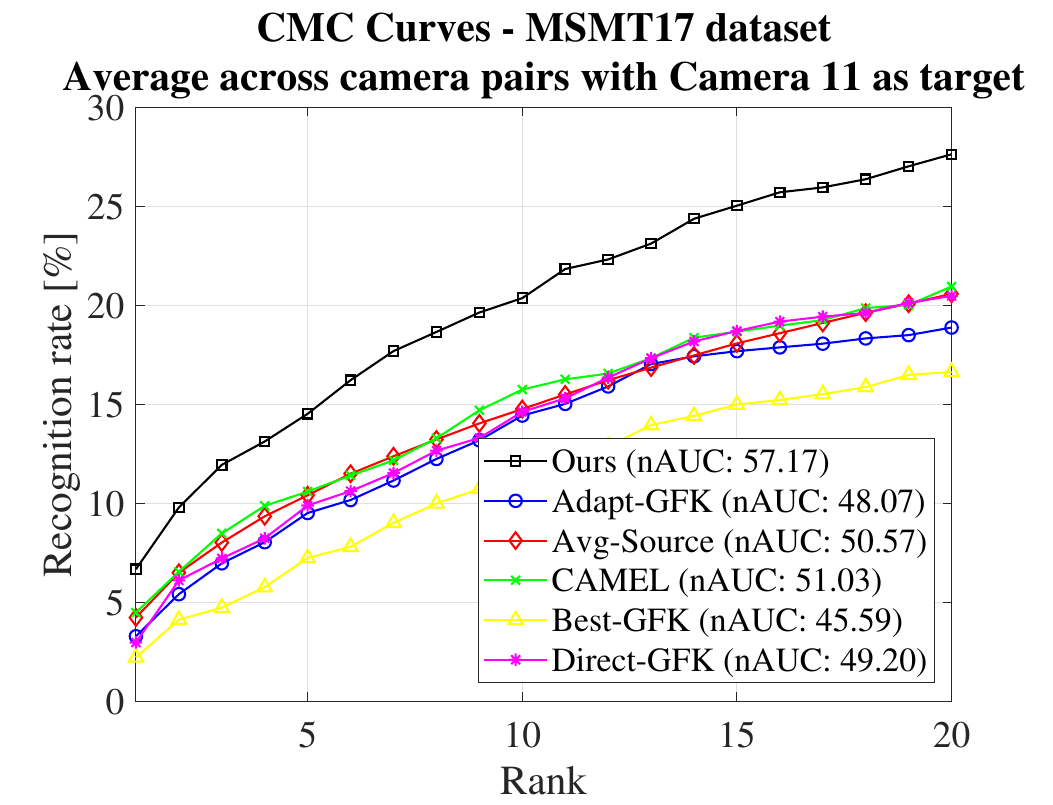}
\caption{}
\end{subfigure}
\begin{subfigure}{0.3\textwidth}
\includegraphics[width=\textwidth]{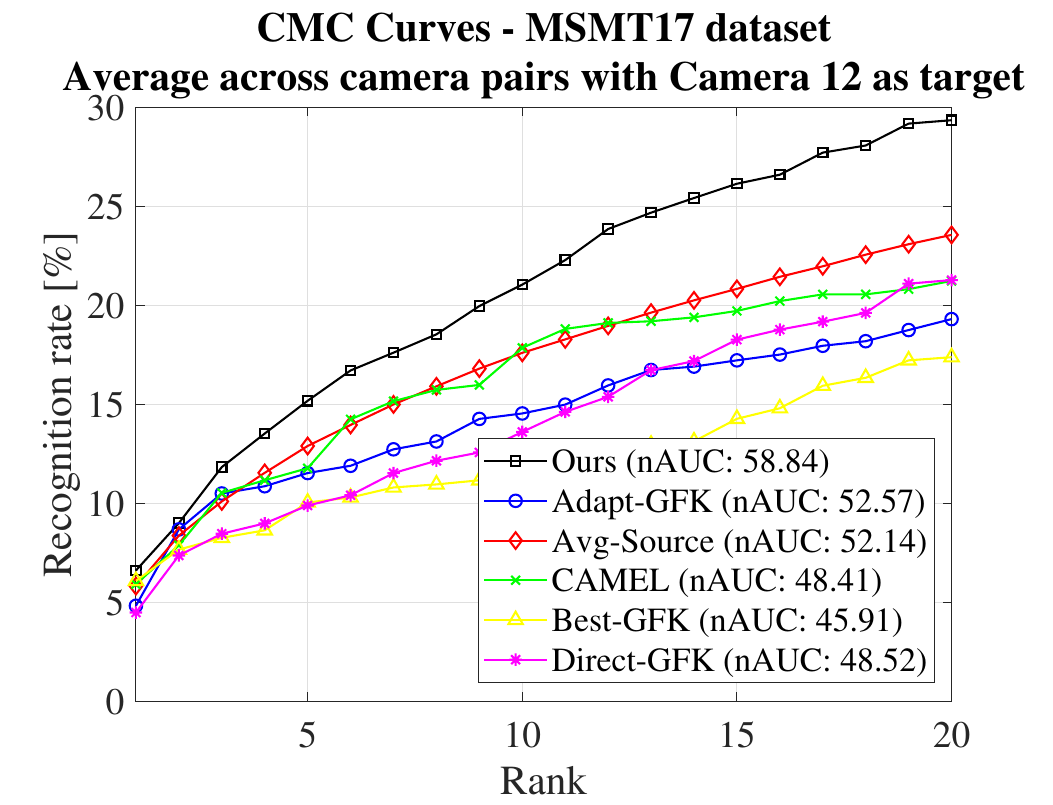}
\caption{}
\end{subfigure} \\
\begin{subfigure}{0.3\textwidth}
\includegraphics[width=\textwidth]{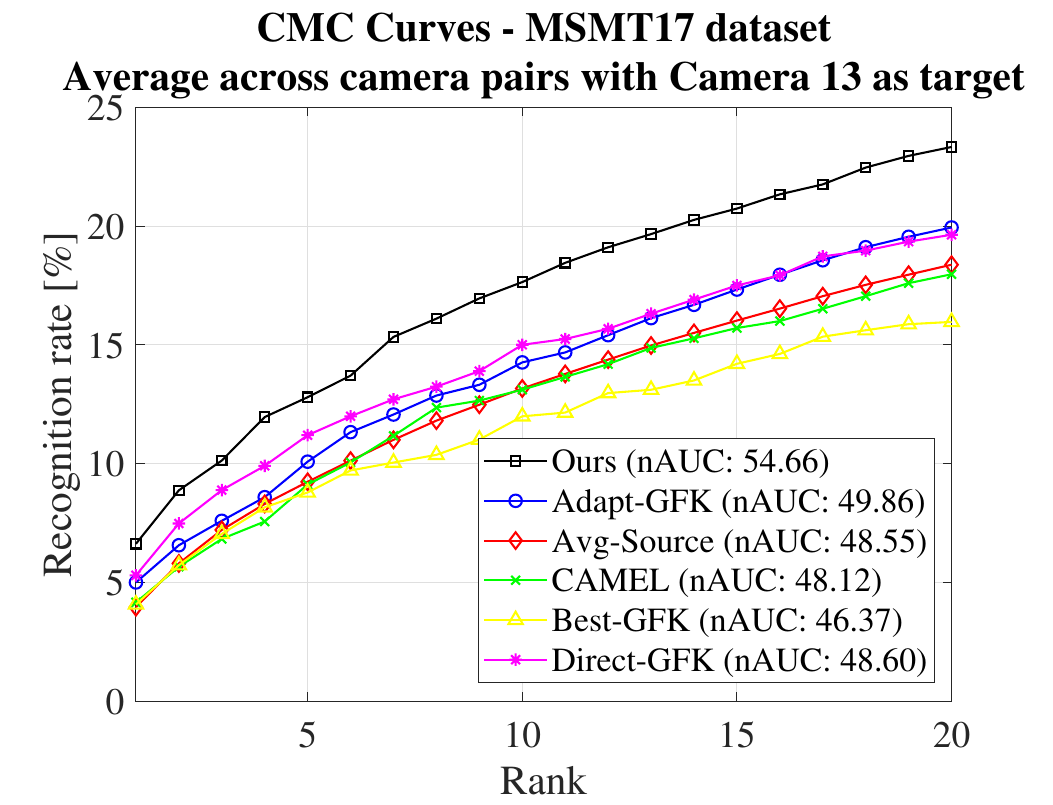}
\caption{}
\end{subfigure}
\begin{subfigure}{0.3\textwidth}
\includegraphics[width=\textwidth]{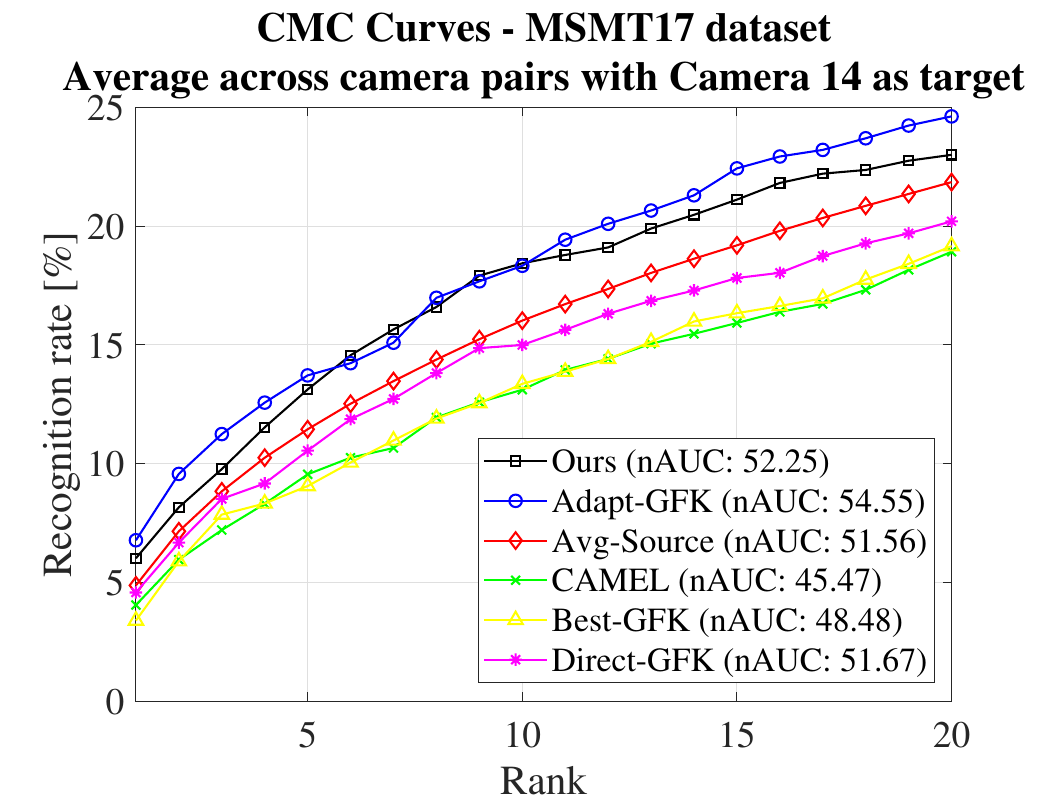}
\caption{}
\end{subfigure} 
\begin{subfigure}{0.3\textwidth}
\includegraphics[width=\textwidth]{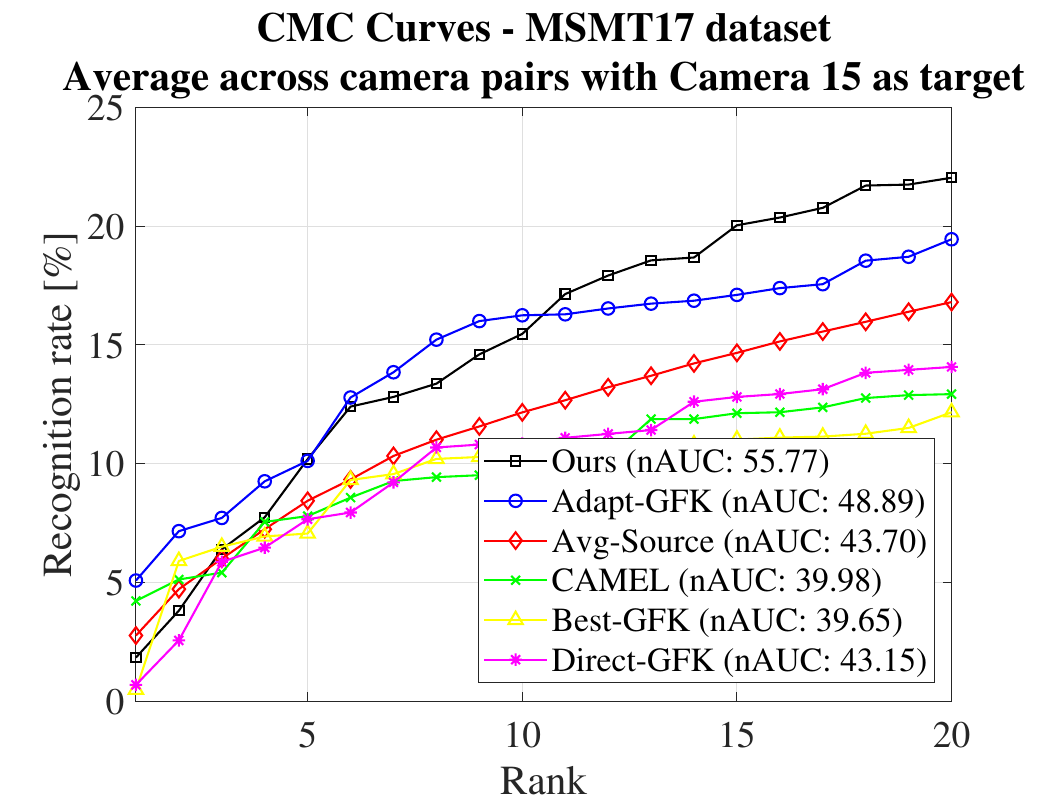}
\caption{}
\end{subfigure} \vspace{-2mm}
\caption{Total 15 plots from 15 cameras as target in MSMT dataset are shown. For all cameras our method outperforms other methods in nAUC. While rank-1 performances varied a lot across different cameras, our method on average performs the best as shown in the main paper. Best viewed in color.}
\label{fig:singlecammsmt}
\end{figure}


\section{On-boarding Multiple New Cameras}
This section covers the camera wise experimental results of on-boarding multiple new cameras (See Figure~(\ref{fig:2parallelmulticam},\ref{fig:3parallelmulticam},\ref{fig:continuousmulticam}). We show for each experiment the camera wise CMC curves that are averaged to a single CMC curve in the main paper. 
\begin{figure}[H]
\centering
\large \underline{Camera wise CMC curves for Market1501 dataset: parallel addition of 2 cameras}\par\medskip
\begin{subfigure}{0.4\textwidth}
\includegraphics[width=\textwidth]{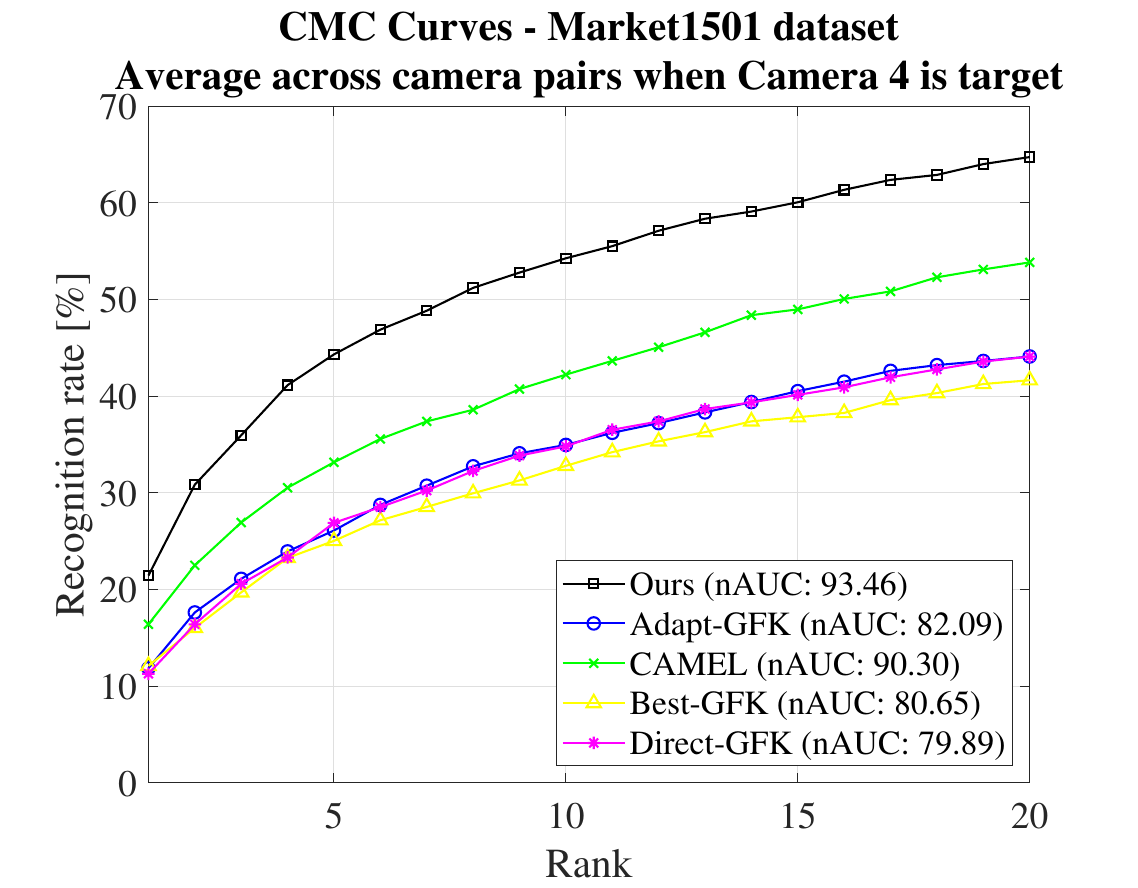}
\caption{}
\end{subfigure}
\begin{subfigure}{0.4\textwidth}
\includegraphics[width=\textwidth]{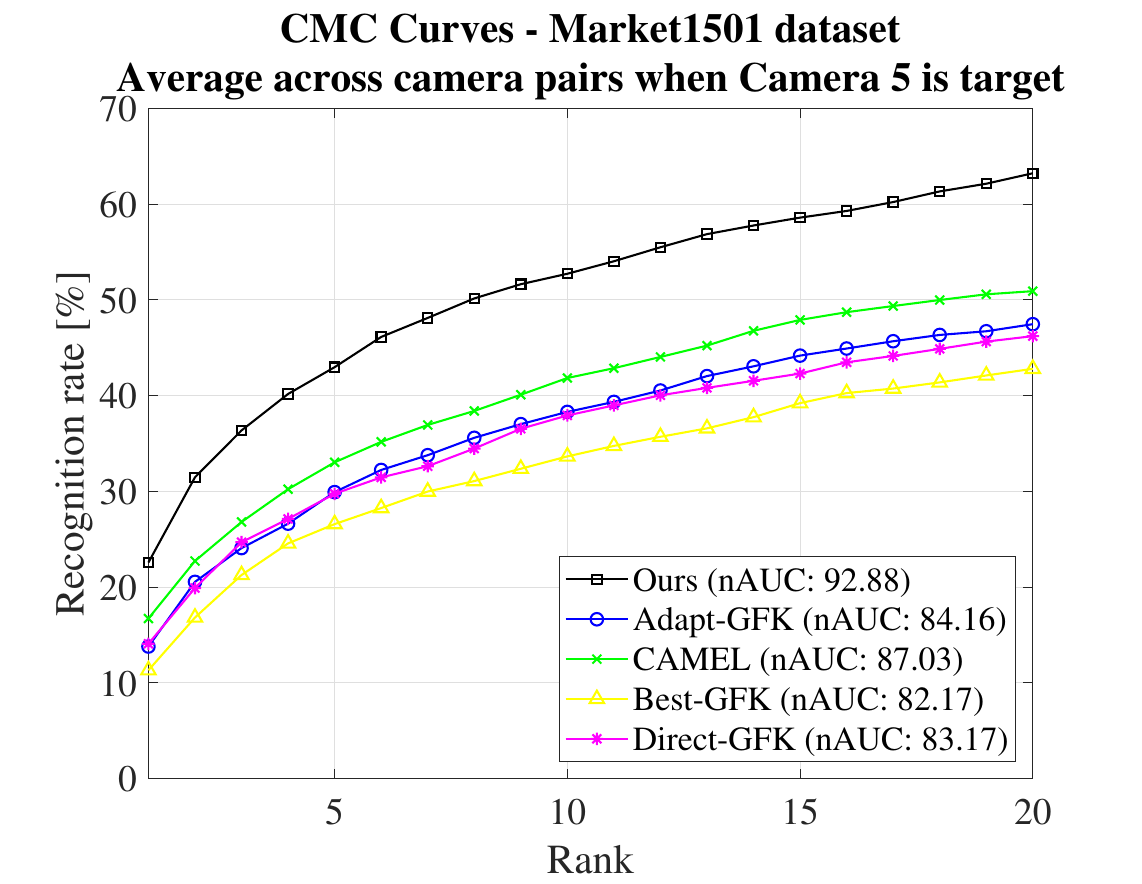}
\caption{}
\end{subfigure} 
\caption{In this figure we used Market1501 dataset to show the effect of parallel on-boarding of multiple cameras (In this case 2 cameras). We effectively set camera 4 and 5 as target and compute 6 source metrics from the remaining cameras to transfer knowledge from. Accuracy is shown between camera 4 and camera (1,2,3,6) (plot(a)) and also between camera 5 and camera (1,2,3,6) (plot(b)) separately. We can see that our method significantly outperform other methods both in rank-1 and nAUC. This shows the effectiveness of our method for adaptation of multiple cameras in the network added in parallel.}
\label{fig:2parallelmulticam}
\end{figure}

\begin{figure}[H]
\centering
\large \underline{Camera wise CMC curves for Market1501 dataset: parallel addition of 3 cameras}\par\medskip
\begin{subfigure}{0.3\textwidth}
\includegraphics[width=\textwidth]{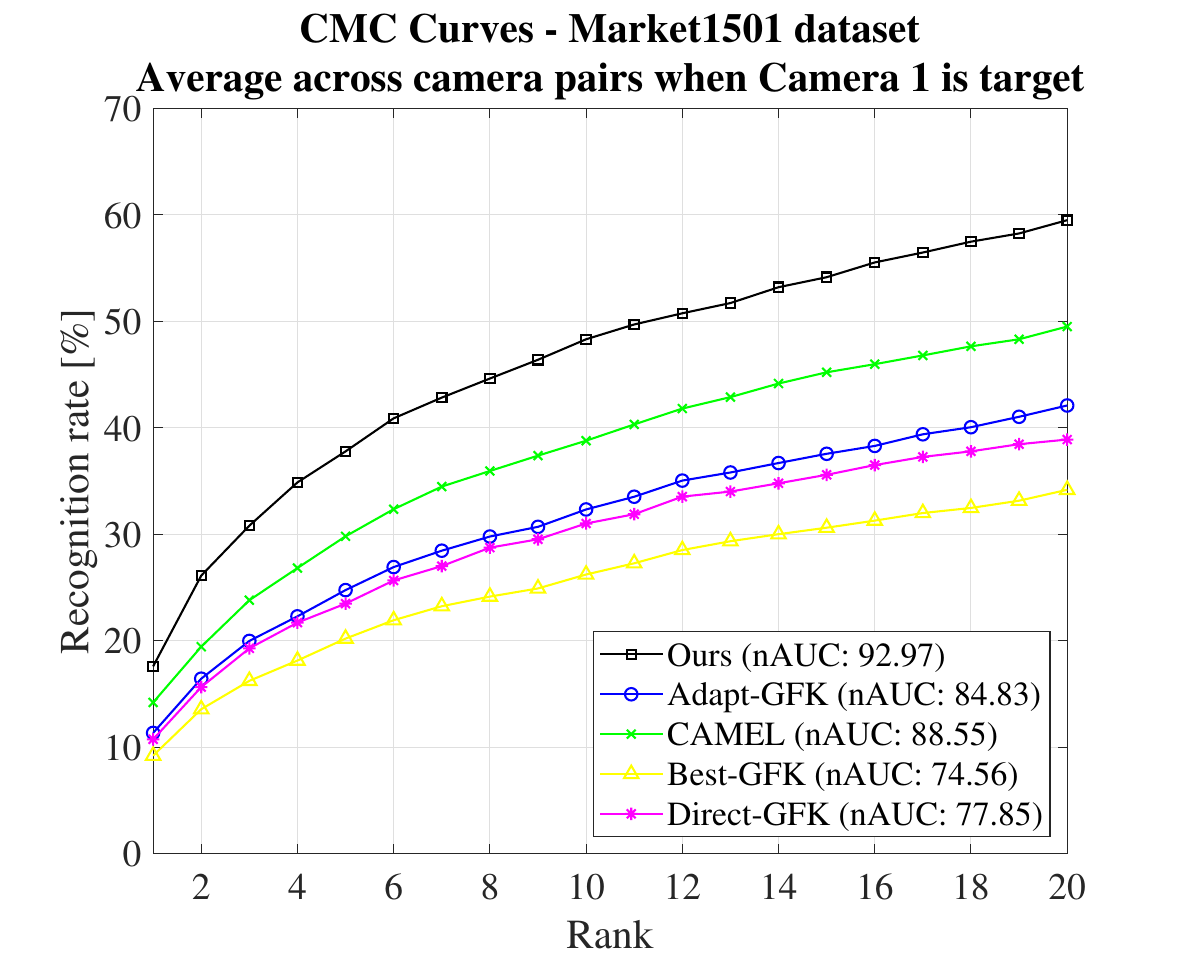}
\caption{}
\end{subfigure}
\begin{subfigure}{0.3\textwidth}
\includegraphics[width=\textwidth]{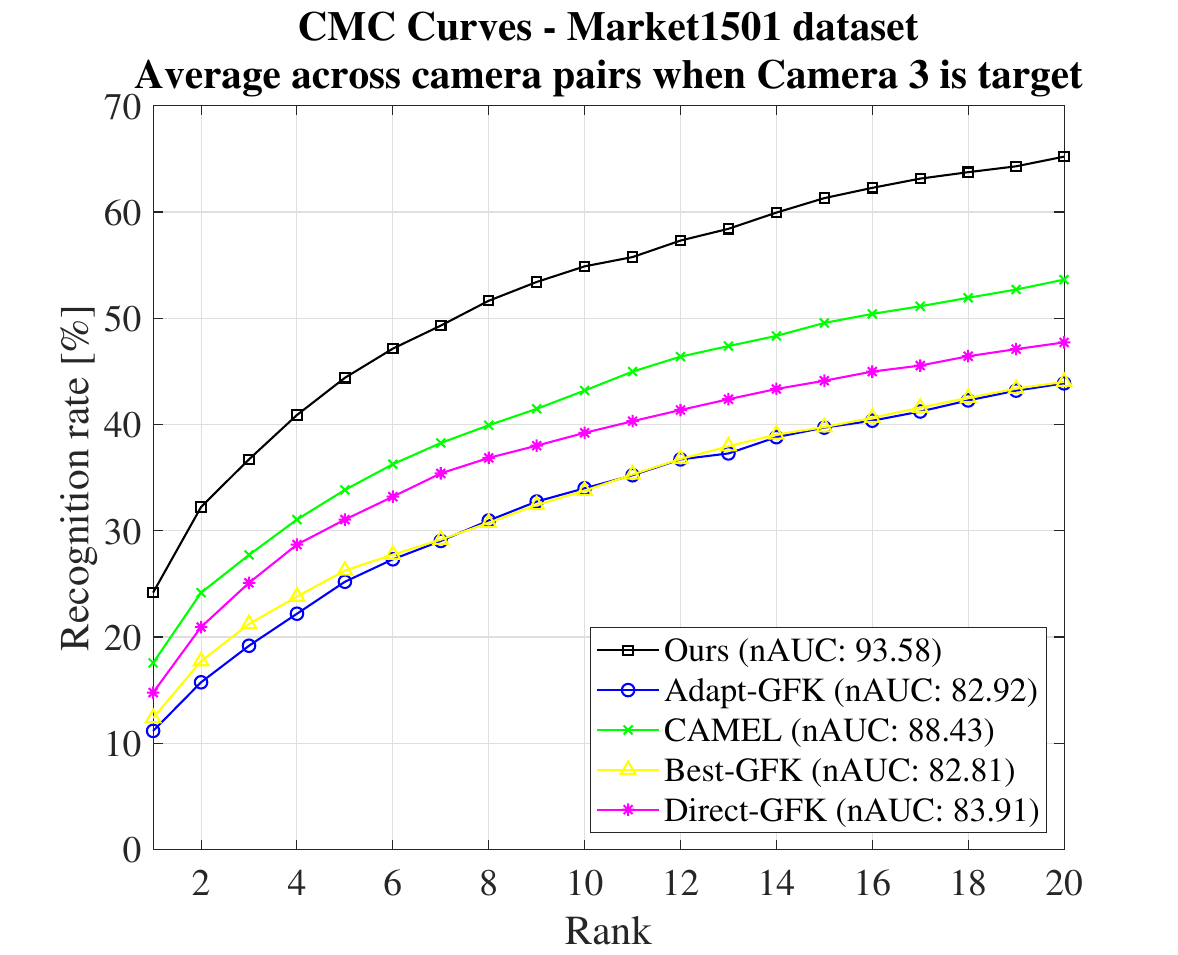}
\caption{}
\end{subfigure}  
\begin{subfigure}{0.3\textwidth}
\includegraphics[width=\textwidth]{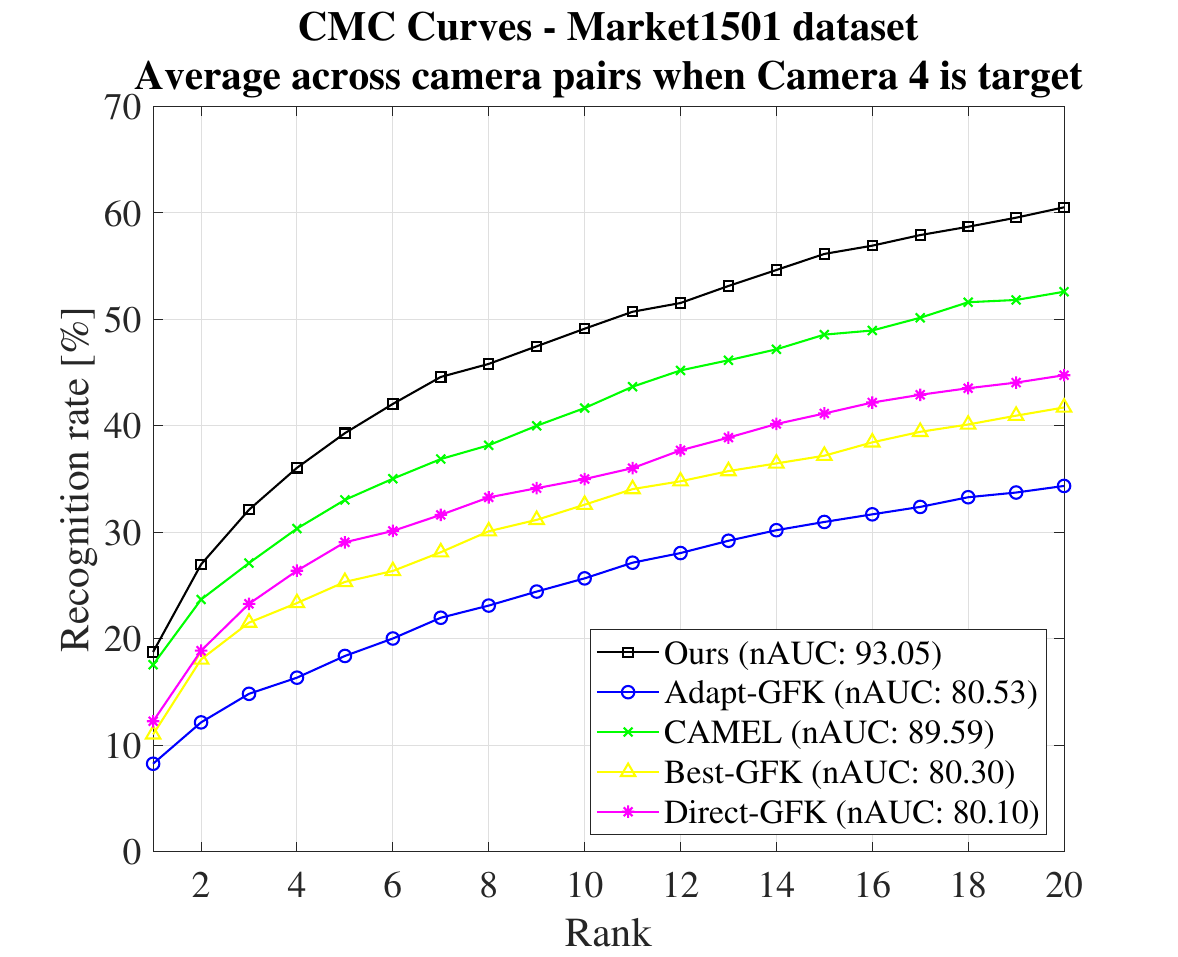}
\caption{}
\end{subfigure} 
\caption{In this figure we used Market1501 dataset to show the effect of parallel on-boarding of multiple cameras (In this case 3 cameras). We effectively set camera 1,3 and 4 as target and compute 3 source metrics from the remaining cameras to transfer knowledge from. Accuracy is shown between camera 1 and camera (2,5,6) (plot(a)),camera 3 and camera (2,5,6) (plot(b)) and also between camera 4 and camera (2,5,6) (plot(c)) separately. We can see that our method significantly outperform other methods both in rank-1 and nAUC. This shows the effectiveness of our method for adaptation of multiple cameras in the network added in parallel. Best viewed in color.}
\label{fig:3parallelmulticam}
\end{figure}

\begin{figure}[ht]
\centering
\large \underline{Camera wise CMC curves for Market1501 dataset: continuous addition of multiple cameras}\par\medskip
\begin{subfigure}{0.3\textwidth}
\includegraphics[width=\textwidth]{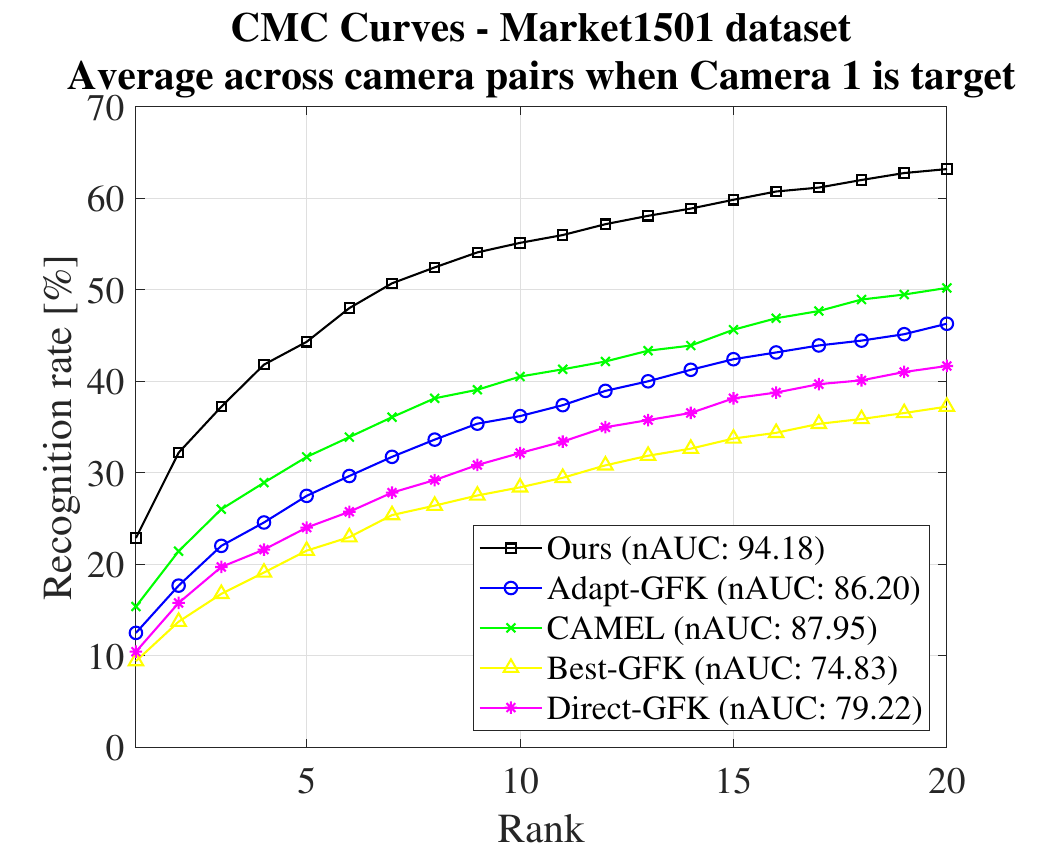}
\caption{}
\end{subfigure}
\begin{subfigure}{0.3\textwidth}
\includegraphics[width=\textwidth]{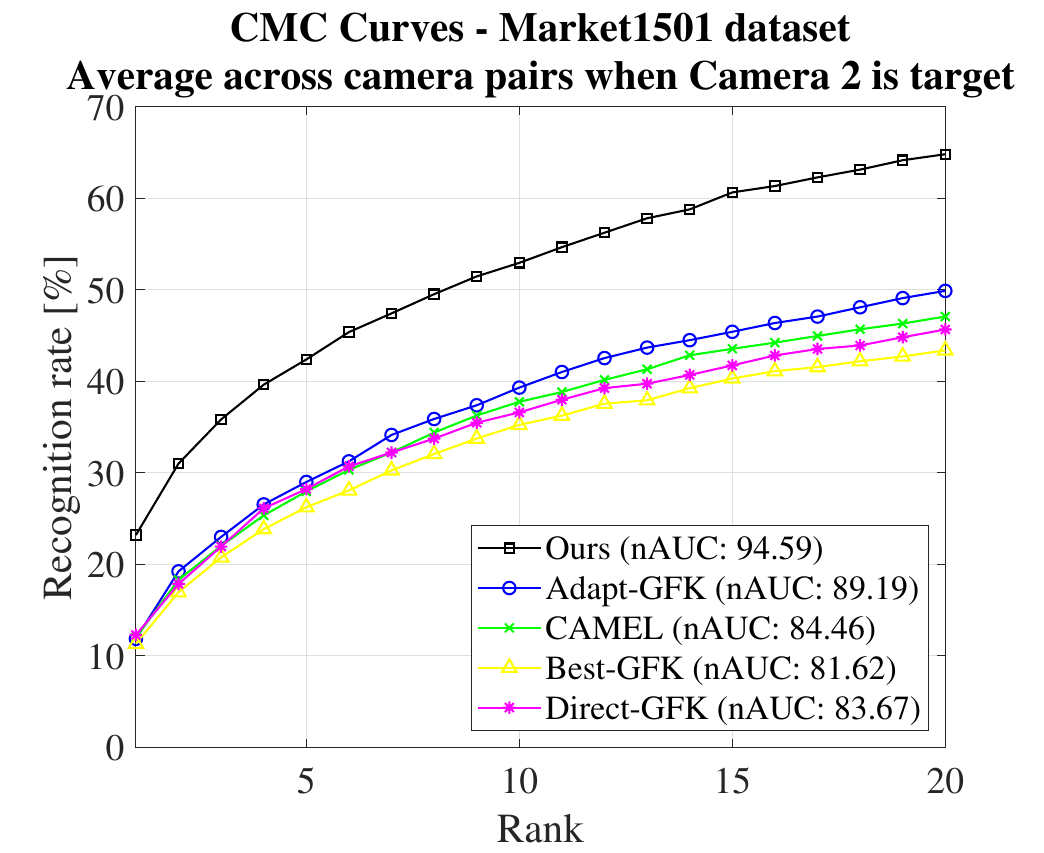}
\caption{}
\end{subfigure}  
\begin{subfigure}{0.3\textwidth}
\includegraphics[width=\textwidth]{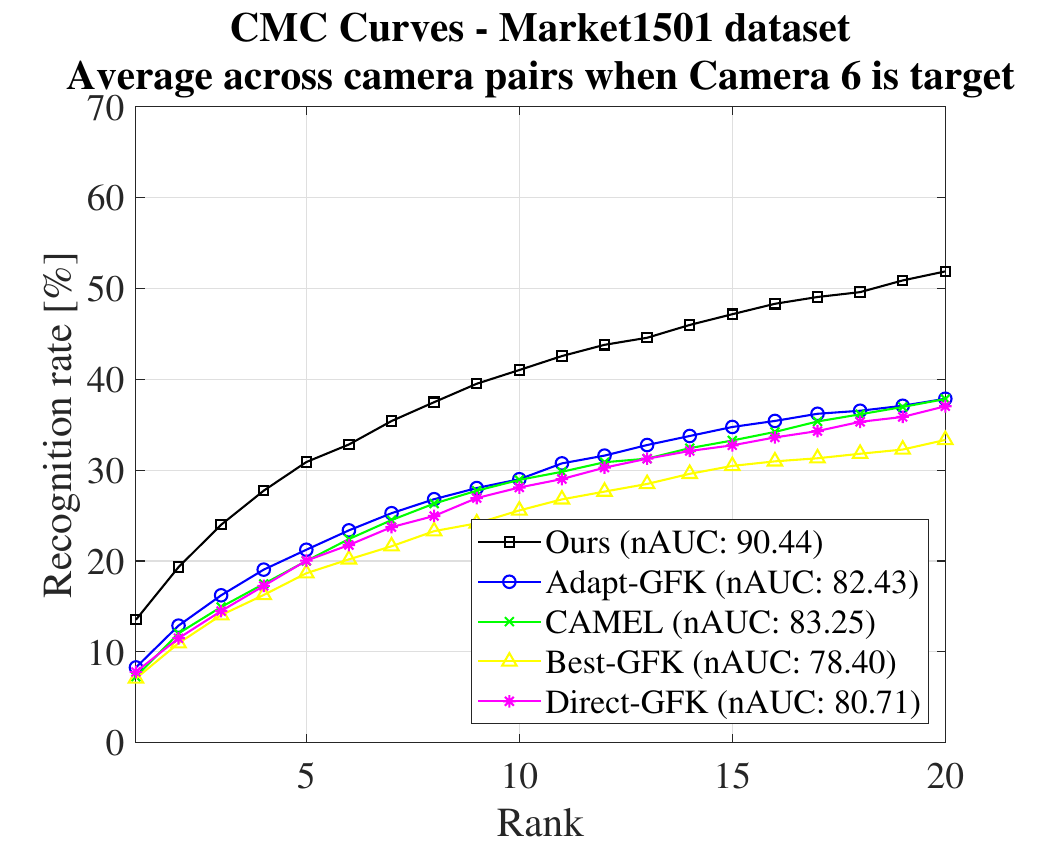}
\caption{}
\end{subfigure} 
\caption{In this figure we used Market1501 dataset to show the effect of sequential on-boarding of multiple cameras (In this case 3 cameras). Source cameras are camera 3,4 and 5 which has three source metrics between them. First camera 1 is added to the network and adapted. Accuracy for camera 1 as target is computed between camera 1 and camera (3,4,5) (plot(a)). Then camera 2 is added and adapted. For calculation of camera 2 adaptation accuracy we calculate matching score between camera 2 and camera (1,3,4,5) (plot(b)). In same fashion camera 6 is added afterwards and accuracy is calculated between camera 6 and camera (1,2,3,4,5) (plot(c)). We can see that our method significantly outperform other methods both in rank-1 and nAUC. This shows the effectiveness of our method for adaptation of multiple cameras in the network added sequentially.}
\label{fig:continuousmulticam}
\end{figure}
\section{Additional Experiments}
\noindent\textbf{Pairwise PCA vs Global PCA:} We calculate one PCA projection matrix for the whole source network and use that in the target to project features in the main paper. Additionally to compare, we did pairwise PCA and observe that it significantly lowers performance, e.g., rank-1 accuracy drops from 51.25 to 22.92 in RAiD, and drops from 62.86 to 25.71 in WARD. We believe this is due to lack of enough data across pair-wise cameras to give a reliable estimate of PCA subspaces. Combining different PCA projected metrics could be an interesting direction for future work.\\
\noindent\textbf{Effect of $\mathbf{\lambda}=0$:} 
When the existing pair-wise learned metrics are not considered (i.e., $\lambda=0$), the rank-1 performance significantly drops from 62.86\% to 27.14\% on WARD. From that we conclude that a finite nonzero positive $\lambda$ is a very crucial factor in order for the algorihm to work.\\

\noindent\textbf{Initialization}:
Since the proposed optimization is convex, initialization has very little effect on the performance. We tried 2 different initializations such as identity and random positive semidefinite matrices with random weights within the first quadrant of unit-norm hypersphere, and found that both resulted minimal difference in rank-1 accuracy (RAiD: 51.25 vs 50.83 and WARD: 62.82 vs 62.38).

\section{Finetuning with Deep Features}
\noindent\textbf{Goal:} In this section our goal is to show the performance of our method (See Table~\ref{table1} and Figure~\ref{fig:sinlecamdeep6}), if we have access to a deep model trained well using the source data. \\
\noindent\textbf{Implementation details:}
This section covers the implementation details of finetuning deep features used in the experiments of Section 5.4 in the main paper.
First, we train a ResNet model~\cite{he2016deep}, pretrained on the Imagenet dataset, using the source camera data. We remove the last classification layer and add two fully connected layers; one which embeds average pooled features to size $1024$ and another which works as a classifier. 
We use the optimized source features to train the source metrics that will later be used to calculate new target metrics.
Afterwards we fine-tune the model using the new target data and use the new optimized target features along with the source metrics in optimization~\ref{opt:main_opt_supp}.
The model is trained for 50 epochs using SGD, with a base learning rate of 0.001, which is decreased by a factor 10 after 20 and 40 epochs. We use a batch size of 32 and perform traditional data augmentation, such as cropping and flipping.
We use the optimized source features to train the source metrics that will later be used to calculate new target metrics.
Afterwards, we fine-tune the model for 30 epochs using the new target data. We fine-tune with a batch size of 32 and a base learning rate is 0.0001 and decreased by a factor 10 after 20 epochs. The new optimized target features are used along with the source metrics in optimization. From Figure 5 (b) of the main paper and Figure~\ref{fig:sinlecamdeep6} in here, we observe that when we remove sixth camera in Market dataset, the accuracy of the test set between sixth and other cameras become very low as $20\%$, whereas in standard result for fully supervised deep model in Market dataset is around  
$80\%$. This drop in accuracy from 80 to 20\% while removing $6^{th}$ camera in Market is due to two reasons. First, removing all the 151 person ids that appear in $6^{th}$ camera results in less labeled examples that leads to a less accurate deep model. Second, $6^{th}$ camera is the most uncorrelated with the other 5 cameras (see Fig. 7 in \cite{zheng2015scalable}). Figure {\color{red} 5}(b) in main paper and Figure~\ref{fig:sinlecamdeep6} in here clearly show that our approach works better than direct adaptation of the source model (even with finetuning) when feature distribution across source and target cameras are very different.
\begin{table}[H]
\centering
\begin{tabular}{|l|l|l|l|l|}
\hline
                  & \multicolumn{2}{l|}{\textbf{Single-query}} & \multicolumn{2}{l|}{\textbf{Multi-query}} \\ \hline
\textbf{Method}    & Top-1                 & mAP                & Top-1                & mAP                \\ \hline
                                                                            
Euclidean          &         46.51              &     40.04               &       54.40               &         48.54           \\
Euclidean-ft       &       51.51                &          45.52          &         59.66             &         54.36           \\
KISSME             &       45.57                &       38.42             &          55.31            &          48.02          \\
KISSME-ft          &         49.13              &          41.77          &          58.52            &            51.58        \\
Ours                &           47.79            &           41.20         &          57.57            &          50.83          \\
Ours-ft             &           \bf{52.84}            &          \bf{46.70}          &          \bf{61.96}            &         \bf{56.28}           \\ \hline
\end{tabular}
\caption{Results for Market1501 when we have a deep model trained using the data of 5 source cameras. We set each camera as target with 25\% labeled data in it and show result of average across all the cameras. \textbf{Euclidean} denotes the accuracy of target camera if the trained source model is directly used to extract features in target test set. \textbf{KISSME} is direct metric learning between new camera and old cameras. \textit{ft} stands for fine tuning. \textbf{Euclidean-ft} and \textbf{KISSME-ft} is same scheme that is described in the top lines of this section, except for the feature extraction policy. In these methods features are extracted using the fine tuned source model with limited target data. We can see that our proposed algorithm using features from fine-tuned model outperforms all the other accuracies.}
\label{table1}
\end{table}

\begin{figure}[ht]
\centering
\large \underline{CMC curves for Market1501 dataset with Camera 6 as target using deep learned features}\par\medskip
\begin{subfigure}{0.3\textwidth}
\includegraphics[width=\textwidth]{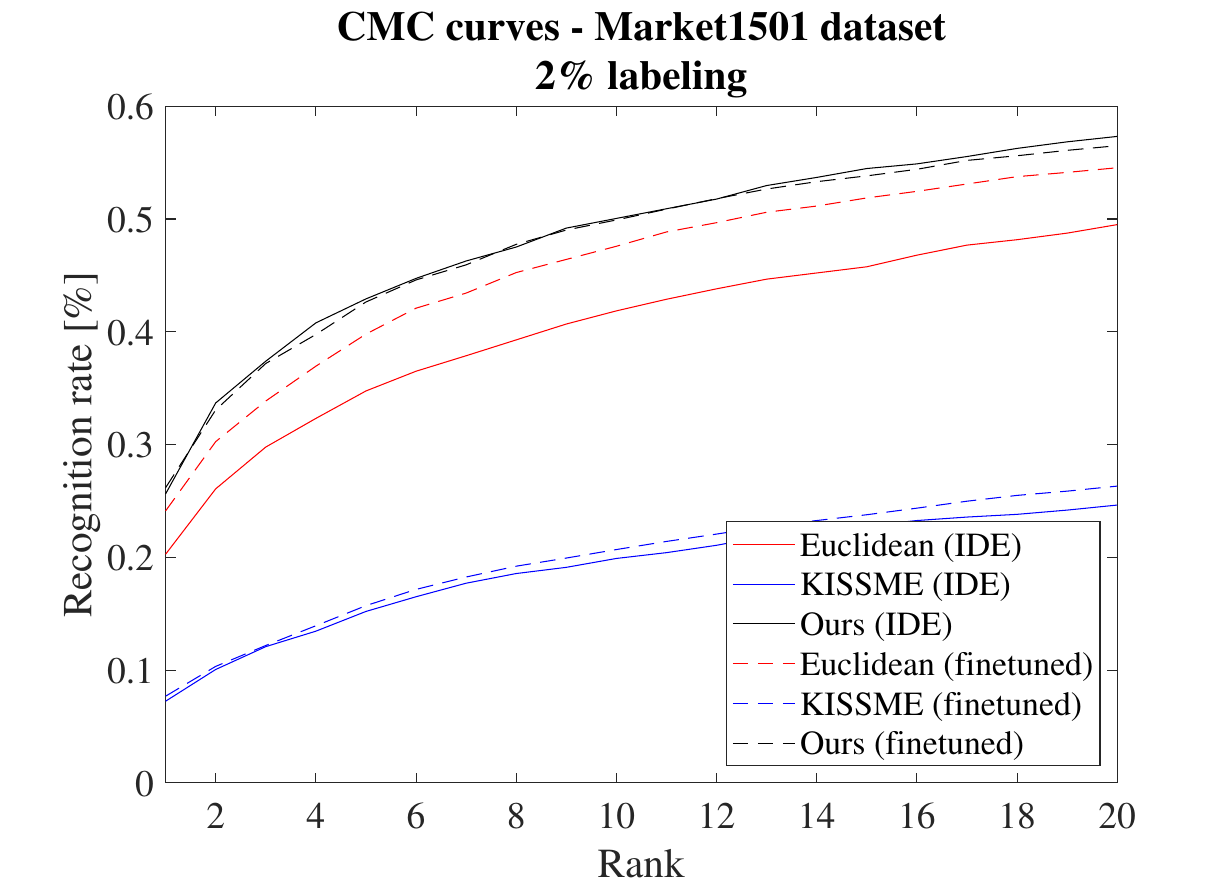}
\caption{}
\end{subfigure}
\begin{subfigure}{0.3\textwidth}
\includegraphics[width=\textwidth]{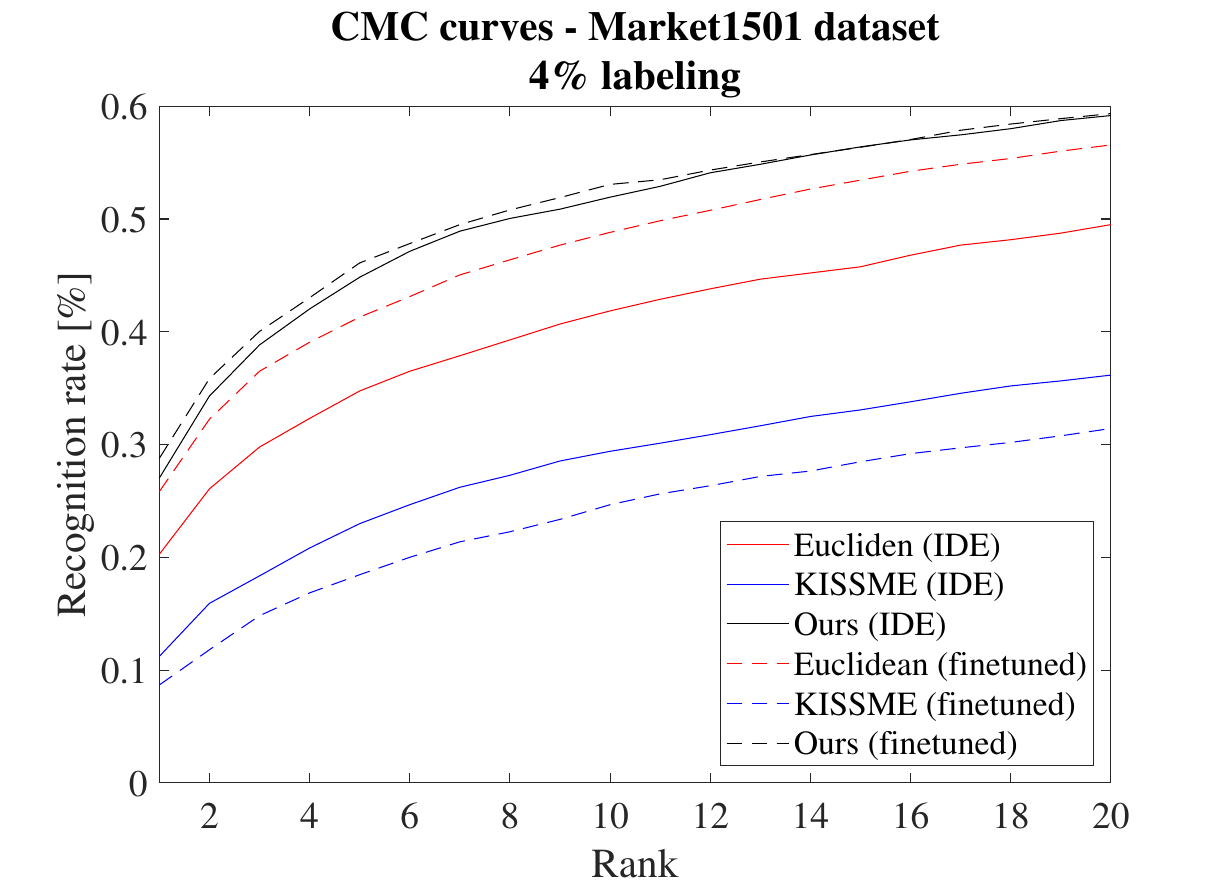}
\caption{}
\end{subfigure} 
\begin{subfigure}{0.3\textwidth}
\includegraphics[width=\textwidth]{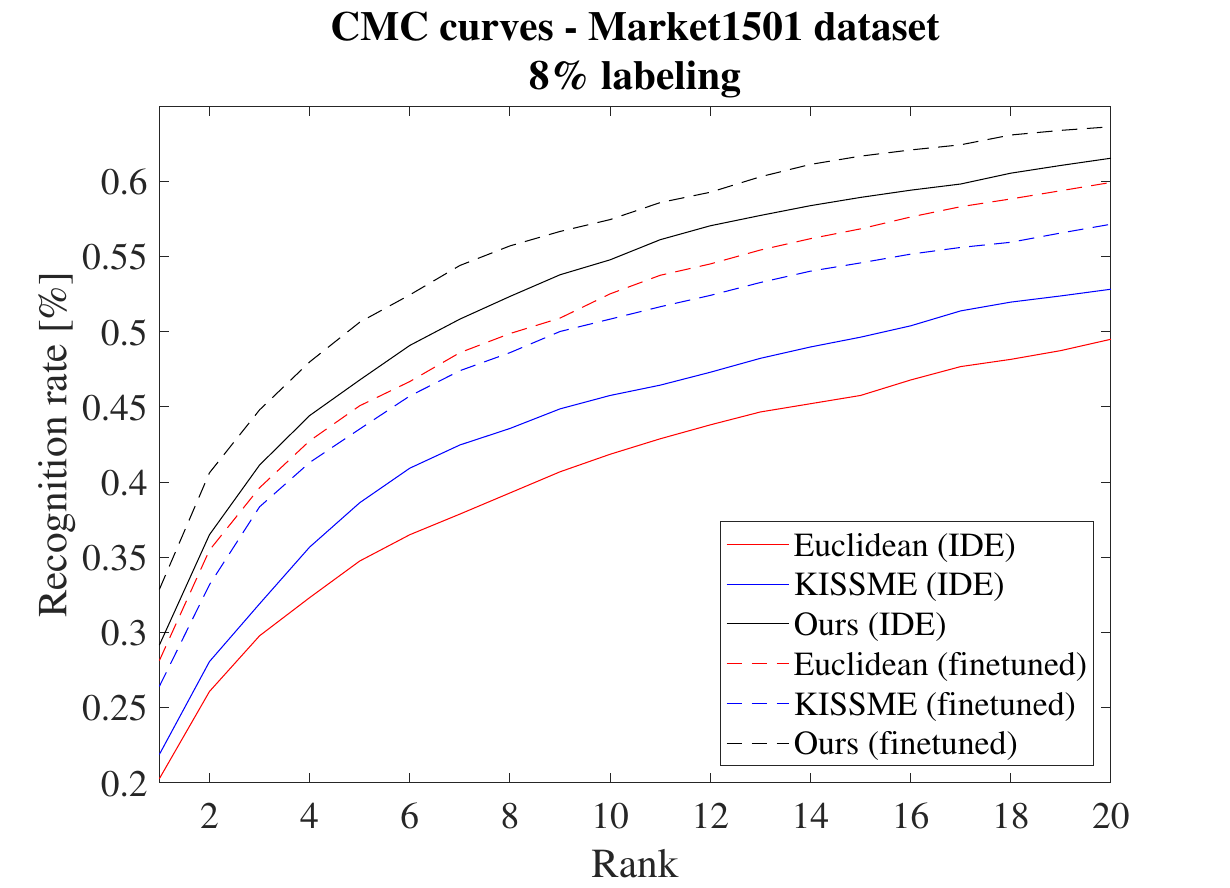}
\caption{}
\end{subfigure} \\
\begin{subfigure}{0.3\textwidth}
\includegraphics[width=\textwidth]{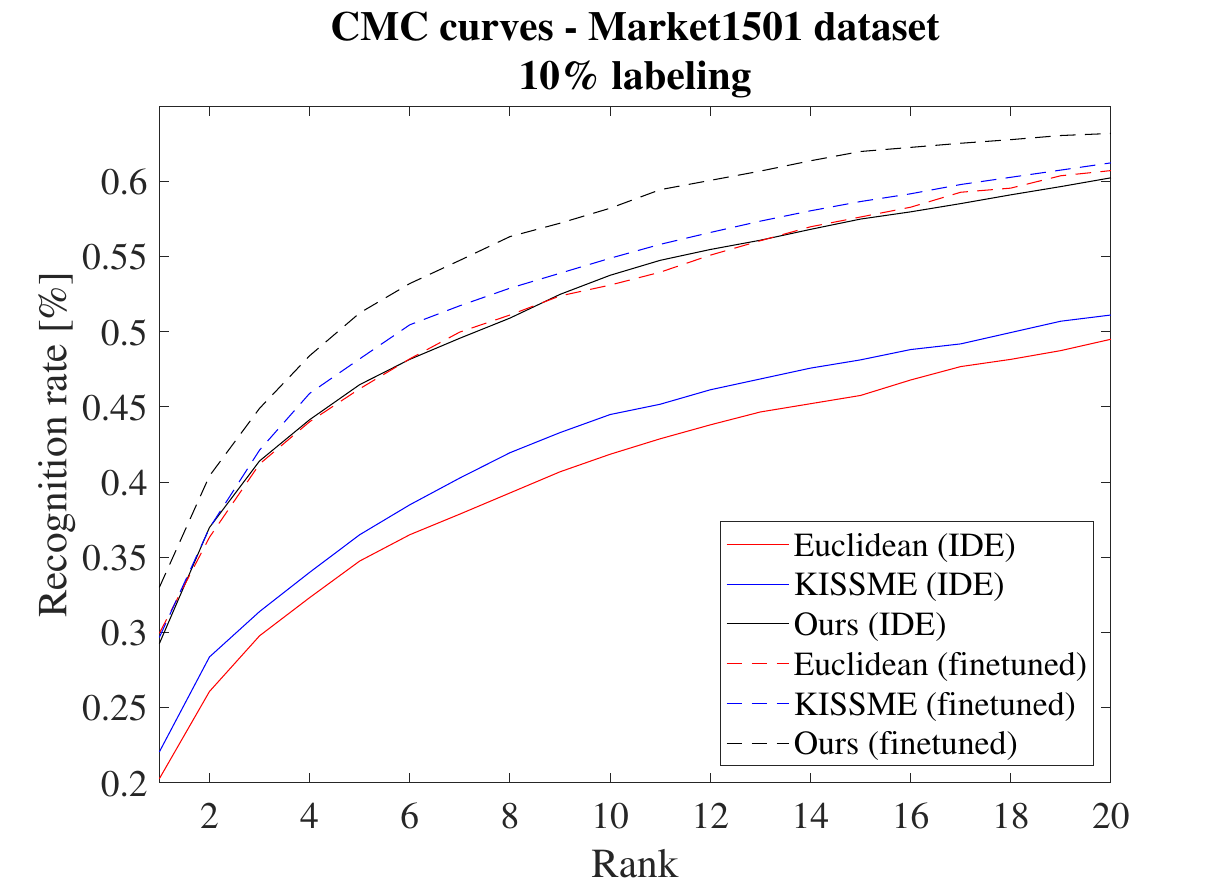}
\caption{}
\end{subfigure} 
\begin{subfigure}{0.3\textwidth}
\includegraphics[width=\textwidth]{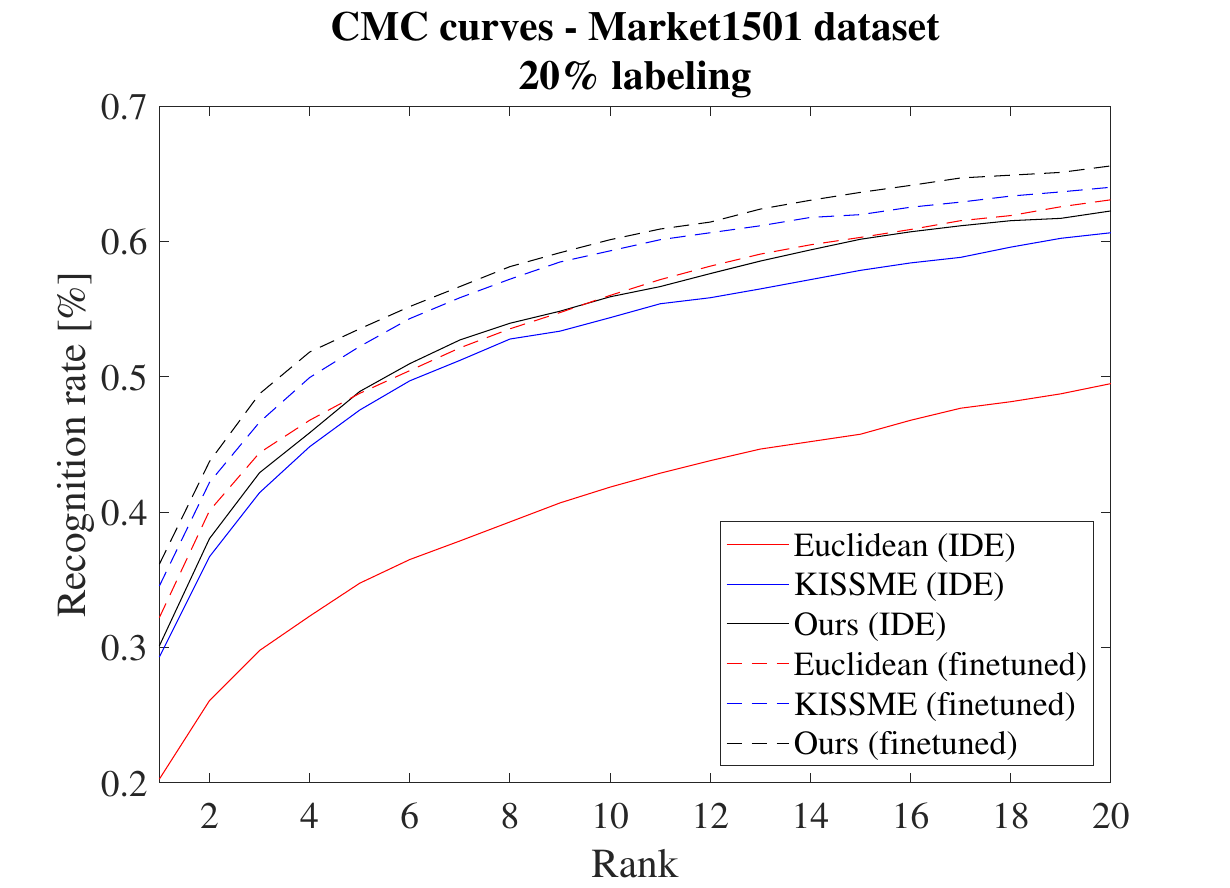}
\caption{}
\end{subfigure}
\caption{These plots show cmc curves for camera 6 of Market1501 dataset using the exact same scheme of Table~\ref{table1} but with different percentage labels in the target. We can clearly see that our method outperforms all the other (That is direct euclidean, direct metric learning and even fine tuning with target data). When the percentage label increase then our method with non-finetuned features merges with the direct fine tuning, whereas if we use our method with the finetuned features, it exceeds all the accuracy. This shows the strength of our method even in the presence of deep learned source model.}
\label{fig:sinlecamdeep6}
\end{figure}

